\documentclass{article}

\PassOptionsToPackage{numbers}{natbib}

\usepackage[final]{neurips_2020}

\usepackage[utf8]{inputenc} 
\usepackage[T1]{fontenc}    
\usepackage{hyperref}       
\usepackage{url}            
\usepackage{booktabs}       
\usepackage{amsfonts}       
\usepackage{nicefrac}       
\usepackage{microtype}      
\usepackage{caption}
\usepackage{bm}
\usepackage{amsthm}
\usepackage{amsmath}
\usepackage{enumitem}
\usepackage[noline,noend]{algorithm2e}
\usepackage{mathtools}
\usepackage{booktabs}
\usepackage{array}
\usepackage{natbib}
\usepackage{amssymb}
\usepackage{dsfont}
\usepackage{stackengine}
\usepackage{scalerel}
\usepackage{xspace}
\usepackage{wrapfig}
\usepackage[dvipsnames]{xcolor}
\usepackage{tikz}
\usepackage{tcolorbox}
\usepackage{xfrac}

\usepackage{longtable} 
\usepackage{multirow}

\usepackage{pifont}
\newcommand{\cmark}{\ding{51}}
\newcommand{\xmark}{\ding{55}}

\newtheorem{assumption}{Assumption}
\newtheorem{proposition}{Proposition}
\newtheorem{corollary}{Corollary}
\newtheorem{theorem}{Theorem}
\newtheorem{lemma}{Lemma}

\newcommand{\Var}{\mathbb{V}\mathrm{ar}}
\newcommand{\R}{\mathbb{R}}
\newcommand{\E}{\mathbb{E}}

\newcommand{\F}{\mathcal{F}}

\newcommand{\expec}[1]{\mathbb{E}\left[#1\right]}
\newcommand{\prob}[1]{\mathbb{P}\left\{#1\right\}}
\newcommand{\indi}[1]{\mathds{1}\left\{#1\right\}}

\newcommand{\argmin}{\operatornamewithlimits{argmin}}
\newcommand{\argmax}{\operatornamewithlimits{argmax}}

\newcommand{\red}[1]{#1}

\newcommand{\A}{\mathcal{A}}
\newcommand{\X}{\mathcal{X}}
\newcommand{\transp}{\mathsf{T}}

\renewcommand{\Re}{\mathbb{R}}
\newcommand{\algo}{\textsc{SOLID}\xspace}

\newcommand{\wb}[1]{\overline{#1}}
\newcommand{\wt}[1]{\widetilde{#1}}
\newcommand{\wh}[1]{\widehat{#1}}

\newcommand{\sumpull}{z^\star}
\makeatletter
\renewcommand{\@algocf@capt@plain}{above}
\makeatother

\usepackage{array}
\newcolumntype{L}[1]{>{\raggedright\let\newline\\\arraybackslash\hspace{0pt}}m{#1}}
\newcolumntype{C}[1]{>{\centering\let\newline\\\arraybackslash\hspace{0pt}}m{#1}}
\newcolumntype{R}[1]{>{\raggedleft\let\newline\\\arraybackslash\hspace{0pt}}m{#1}}

\usepackage[colorinlistoftodos, textwidth=30mm]{todonotes}

\usepackage{minitoc}

\definecolor{pastelgreen}{HTML}{03c03c}
\usepackage{fontawesome}

\title{An Asymptotically Optimal Primal-Dual Incremental Algorithm for Contextual Linear Bandits}

\author{%
  Andrea Tirinzoni\thanks{Work done while at Facebook.}\\
  Politecnico di Milano\\
  \texttt{andrea.tirinzoni@polimi.it} \\
  \And
  Matteo Pirotta \\
  Facebook AI Research \\
  \texttt{pirotta@fb.com} \\
  \AND
  Marcello Restelli \\
  Politecnico di Milano \\
  \texttt{marcello.restelli@polimi.it} \\
  \And
  Alessandro Lazaric \\
  Facebook AI Research \\
  \texttt{lazaric@fb.com} \\
}

\begin{document}

\maketitle

\doparttoc 
\faketableofcontents 



\begin{abstract}
In the \red{contextual linear bandit} setting, algorithms built on the optimism principle fail to exploit the structure of the problem and have been shown to be asymptotically suboptimal. In this paper, we follow recent approaches of deriving asymptotically optimal algorithms from problem-dependent regret lower bounds and we introduce a novel algorithm improving over the state-of-the-art along multiple dimensions. We build on a reformulation of the lower bound, where context distribution and exploration policy are decoupled, and we obtain an algorithm robust to unbalanced context distributions. Then, using an incremental primal-dual approach to solve the Lagrangian relaxation of the lower bound, we obtain a scalable and computationally efficient algorithm. Finally, we remove forced exploration and build on confidence intervals of the optimization problem to encourage a minimum level of exploration that is better adapted to the problem structure. We demonstrate the \emph{asymptotic optimality} of our algorithm, while providing both problem-dependent and worst-case finite-time regret guarantees. Our bounds scale with the logarithm of the number of arms, thus avoiding the linear dependence common in all related prior works. \red{Notably, we establish \emph{minimax optimality} for any learning horizon in the special case of non-contextual linear bandits.} Finally, we verify that our algorithm obtains better empirical performance than state-of-the-art baselines.
\end{abstract}

\section{Introduction}

We study the contextual linear bandit (CLB) setting~\cite[e.g.,][]{lattimore2020bandit}, where at each time step $t$ the learner observes a context $X_t$ drawn from a context distribution $\rho$, pulls an arm $A_t$, and receives a reward $Y_t$ drawn from a distribution whose expected value is a linear combination between \red{$d$-dimensional} features $\phi(X_t,A_t)$ describing context and arm, and an unknown parameter $\theta^\star$. The objective of the learner is to maximize the reward over time, that is to minimize the cumulative regret w.r.t.\ an optimal strategy that selects the best arm in each context. This setting formalizes a wide range of problems such as online recommendation systems, clinical trials, dialogue systems, and many others~\cite{bouneffouf2019survey}. Popular algorithmic principles, such as optimism-in-face-of-uncertainty and Thompson sampling \citep{thompson1933likelihood}, have been applied to this setting leading to algorithms such as \textsc{OFUL}~\cite{abbasi2011improved} and \textsc{LinTS}~\cite{agrawal2013thompson,abeille2017linear} with strong finite-time worst-case regret guarantees. Nonetheless, Lattimore \& Szepesvari~\cite{lattimore2017end} recently showed that these algorithms are not asymptotically optimal (in a problem-dependent sense) as they fail to adapt to the structure of the problem at hand. In fact, in the CLB setting, the values of different arms are tightly connected through the linear assumption and a possibly suboptimal arm may provide a large amount of information about $\theta^\star$ and thus the optimal arm. Optimistic algorithms naturally discard suboptimal arms and thus may miss the chance to acquire information about $\theta^\star$ and significantly reduce the regret.

Early attempts to exploit general structures in MAB either adapted UCB-based strategies~\cite{azar2013sequential,lattimore2014bounded} or focused on different criteria, such as regret to information ratio~\cite{russo2014learning}. While these approaches succeed in improving the finite-time performance of optimism-based algorithms, they still do not achieve asymptotic optimality. An alternative approach to exploit the problem structure was introduced in~\cite{lattimore2017end} for (non-contextual) linear bandits. Inspired by approaches for regret minimization~\red{\cite{agrawal1988asymptotically,burnetas1996optimal,graves1997asymptotically} }and best-arm identification~\cite{garivier2016optimal} in MAB, Lattimore \& Szepesvari~\cite{lattimore2017end} proposed to compute an exploration strategy by solving the (estimated) optimization problem characterizing the asymptotic regret lower bound for linear bandits. While the resulting algorithm matches the asymptotic logarithmic lower bound with tight leading constant, it performs rather poorly in practice. Combes et al.~\cite{combes2017minimal} followed a similar approach and proposed \textsc{OSSB}, an asymptotically optimal algorithm for bandit problems with general structure (including, e.g., linear, Lipschitz, unimodal). Unfortunately, once instantiated for the linear bandit case, \textsc{OSSB} suffers from poor empirical performance due to the large dependency on the number of arms.
Recently, Hao et al.~\cite{hao2019adaptive} introduced \textsc{OAM}, an asymptotically optimal algorithm for the CLB setting. While \textsc{OAM} effectively exploits the linear structure and outperforms other bandit algorithms, it suffers from major limitations. From an algorithmic perspective, at each exploration step, \textsc{OAM} requires solving the optimization problem of the regret lower bound, which can hardly scale beyond problems with a handful of contexts and arms. Furthermore, \textsc{OAM} implements a forcing exploration strategy that often leads to long periods of linear regret and introduces a linear dependence on the number of arms $|\A|$. Finally, the regret analysis reveals a critical dependence on the inverse of the smallest probability of a context (i.e., $\min_x \rho(x)$), thus suggesting that \textsc{OAM} may suffer from poor finite-time performance in problems with unbalanced context distributions.\footnote{Interestingly, Hao et al.~\cite{hao2019adaptive} explicitly mention in their conclusions the importance of properly managing the context distribution to achieve satisfactory finite-time performance.
} Degenne et al.~\cite{degenne2020structure} recently introduced \textsc{SPL}, which significantly improves over previous algorithms for MAB problems with general structures. Inspired by algorithms for best-arm identification~\cite{degenne2019non}, Degenne et al. reformulate the optimization problem in the lower bound as a saddle-point problem and show how to leverage online learning methods to avoid recomputing the exploration strategy from scratch at each step. Furthermore, \textsc{SPL} removes any form of forced exploration by introducing optimism into the estimated optimization problem. As a result, \textsc{SPL} is computationally efficient and achieves better empirical performance in problems with general structures.

\textbf{Contributions.} In this paper, we follow similar steps as in~\cite{degenne2020structure} and introduce \algo, a novel algorithm for the CLB setting.
 Our main contributions can be summarized as follows.
\begin{itemize}[leftmargin=.3in,topsep=0pt,itemsep=2pt,partopsep=0pt, parsep=0pt]
	\item We first reformulate the optimization problem associated with the lower bound for contextual linear bandits~\cite{combes2017minimal,ok2018exploration,hao2019adaptive} by introducing an additional constraint to guarantee bounded solutions and by explicitly decoupling the context distribution and the exploration policy. While we bound the bias introduced by the constraint, we also illustrate how the resulting exploration policy is better adapted to unbalanced context distributions.
	\item Leveraging the Lagrangian dual formulation associated with the constrained lower-bound optimization problem, we derive \algo, an efficient primal-dual learning algorithm that  incrementally updates the exploration strategy at each time step. Furthermore, we replace forced exploration with an optimistic version of the optimization problem by specifically leveraging the linear structure of the problem. Finally, \algo does not require any explicit tracking step and it samples directly from the current exploration strategy. 
	\item \red{We establish the \emph{asymptotic optimality} of \algo, while deriving a finite-time problem-dependent regret bound that scales only with $\log |\A|$ and without any dependence on $\min_x \rho(x)$.
	To this purpose, we introduce a new concentration bound for regularized least-squares that scales as $\mathcal{O}(\log t + d\log\log t)$, hence removing the $d\log t$ dependence of the bound in \citep{abbasi2011improved}.
	{ Moreover, we establish a $\wt{\mathcal{O}}((\sqrt{d}+|\X|)\sqrt{dn})$ worst-case regret bound for any CLB problem with $|\X|$ contexts, $d$ features, and horizon $n$. Notably, this is implies that \algo is the first algorithm to be simultaneously \emph{asymptotically optimal} and \emph{minimax optimal} when $|\X| \leq \sqrt{d}$ (e.g., in non-contextual linear bandits, when $|\X| = 1$).}}
	\item We empirically compare to a number of state-of-the-art methods for contextual linear bandits and show how \algo is more computationally efficient and often has the smallest regret.
\end{itemize}

A thorough comparison between \algo and related work is reported in App.~\ref{app:related.work}.

\section{Preliminaries}

We consider the contextual linear bandit setting. Let $\X$ be the set of contexts and $\A$ be the set of arms with cardinality $|\X| < \infty$ and $|\A| < \infty$, respectively. Each context-arm pair is embedded into $\Re^d$ through a feature map $\phi: \X\times\A\rightarrow \Re^d$. For any reward model $\theta\in\Re^d$, we denote by $\mu_{\theta}(x,a) = \phi(x,a)^\transp \theta$ the expected reward for each context-arm pair.
Let $a^\star_{\theta}(x) := \argmax_{a\in\A}\mu_{\theta}(x,a)$ and $\mu^\star_{\theta}(x) := \max_{a\in\A}\mu_{\theta}(x,a)$ denote the optimal arm and its value for context $x$ and parameter $\theta$. We define the sub-optimality gap of arm $a$ for context $x$ in model $\theta$ as $\Delta_\theta(x,a) := \mu^\star_{\theta}(x) - \mu_\theta(x,a)$. We assume that every time arm $a$ is selected in context $x$, a random observation $Y = \phi(x,a)^\transp\theta + \xi$ is generated, where $\xi \sim \mathcal{N}(0,\sigma^2)$ is a Gaussian noise.\footnote{This assumption can be relaxed by considering sub-Gaussian rewards.
} Given two parameters $\theta,\theta' \in \Re^d$, we define
$d_{x,a}(\theta, \theta') := \frac{1}{2\sigma^2} (\mu_\theta(x,a) - \mu_{\theta'}(x,a))^2,$
which corresponds to the Kullback-Leibler divergence between the Gaussian reward distributions of the two models in context $x$ and arm $a$.

At each time step $t \in \mathbb{N}$, the learner observes a context $X_t\in\X$ drawn i.i.d.\ from a context distribution $\rho$, it pulls an arm $A_t\in\A$, and it receives a reward $Y_{t} = \phi(X_t,A_t)^\transp\theta^\star + \xi_t$, where $\theta^\star \in \Re^d$ is unknown to the learner.
A bandit strategy $\pi := \{\pi_t\}_{t\geq 1}$ chooses the arm $A_t$ to pull at time $t$ as a measurable function $\pi_t(H_{t-1}, X_t)$ of the current context $X_t$ and of the past history $H_{t-1} := (X_1, Y_1, \dots, X_{t-1}, Y_{t-1})$. The objective is to define a strategy that minimizes the expected cumulative regret over $n$ steps,
%
$\E_{\xi,\rho}^\pi \big[ R_n(\theta) \big]:= 
\E_{\xi,\rho}^\pi \left[ \sum_{t=1}^n \left( 
	\mu_{\theta}^\star(X_t) - \mu_{\theta}(X_t, A_t)
\right) \right],$
%
where $\E_{\xi,\rho}^\pi$ denotes the expectation w.r.t.\ the randomness of contexts, the noise of the rewards, and any randomization in the algorithm. We denote by $\theta^\star$ the reward model of the bandit problem at hand, and without loss of generality we rely on the following regularity assumptions.

\begin{assumption}\label{asm:regularity}
        The realizable parameters belong to a compact subset $\Theta$ of $\R^d$ such that $\|\theta\|_2 \leq B$ for all $\theta\in\Theta$. The features are bounded, i.e., $\| \phi(x,a) \|_2 \leq L$ for all $x\in\X,a\in\A$. The context distribution is supported over the whole context set, i.e., $\rho(x) \geq \rho_{\min} > 0$ for all $x\in\X$. Finally, w.l.o.g. we assume $\theta^\star$ has a unique optimal arm in each context~\cite[see e.g.,][]{combes2017minimal,hao2019adaptive}.
\end{assumption}


\textbf{Regularized least-squares estimator.} 
We introduce the regularized least-square estimate of $\theta^\star$ using $t$ samples as $\wh{\theta}_t := \wb{V}_t^{-1}U_t$, where $\wb V_t := \sum_{s=1}^t \phi(X_s, A_s)\phi(X_s,A_s)^\transp + \nu I$, with $\nu \geq \max\{L^2,1\}$ and $I$ the $d\times d$ identity matrix, and $U_t := \sum_{s=1}^t \phi(X_s,A_s) Y_s$. The estimator $\wh\theta_t$ satisfies the following concentration inequality (see App.~\ref{app:conf.set} for the proof and exact formulation).

\begin{theorem}\label{th:conf-theta}
	Let $\delta\in(0,1)$, $n\geq3$, and $\wh\theta_t$ be a regularized least-square estimator obtained using $t\in[n]$ samples collected using an arbitrary bandit strategy $\pi := \{\pi_t\}_{t\geq 1}$. Then,
	\begin{align*}
	\prob{\exists t\in[n] : \|\wh{\theta}_t - \theta^\star\|_{\wb{V}_t} \geq \sqrt{c_{n,\delta}}} \leq \delta,
	\end{align*}
	where $c_{n,\delta}$ is of order $\mathcal{O}(\log(1/\delta) + d\log\log n)$.
\end{theorem}

For the usual choice $\delta = 1/n$, $c_{n,1/n}$ is of order $\mathcal{O}(\log n + d\log\log n)$, which illustrates how the dependency on $d$ is on a lower-order term w.r.t.\ $n$ (as opposed to the well-known concentration bound derived in \citep{abbasi2011improved}). This result is the counterpart of~\citep[][Thm. 8]{lattimore2017end} for the concentration on the reward parameter estimation error instead of the prediction error and we believe it is of independent interest.

\section{Lower Bound}\label{sec:lower.bound}

We recall the asymptotic lower bound for multi-armed bandit problems with structure from~\cite{lai1985asymptotically,combes2017minimal,ok2018exploration}. We say that a bandit strategy $\pi$ is \textit{uniformly good} if $\E_{\xi,\rho}^\pi \big[ R_n \big] = o(n^\alpha)$ for any $\alpha>0$ and any contextual linear bandit problem satisfying Asm.~\ref{asm:regularity}.

\begin{proposition}\label{p:lower.bound}
Let $\pi := \{\pi_t\}_{t\geq 1}$ by a \textit{uniformly good} bandit strategy then,
\begin{equation}\label{eq:regret.lower.bound}
\liminf_{n \rightarrow \infty}\frac{\E_{\xi,\rho}^\pi \big[ R_n(\theta^\star) \big]}{\log(n)} \geq v^\star(\theta^\star),
\end{equation} 
where $v^\star(\theta^\star)$ is the value of the optimization problem
	\begin{equation}\label{eq:optim-lb}\tag{P}
	\begin{aligned}
    &\underset{\eta(x,a) \geq 0}{\inf}&& \sum_{x\in\X}\sum_{a\in\A}\eta(x,a)\Delta_{\theta^\star}(x,a)
    \quad \mathrm{s.t.} \quad
    \inf_{\theta' \in \Theta_{\mathrm{alt}}}\sum_{x\in\X}\sum_{a\in\A}\eta(x,a)d_{x,a}(\theta^\star,\theta') \geq 1,
	\end{aligned}
	\end{equation}
where $\Theta_{\mathrm{alt}} := \{ \theta' \in \Theta\ |\ \exists x\in\X,\ a^\star_{\theta^\star}(x) \neq a^\star_{\theta'}(x) \}$ is the set of alternative reward parameters such that the optimal arm changes for at least a context $x$.\footnote{\red{The infimum over this set can be computed in closed-form when the alternative parameters are allowed to lie in the whole $\R^d$ (see App. \ref{app:implementation-details}). When these parameters are forced to have bounded $\ell_2$-norm, the infimum has no closed-form expression, though its computation reduces to a simple convex optimization problem (see \cite{degenne2020gamification}).}}
\end{proposition}

The variables $\eta(x,a)$ can be interpreted as the number of pulls allocated to each context-arm pair so that enough information is obtained to correctly identify the optimal arm in each context while minimizing the regret. 
%
Formulating the lower bound in terms of the solution of~\eqref{eq:optim-lb} is not desirable for two main reasons. First,~\eqref{eq:optim-lb} is not a well-posed optimization problem since the inferior may not be attainable, i.e., the optimal solution may allocate an infinite number of pulls to some optimal arms. Second,~\eqref{eq:optim-lb} removes any dependency on the context distribution $\rho$. In fact, the optimal solution $\eta^\star$ of~\eqref{eq:optim-lb} may prescribe to select a context-arm $(x,a)$ pair a large number of times, despite $x$ having low probability of being sampled from $\rho$. While this has no impact on the asymptotic performance of $\eta^\star$ (as soon as $\rho_{\min}>0$), building on $\eta^\star$ to design a learning algorithm may lead to poor finite-time performance. 
%
In order to mitigate these issues, we propose a variant of the previous lower bound obtained by adding a constraint on the cumulative number of pulls in each context and explicitly decoupling the context distribution $\rho$ and the \textit{exploration policy} $\omega(x,a)$ defining the probability of selecting arm $a$ in context $x$. Given $z\in\Re_{>0}$, we define the optimization problem 
\begin{equation}\label{eq:optim-lb-z}\tag{P$_z$}
\begin{aligned}
&\underset{\omega \in \Omega}{\min}&& z\E_{\rho}\bigg[\sum_{a\in\A}\omega(x,a)\Delta_{\theta^\star}(x,a)\bigg]
\quad \mathrm{s.t.} \quad
\inf_{\theta' \in \Theta_{\mathrm{alt}}}\E_{\rho}\bigg[\sum_{a\in\A}\omega(x,a)d_{x,a}(\theta^\star,\theta') \bigg] \geq 1/z
\end{aligned}
\end{equation}
where $\Omega = \{ \omega(x,a) \geq 0 \mid \forall x\in\X : \sum_{a\in\A} \omega(x,a) = 1\}$ is the probability simplex. We denote by $\omega^\star_{z,\theta^\star}$ the optimal solution of~\eqref{eq:optim-lb-z} and $u^\star(z, \theta^\star)$ its associated value (if the problem is unfeasible we set $u^\star(z, \theta^\star) = +\infty$). Inspecting~\eqref{eq:optim-lb-z}, we notice that $z$ serves as a global constraint on the number of samples. In fact, for any $\omega\in\Omega$, the associated number of samples $\eta(x,a)$ allocated to a context-arm pair $(x,a)$ is now $z \rho(x) \omega(x,a)$. Since $\rho$ is a distribution over $\X$ and $\sum_a \omega(x,a) = 1$ in each context, the total number of samples sums to $z$. As a result,~\eqref{eq:optim-lb-z} admits a minimum and it is more amenable to designing a learning algorithm based on its Lagrangian relaxation. Furthermore, we notice that $z$ can be interpreted as defining a more ``finite-time'' formulation of the lower bound. Finally, we remark that the total number of samples that can be assigned to a context $x$ is indeed constrained to $z\rho(x)$. This constraint crucially makes~\eqref{eq:optim-lb-z} more context aware and forces the solution $\omega$ to be more adaptive to the context distribution. 
In Sect.~\ref{sec:algorithm}, we leverage these features to design an incremental algorithm whose finite-time regret does not depend on $\rho_{\min}$, thus improving over previous algorithms~\cite{lattimore2017end,hao2019adaptive}, as supported by the empirical results in Sect.~\ref{sec:experiments}.
The following lemma provides a characterization of~\eqref{eq:optim-lb-z} and its relationship with~\eqref{eq:optim-lb} (see App.~\ref{app:lower.bound} for the proof and further discussion).
\begin{lemma}\label{lem:z-opt-upperbound}
        Let $\underline{z}(\theta^\star) := \min\left\{z > 0 : \text{\eqref{eq:optim-lb-z} is feasible}\right\}$, $\wb{z}(\theta^\star) = \max_{x \in \X} \sum_{a \neq a^\star_{\theta^\star}(x)} \frac{\eta^\star(x,a)}{\rho(x)}$ and $\sumpull(\theta^\star):= \sum_{x\in\X}\sum_{a \neq a^\star_{\theta^\star}(x)} \eta^\star(x,a)$.
Then
%
$\frac{1}{\underline{z}(\theta^\star)} = \max_{\omega\in\Omega}\inf_{\theta' \in \Theta_{\mathrm{alt}}}\E_{\rho}\left[\sum_{a\in\A}{\omega}(x,a)d_{x,a}(\theta^\star,\theta') \right]$
%
and there exists a constant $c_{\Theta}>0$ such that, for any $z \in (\underline{z}(\theta^\star), +\infty)$,
\begin{align*}
        u^\star(z,\theta^\star) \leq v^\star(\theta^\star) + \frac{2 zBL\underline{z}(\theta^\star)}{z - \underline{z}(\theta^\star)} \cdot \begin{cases} 1 
         & \text{if } z < \wb{z}(\theta^\star) \\
        \min\left\{
            \max\left\{
                \frac{c_\Theta \sqrt{2} \sumpull(\theta^\star)}{\sigma\sqrt{z}}, \frac{\sumpull(\theta^\star)}{ z}
            \right\}
        , 1\right\}& \text{otherwise} \end{cases}
\end{align*}
\end{lemma}
The first result characterizes the range of $z$ for which~\eqref{eq:optim-lb-z} is feasible. Interestingly, $\underline{z}(\theta^\star) < +\infty$
is the inverse of the sample complexity of the best-arm identification problem \citep{degenne2020gamification} and the associated solution is the one that maximizes the amount of information gathered about the reward model $\theta^\star$. As $z$ increases, $\omega^\star_{z,\theta^\star}$ becomes less aggressive in favoring informative context-arm pairs and more sensitive to the regret minimization objective. 
The second result quantifies the bias w.r.t.\ the optimal solution of~\eqref{eq:optim-lb-z}. 
For $z \geq \wb{z}(\theta^\star)$, the error decreases approximately at a rate $1/\sqrt{z}$ showing that the solution of~\eqref{eq:optim-lb-z} can be made arbitrarily close to $v^\star(\theta^\star)$.

In designing our learning algorithm, we build on the Lagrangian relaxation of ~\eqref{eq:optim-lb-z}. For any $\omega\in\Omega$, let $f(\omega; \theta^\star)$ denote the objective function and $g(\omega, z; \theta^\star)$ denote the KL constraint
\begin{align*}
f(\omega; \theta^\star) = \E_{\rho}\Big[\sum_{a\in\A}\omega(x,a)\mu_{\theta^\star}(x,a)\Big], \enspace g(\omega; z, \theta^\star) = \!\!\inf_{\theta' \in \Theta_{\mathrm{alt}}}\!\E_{\rho}\Big[\sum_{a\in\A}\omega(x,a)d_{x,a}(\theta^\star,\theta') \Big] - \frac{1}{z}.
\end{align*}
We introduce the Lagrangian relaxation problem
\begin{align}\label{eq:lagr-rel}\tag{P$_\lambda$}
\min_{\lambda \geq 0} \max_{\omega \in \Omega} \Big\{ h(\omega, \lambda; z, \theta^\star) := f(\omega; \theta^\star) + \lambda g(\omega; z,\theta^\star) \Big\},
\end{align}
where $\lambda\in\Re_{\geq 0}$ is a multiplier. Notice that $f(\omega;\theta^\star)$ is not equal to the objective function of~\eqref{eq:optim-lb-z}, since we replaced the gap $\Delta_{\theta^\star}$ by the expected value $\mu_{\theta^\star}$ and we removed the constant multiplicative factor $z$ in the objective function. The associated problem is thus a concave maximization problem. While these changes do not affect the optimality of the solution, they do simplify the algorithmic design.
Refer to App.~\ref{app:lagrangian} for details about the Lagrangian formulation. 


\section{Asymptotically Optimal Linear Primal Dual Algorithm}\label{sec:algorithm}

%

\begin{wrapfigure}[30]{r}{0.52\textwidth}
        \vspace{-.2in}
        \begin{algorithm}[H]
        \caption{\algo} \label{alg:primal-dual-alg-phased-nostop}
        \vspace{-.1in}
            \begin{tcolorbox}[boxsep=1pt,left=2pt,right=2pt,top=0pt,bottom=0pt,
                              colframe=black, colback=white,boxrule=0.5pt,arc=0pt]
                    \DontPrintSemicolon
                    \footnotesize
                    \SetInd{0.4em}{0.9em}
                    \KwIn{Multiplier $\lambda_1$, confidence values $\{\beta_t\}_t$ and $\{\gamma_t\}_t$, maximum multiplier $\lambda_{\mathrm{max}}$, normalization factors $\{z_k\}_{k\geq 0}$, phase lengths $\{p_k\}_{k\geq 0}$, step sizes $\alpha_k^\lambda, \alpha_k^\omega$}
                    \vspace{4pt}
                    Set $\omega_1 \leftarrow \frac{\boldsymbol{1}_{\X\A}}{|\A|}$, $\wb V_0 \leftarrow \nu \boldsymbol{I}$, $U_0 \leftarrow \boldsymbol{0}$, $\wt{\theta}_0 \leftarrow \boldsymbol{0}$,
                    $S_0 \leftarrow 0$\;
                    Phase index: $K_1 \leftarrow 0$\;
                    \For{$t=1,\ldots, n$}{
                        Receive context $X_t \sim \rho$\;
                        Set $K_{t+1} \leftarrow K_t$\;
                        \eIf{$\inf_{\theta'\in \wb{\Theta}_{t-1}} \| \wt{\theta}_{t-1} - \theta' \|_{\wb{V}_{t-1}}^2 > \beta_{t-1}$}{
                                {\color{MidnightBlue}\textsc{// Exploitation Step}}\;
                                $A_t \leftarrow \argmax_{a\in\A}{\mu}_{\wt{\theta}_{t-1}}(X_t,a)$\;
                                $\lambda_{t+1} \leftarrow \lambda_t$, $\omega_{t+1} \leftarrow \omega_t$
                        }{
                                {\color{MidnightBlue}\textsc{// Exploration Step}}\;
                                Sample arm: $A_t \sim {\omega}_{t}(X_t, \cdot)$\;
                                Set $S_t \leftarrow S_{t-1} + 1$\;
                                {\color{MidnightBlue!80!white}\scriptsize\textsc{// Update Solution}}\;
                                Compute $q_t \in \partial h_t(\omega_t, \lambda_t, z_{K_t})$ (see Eq.~\ref{eq:pdlin.ht})\;
                                Update policy\\ 
                                $\omega_{t+1}(x,a) \leftarrow \frac{\omega_{t}(x,a)e^{\alpha_{K_t}^\omega q_t(x,a)}}{\sum_{a'\in\A}\omega_{t}(x,a')e^{\alpha_{K_t}^\omega q_t(x,a')}}$\;
                                Update multiplier\\
                                ${\lambda_{t+1} \leftarrow \min\{ [\lambda_t - \alpha_{K_t}^\lambda g_t(\omega_t, z_{K_t})]_+ , \lambda_{\mathrm{max}}\}}$
                                {\color{MidnightBlue!80!white}\scriptsize\textsc{// Phase Stopping Test}}\;
                                \If{$S_t - S_{T_{K_t} - 1} = p_k$}{
                                    Change phase: $K_{t+1} \leftarrow K_t + 1$\;
                                    Reset solution: $\omega_{t+1} \leftarrow \omega_1$, $\lambda_{t+1} \leftarrow \lambda_1$
                                }
                        }
                        Pull $A_t$ and observe outcome $Y_t$\;
                        Update $\wb{V}_{t}$, $U_t$, $\wh\theta_t$, $\wh\rho_t$ using $X_t, A_t, Y_t$\;
                        \red{Set $\wt{\theta}_t := \argmin_{\theta \in \Theta \cap \mathcal{C}_t}\| \theta - \wh{\theta}_t \|_{\wb{V}_t}$}
                    }
            \end{tcolorbox}
            \end{algorithm}
\end{wrapfigure}
We introduce \algo
(aSymptotic Optimal Linear prImal Dual),
which combines a primal-dual approach to incrementally compute the solution of an optimistic estimate of the Lagrangian relaxation~\eqref{eq:lagr-rel} within a scheme that, depending on the accuracy of the estimate $\wh\theta_t$, separates \emph{exploration} steps, where arms are pulled according to the exploration policy $\omega_t$, and \emph{exploitation} steps, where the greedy arm is selected. The values of the input parameters for which \algo enjoys regret guarantees are reported in Sect.~\ref{sec:regret}. In the following, we detail the main ingredients composing the algorithm (see Alg.~\ref{alg:primal-dual-alg-phased-nostop}).

\textbf{Estimation.} \algo stores and updates the regularized least-square estimate $\wh\theta_t$ using all samples observed over time. 
\red{To account for the fact that $\wh\theta_t$ may have large norm (i.e., $\|\wh\theta_t\|_2 > B$ and $\wh\theta_t \notin \Theta$), \algo explicitly projects $\wh\theta_t$ onto $\Theta$. Formally, let $\mathcal{C}_t := \{\theta \in \R^d : \| \theta - \wh{\theta}_t\|_{\wb{V}_{t}}^2 \leq \beta_{t}\}$ be the confidence ellipsoid at time $t$. Then, \algo computes
$\wt{\theta}_t := \argmin_{\theta \in \Theta \cap \mathcal{C}_t}\| \theta - \wh{\theta}_t \|_{\wb{V}_t}^2$.
This is a simple convex optimization problem, though it has no closed-form expression.\footnote{The projection is required to carry out the analysis, while we ignore it in our implementation (see App. \ref{app:implementation-details}).} Note that, on those steps where $\theta^\star \notin \mathcal{C}_t$, $\Theta \cap \mathcal{C}_t$ might be empty, in which case we can set $\wt{\theta}_t = \wt{\theta}_{t-1}$. Then, \algo uses $\wt{\theta}_{t}$ instead of $\wh{\theta}_{t}$ in all steps of the algorithm.}
\algo also computes an empirical estimate of the context distribution as $\wh\rho_{t}(x) = \frac{1}{t}\sum_{s=1}^{t} \indi{X_s = x}$.

\textbf{Accuracy test and tracking.} Similar to previous algorithms leveraging asymptotic lower bounds, we build on the generalized likelihood ratio test \citep[e.g.,][]{degenne2019non} to verify the accuracy of the estimate $\wh\theta_t$. At the beginning of each step $t$, \algo first computes $\inf_{\theta'\in\wb{\Theta}_{t-1}} \| \wt{\theta}_{t-1} - \theta' \|_{\wb{V}_{t-1}}^2$,
where $\wb{\Theta}_{t-1} = \{ \theta' \in \Theta\ |\ \exists x\in\X,\ a^\star_{\wt\theta_{t-1}}(x) \neq a^\star_{\theta'}(x) \}$ is the set of alternative models. This quantity measures the accuracy of the algorithm, where the infimum over alternative models defines the problem $\theta'$ that is closest to $\wt\theta_{t-1}$ and yet different in the optimal arm of at least one context.\footnote{In practice, it is more efficient to take the infimum only over problems with different optimal arm in the last observed context $X_t$. This is indeed what we do in our experiments and all our theoretical results follow using this alternative definition with only minor changes.} This serves as a worst-case scenario for the true $\theta^\star$, since if $\theta^* = \theta'$ then selecting arms according to $\wt\theta_{t-1}$ would lead to linear regret. If the accuracy exceeds a threshold $\beta_{t-1}$, then \algo performs an exploitation step, 
where the estimated optimal arm $a^\star_{\wt\theta_{t-1}}(X_t)$ is selected in the current context. On the other hand, if the test fails, the algorithm moves to an exploration step, where an arm $A_t$ is sampled according to the estimated exploration policy $\omega_{t}(X_t,\cdot)$. While this approach is considerably simpler than standard tracking strategies (e.g., selecting the arm with the largest gap between the policy $\omega_t$ and the number of pulls), in Sect.~\ref{sec:regret} we show that sampling from $\omega_t$ achieves the same level of tracking efficiency.

\textbf{Optimistic primal-dual subgradient descent.} At each step $t$, we define an estimated optimistic version of the Lagrangian relaxation~\eqref{eq:lagr-rel} as
%
{\fontsize{9.5pt}{11.4pt}
\begin{align}
\label{eq:pdlin.ft}
f_t(\omega) &:= \sum_{x\in\X} \wh{\rho}_{t-1}(x) \sum_{a\in\A} \omega(x,a) \left( {\mu}_{\wt{\theta}_{t-1}}(x,a) +  \sqrt{\gamma_t} \| \phi(x,a)\|_{\wb{V}_{t-1}^{-1}} \right),\\
\label{eq:pdlin.gt}
g_t(\omega, z) &:= \inf_{\theta' \in \wb{\Theta}_{t-1}}\sum_{x\in\X}\wh{\rho}_{t-1}(x)\sum_{a\in\A}\omega(x,a)\left( {d}_{x,a}(\wt{\theta}_{t-1},\theta') + \frac{2BL}{\sigma^2}\sqrt{\gamma_t}\| \phi(x,a)\|_{\wb{V}_{t-1}^{-1}} \right) - \frac{1}{z},\\
\label{eq:pdlin.ht}
h_t(\omega, \lambda, z) &:= f_t(\omega) + \lambda g_t(\omega,z),
\end{align}}
where $\gamma_t$ is a suitable parameter defining the size of the confidence interval.

Notice that we do not use optimism on the context distribution, which is simply replaced by its empirical estimate. Therefore, $h_t$ is not necessarily optimistic with respect to the original Lagrangian function $h$. Nonetheless, we prove in Sect.~\ref{sec:regret} that this level of optimism is sufficient to induce enough exploration to have accurate estimates of $\theta^\star$. This is in contrast with the popular forced exploration strategy \citep[e.g.][]{lattimore2017end,combes2017minimal,ok2018exploration,hao2019adaptive}, which prescribes a minimum fraction of pulls $\epsilon$ such that at any step $t$, any of the arms with less than $\epsilon S_t$ pulls is selected, where $S_t$ is the number of exploration rounds so far. While this strategy is sufficient to guarantee a minimum level of accuracy for $\wh\theta_t$ and to obtain asymptotic regret optimality, in practice it is highly inefficient as it requires selecting all arms in each context regardless of their value or  amount of information.

At each step $t$, \algo updates the estimates of the optimal exploration policy  $\omega_t$ and the Lagrangian multiplier $\lambda_t$. In particular, given the sub-gradient $q_t$ of $h_t(\omega_t, \lambda_t, z_{K_t})$, 
\algo updates $\omega_t$ and $\lambda_t$ by performing one step of projected sub-gradient descent with suitable learning rates $\alpha^\omega_{K_t}$ and $\alpha^\lambda_{K_t}$. In the update of $\omega_t$, we perform the projection onto the simplex $\Omega$ using an entropic metric, while the multiplier is clipped in $[0,\lambda_{\max}]$. While this is a rather standard primal-dual approach to solve the Lagrangian relaxation~\eqref{eq:lagr-rel}, the interplay between estimates $\wh\theta_t$, $\rho_t$, the optimism used in $h_t$, and the overall regret performance of the algorithm is at the core of the analysis in Sect.~\ref{sec:regret}.

This approach significantly reduces the computational complexity compared to~\citep{combes2017minimal,hao2019adaptive}, which require solving problem~\ref{eq:optim-lb} at each exploratory step. 
In Sect.~\ref{sec:experiments}, we show that the incremental nature of \algo allows it to scale to problems with much larger context-arm spaces. Furthermore, we leverage the convergence rate guarantees of the primal-dual gradient descent to show that the incremental nature of \algo does not compromise the asymptotic optimality of the algorithm (see Sect.~\ref{sec:regret}).

\textbf{The $z$ parameter.}  
While the primal-dual algorithm is guaranteed to converge to the solution of~\eqref{eq:optim-lb-z} for any fix $z$, it may be difficult to properly tune $z$ to control the error w.r.t.\ \eqref{eq:optim-lb}.
\algo leverages the fact that the error scales as $1/\sqrt{z}$ (Lem.~\ref{lem:z-opt-upperbound} for $z$ sufficiently large) and it increases $z$ over time.
Given as input two non-decreasing sequences $\{p_k\}_k$ and $\{z_k\}_k$, at each phase $k$, \algo uses $z_k$ in the computation of the subgradient of $h_t$ and in the definition of $f_t$ and $g_t$. After $p_k$ explorative steps, it resets the policy $\omega_t$ and the multiplier $\lambda_t$ and transitions to phase $k+1$. Since $p_k = S_{T_{k+1}-1} - S_{T_k-1}$ is the number of \textit{explorative} steps of phase $k$ starting at time $T_k$, the actual number of steps during $k$ may vary. Notice that at the end of each phase only the optimization variables are reset, while the learning variables (i.e., $\wh\theta_t$, $\wb{V}_t$, and $\wh\rho_t$) use all the samples collected through phases. 

\section{Regret Analysis}\label{sec:regret}
Before reporting the main theoretical result of the paper, we introduce the following  assumption.
\begin{assumption}\label{asm:lambda_max}
        The maximum multiplier used by \algo is such that $\lambda_{\max} \geq 2 BL \underline{z}(\theta^\star)$.
\end{assumption}
While an assumption on the maximum multiplier is rather standard for the analysis of primal-dual projected subgradient~\cite[e.g.,][]{nedic2009subgradient,efroni2020exploration}, we conjecture that it may be actually relaxed in our case by replacing the fixed $\lambda_{\max}$ by an increasing sequence as done for $\{z_k\}_k$.
\begin{theorem}\label{thm:regret.bound}
Consider a contextual linear bandit problem with contexts $\X$, arms $\A$, reward parameter $\theta^\star$, features bounded by $L$, zero-mean Gaussian noise with variance $\sigma^2$ and context distribution $\rho$ satisfying Asm.~\ref{asm:regularity}.
If \algo is run with confidence values $\beta_{t-1} = c_{n,1/n}$ and $\gamma_t = c_{n,1/S_t^2}$, where $c_{n,\delta}$ is defined as in Thm.~\ref{th:conf-theta}, learning rates $\alpha_k^\lambda = \alpha_k^\omega = 1/\sqrt{p_k}$ and increasing sequences $z_k = z_0 e^k$ and $p_k = z_k e^{2k}$, for some $z_0 \geq 1$, then it is \textbf{asymptotically optimal} with the same constant as in the lower bound of Prop.~\ref{p:lower.bound}. 
%
%
Furthermore, for any finite $n$ the regret of \algo is bounded as
\begin{align}
        \label{eq:regret.bound.finite}
        \E_{\xi,\rho}^\pi \big[ R_n(\theta^\star) \big]
        &\leq v^\star(\theta^\star) \frac{c_{n,1/n}}{2\sigma^2} + C_{\log} (\log\log n)^{\frac{1}{2}}(\log n)^{\frac{3}{4}} + C_{\mathrm{const}},
\end{align}
where 
$C_{\log} = lin_{\geq 0}(v^\star(\theta^\star),|\X|,L^2,B^2,\sqrt{d},1/\sigma^2)$ and $C_{\mathrm{const}} = v^\star(\theta^\star)\frac{B^2L^2}{\sigma^2}+ lin_{\geq 0 }(L,B, z_0 (\sfrac{\underline{z}(\theta^\star)}{z_0})^{3}, (\sfrac{\wb{z}(\theta^\star)}{z_0})^2)$.\footnote{\emph{lin}($\cdot$) denotes any function with linear or sublinear dependence on the inputs (ignoring logarithmic terms). For example, $\textit{lin}_{\geq 0}(x,y^2) \in \{ a_0 + a_1 x + a_2 y + a_3 y^2 + a_4 xy^2 : a_i \geq 0\}$.} 
\end{theorem}

The first result shows that \algo run with an exponential schedule for $z$ is asymptotic optimal, while the second one provides a bound on the finite-time regret.
%
We can identify three main components in the finite-time regret.
\textbf{1)} The first term scales with the logarithmic term $c_{n,1/n}=O({\log n + d \log\log n})$ and a leading constant $v^\star(\theta^\star)$, which is optimal as shown in Prop.~\ref{p:lower.bound}. 
In most cases, this is the dominant term of the regret. 
\textbf{2)} Lower-order terms in $o(\log n)$. Notably, a regret of order $\sqrt{\log n}$ is due to the incremental nature of \algo and it is directly inherited from the convergence rate of the primal-dual algorithm we use to optimize~\eqref{eq:optim-lb-z}. The larger term $(\log n)^{3/4}$ that we obtain in the final regret is actually due to the schedule of $\{z_k\}$ and $\{p_k\}$. While it is possible to design a different phase schedule to reduce the exponent towards $1/2$, this would negatively impact the constant regret term. 
%
%
\textbf{3)} The constant regret $C_{\mathrm{const}}$ is due to the exploitation steps, burn-in phase and the initial value $z_0$.
The regret due to $z_0$ takes into account the regime when~\eqref{eq:optim-lb-z} is unfeasible ($z_k < \underline{z}(\theta^\star)$) or when $z_k$ is too small to assess the rate at which $u^\star(z_k, \theta^\star)$ approaches $v^\star(\theta^\star)$ ($z < \wb{z}(\theta^\star)$), see Lem.~\ref{lem:z-opt-upperbound}. Notably, the regret due to the initial value $z_0$ vanishes when $z_0 > \wb{z}(\theta^\star)$. A more aggressive schedule for $z_k$ reaching $\wb{z}(\theta^\star)$ in few phases would reduce the initial regret at the cost of a larger exponent in the sub-logarithmic terms.

The sub-logarithmic terms in the regret have only logarithmic dependency on the number of arms. This is better than existing algorithms based on exploration strategies built from lower bounds. OSSB~\citep{combes2017minimal} indeed depends on $|\A|$ directly in the main $\mathcal{O}(\log n)$ regret terms. While the regret analysis of OAM is asymptotic, it is possible to identify several lower-order terms depending linearly on $|\A|$. In fact, OAM as well as OSSB require forced exploration on each context-arm pair, which inevitably translates into regret. In this sense, the dependency on $|\A|$ is hard-coded into the algorithm and cannot be improved by a better analysis. SPL depends linearly on $|\A|$ in the explore/exploit threshold (the equivalent of our $\beta_t$) and in other lower-order terms due to the analysis of the tracking rule. On the other hand, \algo never requires all arms to be repeatedly pulled and we were able to remove the linear dependence on $|\A|$ through a refined analysis of the sampling procedure (see App.~\ref{app:sampling}). This is inline with the experimental results where we did not notice any explicit linear dependence on $|\A|$.


The constant regret term depends on the context distribution through $\wb{z}(\theta^\star)$ (Lem.~\ref{lem:z-opt-upperbound}). Nonetheless, this dependency disappears whenever $z_0$ is a fraction $\wb{z}(\theta^\star)$. This is in striking contrast with OAM, whose analysis includes several terms depending on the inverse of the context probability $\rho_{\min}$. This confirms that \algo is able to better adapt to the distribution generating the contexts.
%
%
While the phase schedule of Thm.~\ref{thm:regret.bound} leads to an asymptotically-optimal algorithm and sublinear-regret in finite time, it may be possible to find a different schedule having the same asymptotic performance and better finite-time guarantees, although this may depend on the horizon $n$. Refer to App.~\ref{app:regret.exploration} for a regret bound highlighting the explicit dependence on the sequences $\{z_k\}$ and $\{p_k\}$.

As shown in~\citep{hao2019adaptive}, when the features of the optimal arms span $\R^d$, the asymptotic lower bound vanishes (i.e., $v^\star(\theta^\star) =0$). In this case, selecting optimal arms is already informative enough to correctly estimate $\theta^\star$ and no explicit exploration is needed and \algo, like OAM, has sub-logarithmic regret.

\textbf{Worst-case analysis.} The constant terms in Thm.~\ref{thm:regret.bound} are due to a naive bound which assumes linear regret in those phases where $z_k$ is small (e.g., when the optimization problem is infeasible). While this simplifies the analysis for asymptotic optimality, we verify that \algo always suffers sub-linear regret, regardless of the values of $z_k$. For the following result, we do not require Asm.~\ref{asm:lambda_max} to hold.

\begin{theorem}[Worst-case regret bound]\label{th:worst-case}
Let $z_k$ be arbitrary, $p_k = e^{rk}$ for some constant $r \geq 1$, and the other parameters be the same as in Thm.~\ref{thm:regret.bound}. Then, for any $n$ the regret of \algo is bounded as
{\begin{align*}
\E_{\xi,\rho}^\pi \big[ R_n(\theta^\star) \big] \leq 12BL\pi^2C_{\lambda_{\max}} + \frac{2e^r\left(\lambda_{\max}^2 + \log|\A|\right)}{r}\sqrt{n} + C_{\mathrm{sqrt}} C_{\lambda_{\max}}\log (n) \sqrt{dn},
\end{align*}
where $C_{\mathrm{sqrt}} = lin_{\geq 0}(|\X| + \sqrt{d}, B, L)$ and $C_{\lambda_{\max}} := \left(1 + \frac{\lambda_{\max}BL}{\sigma^2}\right)$.}
\end{theorem}
Notably, this bound removes the dependencies on $\underline{z}(\theta^\star)$ and $\wb{z}(\theta^\star)$, while its derivation is agnostic to the values of $z_k$. Interestingly, we could set $\lambda_{\max} = 0$ and the algorithm would completely ignore the KL constraint, thus focusing only on the objective function. This is reflected in the worst-case bound since all terms with a dependence on $\sigma^2$ or a quadratic dependence on $BL$ disappear. The key result is that the objective function alone, thanks to optimism, is sufficient for proving sub-linear regret but not for proving asymptotic optimality. { More precisely, the bound is $\wt{\mathcal{{O}}}((\sqrt{d} + |\X|)\sqrt{dn} + \log|\A|\sqrt{n})$, which matches the minimax optimal rate apart from the dependence on $|\X|$ (see \cite{lattimore2020bandit}, Sec. 24.1). We believe the latter could be reduced to $\sqrt{|\X|}$ by a refined analysis. It remains an open question how to design an asymptotically optimal algorithm for the contextual case whose regret does not scale with $|\X|$.}


\begin{figure*}[t!]
	\centering
	\includegraphics[height=4.0cm]{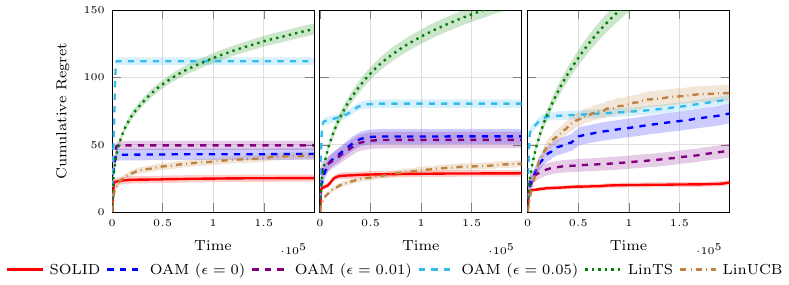}\\
	\caption{Toy problem with 2 contexts and (left) $\rho(x_1) =.5$, (center) $\rho(x_1) =.9$, (right) $\rho(x_1)=.99$.}
	\label{fig:two_cont}
	\vspace{-.1in}
\end{figure*}

\section{Numerical Simulations}\label{sec:experiments}

We compare \algo to LinUCB, LinTS, and OAM. For \algo, we set $\beta_t = \sigma^2 (\log(t) + d\log\log(n))$ and $\gamma_t = \sigma^2 (\log(S_t) + d\log\log(n))$ (i.e., we remove all numerical constants) and we use the exponential schedule for phases defined in Thm.~\ref{thm:regret.bound}. For OAM, we use the same $\beta_t$ for the explore/exploit test and we try different values for the forced-exploration parameter $\epsilon$. LinUCB uses the confidence intervals from Thm. 2 in~\cite{abbasi2011improved} with the log-determinant of the design matrix, and LinTS is as defined in~\cite{agrawal2013thompson} but without the extra-sampling factor $\sqrt{d}$ used to prove its frequentist regret. 
All plots are the results of $100$ runs with $95\%$ Student's t confidence intervals. See App.~\ref{app:experiments} for additional details and results on a real dataset.

\begin{figure*}[t]
    \centering
    \includegraphics[height=4.0cm]{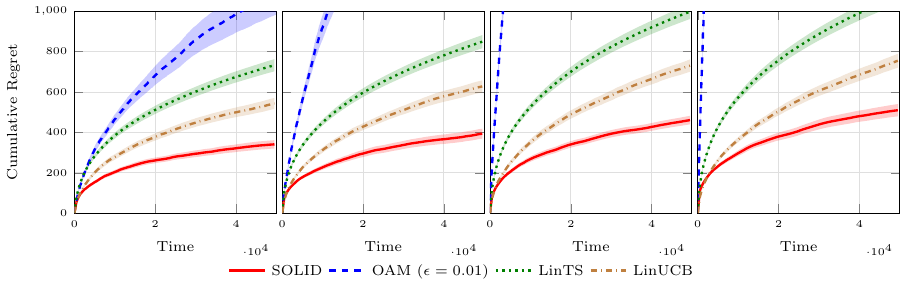}
    \caption{Randomly generated bandit problems with $d=8, |\X| =4$, and $|\A|=4,8,16,32$.}\label{fig:random}
    \vspace{-0.1in}
\end{figure*}

\textbf{Toy contextual linear bandit with structure.} We start with a CLB problem with $|\X|=2$ and $|\A|,d=3$. Let $x_i$ ($a_i$) be the $i$-th context (arm). We have $\phi(x_1,a_1) = [1, 0, 0]$, $\phi(x_1,a_2) = [0, 1, 0]$, $\phi(x_1,a_3) = [1-\xi, 2\xi, 0]$, $\phi(x_2,a_1) = [0, 0.6, 0.8]$, $\phi(x_2,a_2) = [0, 0, 1]$, $\phi(x_2,a_3) = [0, \xi/10, 1-\xi]$ and $\theta^\star = [1, 0, 1]$. We consider a balanced context distribution $\rho(x_1) = \rho(x_2) = 0.5$. This is a two-context counterpart of the example presented by \citep{lattimore2017end} to show the asymptotic sub-optimality of optimism-based strategies. The intuition is that, for $\xi$ small, an optimistic strategy pulls $a_2$ in $x_1$ and $a_1$ in $x_2$ only a few times since their gap is quite large, and suffers high regret (inversely proportional to $\xi$) to figure out which of the remaining arms is optimal. On the other hand, an asymptotically optimal strategy allocates more pulls to ``bad" arms as they bring information to identify $\theta^\star$, which in turns avoids a regret scaling with $\xi$. This indeed translates into the empirical performance reported in Fig.~\ref{fig:two_cont}-(left), where \algo effectively exploits the structure of the problem and significantly reduces the regret compared to LinTS and LinUCB. Actually, not only the regret is smaller but the ``trend'' is better. In fact, the regret curves of LinUCB and LinTS have a larger slope than \algo's, suggesting that the gap may increase further with $n$, thus confirming the theoretical finding that the asymptotic performance of \algo is better. OAM has a similar behavior, but the actual performance is worse than \algo and it seems to be very sensitive to the forced exploration parameter, where the best performance is obtained for $\epsilon = 0.0$, which is not theoretically justified.

We also study the influence of the context distribution. We first notice that solving~\eqref{eq:optim-lb} leads to an optimal exploration strategy $\eta^\star$ where the only sub-optimal arm with non-zero pulls is $a_1$ in $x_2$ since it yields lower regret and similar information than $a_2$ in $x_1$. This means that the lower bound prescribes a greedy policy in $x_1$, deferring exploration to $x_2$ alone. In practice, tracking this optimal allocation might lead to poor finite-time performance when the context distribution is unbalanced towards $x_1$,  in which case the algorithm would take time proportional to $1/\rho(x_2)$ before performing any meaningful exploration. We verify these intuitions empirically by considering the case of $\rho(x_1) = 0.9$ and $\rho(x_1) = 0.99$ (middle and right plots in Fig.~\ref{fig:two_cont} respectively). \algo is consistently better than all other algorithms, showing that its performance is not negatively affected by $\rho_{\min}$. On the other hand, OAM is more severely affected by the context distribution. In particular, its performance with $\epsilon=0$ significantly decreases when increasing $\rho(x_1)$ and the algorithm reduces to an almost greedy strategy, thus suffering linear regret in some problems. In this specific case, forcing exploration leads to slightly better finite-time performance since the algorithm pulls the informative arm $a_2$ in $x_1$, which is however not prescribed by the lower bound.

\textbf{Random problems.}
We evaluate the impact of the number of actions $|\A|$ in randomly generated structured problems with $d=8$ and $|\X| = 4$. 
%
We run each algorithm for $n=50000$ steps. For OAM, we set forced-exploration $\epsilon=0.01$ and solve~\eqref{eq:optim-lb} every $100$ rounds to speed-up execution as computation becomes prohibitive. The plots in Fig.~\ref{fig:random} show the regret over time for $|\A|=4,8,16,32$. This test confirms the advantage of \algo over the other methods. Interestingly, the regret of \algo does not seem to significantly increase as a function of $|\A|$, thus supporting its theoretical analysis. On the other hand, the regret of OAM scales poorly with $|\A|$ since forced exploration pulls all arms in a round robin fashion.



\section{Conclusion}


We introduced \algo, a novel asymptotically-optimal algorithm for contextual linear bandits with finite-time regret and computational complexity improving over similar methods and better empirical performance w.r.t.\ state-of-the-art algorithms in our experiments. The main open question is whether \algo is minimax optimal for contextual problems with $|\X| > \sqrt{d}$. In future work, our method could be extended to continuous contexts, which would probably require a reformulation of the lower bound and the adoption of parametrized policies. Furthermore, it would be interesting to study finite-time lower bounds, especially for problems in which bounded regret is achievable \cite{lattimore2014bounded, bastani2017mostly, tirinzoni2020novel}. Finally, we could use algorithmic ideas similar to \algo to go beyond the realizable linear bandit setting.

\section*{Broader Impact}

This work is mainly a theoretical contribution. We believe it does not present any foreseeable societal consequence.

\section*{Funding Transparency Statement}

Marcello Restelli was partially funded by the Italian MIUR PRIN 2017 Project ALGADIMAR ``Algorithms, Games, and Digital Market''.

\section*{Acknowledgements}

The authors would like to thank R\'emy Degenne, Han Shao, and Wouter Koolen for kindly sharing the draft of their paper before publication. 
\red{We also would like to thank Pierre M\'{e}nard for carefully reading the paper and for providing insightful feedback.}

\bibliographystyle{unsrt}
\bibliography{biblio}

\onecolumn
\appendix

\part{Appendix}


\parttoc
\newpage
\section{Notation and Definitions}\label{app:notation.definition}

We provide this table for easy reference. Notation will also be defined as it is introduced.

{
\footnotesize
\renewcommand{\arraystretch}{1.1}

\begin{longtable}{l L{6.5cm}}
\caption{Symbols}\\
\hline
$\theta^\star$ & The true reward parameter
\\
$\X$ & Finite set of contexts
\\
$\A$ & Finite set of arms
\\
$\sigma^2$ & Variance of the Gaussian reward noise
\\
$B$ & Maximum $l_2$-norm of realizable reward parameters
\\
$L$ & Maximum $l_2$-norm of the features
\\
$\rho$ & Context distribution
\\
$\wh\rho_{t}(x) := \frac{1}{t}\sum_{s=1}^{t} \indi{X_s = x}$ & Estimated context distribution
\\
$\mu_\theta(x,a)$ & Mean reward of context $x$ and arm $a$
\\
$\Delta_\theta(x,a) := \max_{a'\in\A}\mu_\theta(x,a') - \mu_\theta(x,a)$ & Gap of context $x$ and arm $a$
\\
$a^\star_{\theta}(x) := \argmax_{a\in\A}\mu_{\theta}(x,a)$ & Optimal arm of context $x$
\\
$\mu^\star_{\theta}(x) := \max_{a\in\A}\mu_{\theta}(x,a)$ & Optimal reward value of context $x$
\\
${d}_{x,a}(\theta,\theta') := \frac{1}{2\sigma^2} ({\mu}_\theta(x,a) - \mu_{\theta'}(x,a))^2$ & KL divergence between $\theta$ and $\theta'$ at $x,a$
\\
$\Theta_{\mathrm{alt}} := \{ \theta' \in \Theta\ |\ \exists x\in\X,\ a^\star_{\theta^\star}(x) \neq a^\star_{\theta'}(x) \}$ & Set of alternative reward models
\\
$\wb{\Theta}_{t-1} = \{ \theta' \in \Theta\ |\ \exists x\in\X,\ a^\star_{\wt\theta_{t-1}}(x) \neq a^\star_{\theta'}(x) \}$ & Estimated set of alternative reward models
\\
$v^\star(\theta^\star)$ & Optimal value of the optimization problem \eqref{eq:optim-lb}
\\
$\eta^\star$ & Optimal solution of the optimization problem \eqref{eq:optim-lb}
\\
$u^\star(z, \theta^\star)$ & Optimal value of the optimization problem \eqref{eq:optim-lb-z}
\\
$\omega^\star_{z,\theta^\star}$ & Optimal solution of the optimization problem \eqref{eq:optim-lb-z}
\\
$\underline{z}(\theta^\star) := \min\left\{z > 0 : \text{\eqref{eq:optim-lb-z} is feasible}\right\}$ & Feasibility threshold of \eqref{eq:optim-lb-z}
\\
$h(\omega, \lambda; z, \theta^\star) := f(\omega; \theta^\star) + \lambda g(\omega; z,\theta^\star)$ & Lagrangian relaxation of \eqref{eq:optim-lb-z}
\\
$f(\omega; \theta^\star)$ & Objective function
\\
$f_t(\omega)$ & Estimated (optimistic) objective function (see Eq.~\ref{eq:pdlin.ft})
\\
$g(\omega; z,\theta^\star)$ & Constraint function
\\
$g_t(\omega, z)$ & Estimated (optimistic) constraint (see Eq.~\ref{eq:pdlin.gt})
\\
$E_t := \indi{\inf_{\theta'\in\wb{\Theta}_{t-1}} \| \wt{\theta}_{t-1} - \theta' \|_{\wb{V}_{t-1}}^2 \leq \beta_{t-1} }$ & Exploration round
\\
$N_{t}(x,a):=\sum_{s=1}^t \indi{X_t=x,A_t=a}$ & Total number of visits to $(x,a)$
\\
$N_{t}^E(x,a):=\sum_{s=1}^t \indi{X_t=x,A_t=a,E_t}$ & Number of visits to $(x,a)$ in exploration rounds
\\
$S_t:=\sum_{s=1}^t \indi{E_t}$ & Total number of exploration rounds\\
$\beta_{t-1}:=c_{n,1/n}$ & Theoretical threshold for the exploitation test in \algo
\\
$\gamma_t:=c_{n,1/S_t^2}$ & Theoretical value for the confidence intervals in \algo
\\
$K_t \in \{0,1,\dots\}$ & Phase index at time $t$
\\
$T_k$ & Time at which phase $k$ starts
\\
$\mathcal{T}_k:=\{t\in[n] : K_t = k\}$ & Time steps in phase $k$
\\
$\mathcal{T}_k^E:=\{t\in\mathcal{T}_k : E_t\}$ & Exploration rounds in phase $k$
\\
$\{p_k\}_{k\geq 0}$ & Total number of exploration rounds in each phase
\\
$\alpha_k^\lambda, \alpha_k^\omega$ & Step sizes
\\
$V_t := \sum_{s=1}^t \phi(X_s, A_s)\phi(X_s,A_s)^\transp$  & Design matrix
\\
$\wb{V}_t := V_t + \nu I$ & Regularized design matrix ($\nu \geq 1$)
\\
$U_t := \sum_{s=1}^t \phi(X_s,A_s) Y_s$ & Sum of reward-weighted features
\\
$\wh{\theta}_t := \wb{V}_t^{-1}U_t$ & Regularized least-squares estimate
\\
$\wt{\theta}_t := \argmin_{\theta \in \Theta \cap \mathcal{C}_t}\| \theta - \wh{\theta}_t \|_{\wb{V}_t}^2$ & Projected least-squares estimates \\
$\mathcal{C}_t := \{\theta \in \R^d : \| \theta - \wh{\theta}_t\|_{\wb{V}_{t}}^2 \leq \beta_{t}\}$ & Confidence ellipsoid at time $t$
\\
$G_t$ & Good event (see App. \ref{app:good.events})
\\
$M_n = \sum_{t=1}^n \indi{E_t, \neg G_t}$ & Number of exploration rounds without good event
\\
$M_{n,k} = \sum_{t\in\mathcal{T}_k^E} \indi{\neg G_t}$ & Number of exploration rounds in phase $k$ without good event
\\
\hline
\end{longtable}
}

\section{Comparison to Related Work}\label{app:related.work}

\begin{table*}[t!]
\begin{center}
	\begin{small}
            \begin{tabular}{L{2.5cm} C{1.7cm} C{1.5cm} C{3cm} C{3cm}} 
			\toprule
			Feature/Algorithm & OSSB & OAM & SPL & \algo \\
			\cmidrule{1-5}
			\textit{Setting} & general MAB & linear contextual & general MAB & linear contextual \\
			\textit{Objective fun.} & constrained & constrained & saddle (ratio) & saddle (Lagrangian) \\
			\textit{Opt. variables} & counts & counts & rates & policies \\
			\textit{Asympt. optimality} & order-opt & opt & opt & opt \\
			\textit{Finite-time bound} & \cmark & \xmark & \cmark & \cmark \\
			\textit{Explore/exploit} & tracking test & glrt & glrt & glrt \\
			\textit{Tracking} & direct & direct & cumulative & sampling \\
			\textit{Optimization} & exact & exact & incr. and best-response & incr. \\
			\textit{Exp. level} & forcing & forcing & unstruct. optimism & optimism \\
			\textit{Parameters} & forcing, test & forcing, test & gaps clip, test, conf. values &  $\lambda_{\max}$, test, conf. values, phases \\

			\bottomrule
		\end{tabular}
	\end{small}
\end{center}
	\caption{Comparison of structured bandit algorithms. OSSB \citep{combes2017minimal}, OAM \citep{hao2019adaptive}, SPL~\citep{degenne2020structure} and \algo (this paper).}
	\label{tab:hyperparams}
\end{table*}

In Table~\ref{tab:hyperparams} we compare several bandit algorithms along several dimensions:
\begin{itemize}
	\item \textit{Setting} refers to whether the algorithm is designed for general multi-armed bandit (non-contextual) structured problems or it is for the linear contextual case.
	\item \textit{Objective function} refers to the optimization problem solved by the algorithm. It can be either the original constrained optimization in~\eqref{eq:optim-lb} or a saddle point problem (either obtained by taking the ratio of objective and constraints or the Lagrangian relaxation in~\eqref{eq:optim-lb-z}).
	\item \textit{Optimization variables} refers to the variables that are optimized by the algorithm: \textit{counts} is the $\eta$ variables in~\eqref{eq:optim-lb}, \textit{rates} is the ratio fraction of regret, \textit{policies} is the $\omega$ variables in~\eqref{eq:optim-lb-z}. 
	\item \textit{Asymptotic optimality} is either \textit{order optimal} when only a logarithmic rate is proved with non-optimal constants, or \textit{optimal}, in which case the leading constant is $v^\star(\theta)$ as in Prop.~\ref{p:lower.bound}.
	\item \textit{Finite-time bound} is whether finite-time guarantees are reported.
	\item \textit{Explore/exploit} refers to the separation between exploration and exploitation steps and whether it is based on a \textit{tracking performance test} or on the generalized likelihood ratio test (GLRT).\footnote{Notice that none of the algorithms implement the exact form of the GLRT, but slight variations that provide equivalent guarantees.}
	\item \textit{Tracking} refers to how arms are selected during the exploration phase.
	\item \textit{Optimization} refers to whether the optimization problem is solved exactly at each step or using an incremental method. SPL combines an incremental method using an exact computation of a best response solution.
	\item \textit{Exploration level} refers to the technique used during exploration steps to guarantee a minimum level of exploration. The first option is \textit{forcing} all arms to satisfy a hard threshold of minimal pulls. The second option is to include a form of \textit{optimism} in the optimization problem. 
	\item \textit{Parameters} list the major parameters in the definition of the algorithm. This is often difficult since some algorithms directly pick theoretical values for some input parameters, while others may provide specific values only during the analysis. OSSB requires tuning the forcing parameter and the parameter used in the exploration/exploitation test. OAM has a forcing parameter and needs to properly tune the GLRT. SPL requires clipping the gap estimates from below, tuning the GLRT, and designing suitable confidence intervals for optimism. \algo requires an upper bound for the multiplier, tuning of the GLRT, confidence intervals, and phases to tune the normalization factor $z$.
\end{itemize}

The major insights from this comparison can be summarized as follows:
\begin{itemize}
	\item \textit{Comparison \algo/OAM:} This is the more direct comparison, since both algorithms are designed for contextual linear (see Sect.~\ref{sec:experiments} for the empirical comparison). \algo improves over OAM in almost all dimensions. On the theoretical side, we provide explicit finite-time regret bounds showing that \algo successfully adapts to the context distribution, while the performance of OAM is significantly affected by $\rho_{\min}$. Furthermore, in many lower-order regret terms in the analysis of OAM the cardinality of the arm space appears linearly, while the regret of \algo only depends on $\log(|\A|)$. 
	On the algorithmic side, \algo leverages a primal-dual gradient descent that greatly improves the computational complexity compared to the exact solution of the constrained optimization problem done in OAM at each exploration step. Furthermore, replacing the forcing strategy with an optimistic version of the optimization problem allows \algo to better adapt to the problem and avoid pulling highly suboptimal/non-informative arms.
	\item \textit{Comparison \algo/SPL:} The comparison is more on the algorithmic and theoretical properties rather than the actual algorithms, since they are designed for different settings.\footnote{While the general structured bandit problem does contain the linear case, it is unclear how it can manage the contextual linear case.} While both algorithms replace the constrained problem in the lower bound by a saddle point problem, SPL takes the ratio between constraints and regret, while in \algo we take a more straightforward Lagrangian relaxation. As a result, in \algo we rely on a rather standard primal-dual gradient approach to optimize~\eqref{eq:optim-lb-z}, while SPL relies on online learning algorithms for the solution of the saddle-point problem. Finally, both algorithms replace forcing by an optimistic version of the optimization problem. Nonetheless, SPL uses separate confidence intervals for each arm that ignore the structure of the problem, while \algo relies on confidence intervals build specifically for the linear case. Finally, the regret bound of SPL, similarly to the one of OAM, depends linearly on $|\A|$ in several lower-order terms, even when instantiated for linear structures. \algo, on the other hand, has only $\log(|\A|)$ dependence.
\end{itemize}

\section{Lower Bound}\label{app:lower.bound}


\subsection{Proof of Lem.~\ref{lem:z-opt-upperbound}}

\paragraph{Feasibility of~\eqref{eq:optim-lb-z}.} We start from the first result in Lem.~\ref{lem:z-opt-upperbound}, which states the minimal value of $z$ for which \eqref{eq:optim-lb-z} is feasible. Clearly, the maximal value that the left-hand side of the KL constraint can assume is
\begin{align*}
\max_{\omega\in\Omega}\inf_{\theta' \in \Theta_{\mathrm{alt}}}\E_{\rho}\left[\sum_{a\in\A}{\omega}(x,a)d_{x,a}(\theta^\star,\theta') \right],
\end{align*}
which can also be interpreted as the solution to the associated pure-exploration (or best-arm identification) problem \citep[e.g.,][]{degenne2019non}. Therefore,
\begin{align*}
\underline{z}(\theta^\star) &:= \min\left\{z > 0 : \text{\eqref{eq:optim-lb-z} is feasible}\right\}\\ &= \min\left\{ z > 0 : \max_{\omega\in\Omega}\inf_{\theta' \in \Theta_{\mathrm{alt}}}\E_{\rho}\left[\sum_{a\in\A}{\omega}(x,a)d_{x,a}(\theta^\star,\theta') \right] \geq \frac{1}{z}\right\}\\ &= \frac{1}{ \max_{\omega\in\Omega}\inf_{\theta' \in \Theta_{\mathrm{alt}}}\E_{\rho}\left[\sum_{a\in\A}{\omega}(x,a)d_{x,a}(\theta^\star,\theta') \right]}.
\end{align*}
This proves the first statement in Lem.~\ref{lem:z-opt-upperbound}.

\paragraph{Connection between~\eqref{eq:optim-lb} and~\eqref{eq:optim-lb-z}.}
In order to prove the second result, let us rewrite~\eqref{eq:optim-lb-z} in the following more convenient form:
\begin{equation}\label{eq:optim-lb-ztilde.app}\tag{$\text{P}'_z$}
\begin{aligned}
&\underset{\eta(x,a) \geq 0}{\mathrm{minimize}}&& \sum_{x\in\X}\rho(x)\sum_{a\in\A}\eta(x,a)\Delta_{\theta^\star}(x,a)\\
&\mathrm{subject\ to} && \inf_{\theta' \in \Theta_{\mathrm{alt}}}\sum_{x\in\X}\rho(x)\sum_{a\in\A}\eta(x,a)d_{x,a}(\theta^\star,\theta') \geq 1,\\
&&& \sum_{a\in\A}\eta(x,a) = z \quad \forall x\in\X.
\end{aligned}
\end{equation}
Note that~\eqref{eq:optim-lb-ztilde.app} is obtained from~\eqref{eq:optim-lb-z} in the main paper by performing the change of variables $\eta(x,a) = z\omega(x,a)$, hence the two problems are equivalent.
Recall that $v^\star(\theta^\star)$ is the optimal value of~\eqref{eq:optim-lb} and $u^\star(z,\theta^\star)$ is the optimal value of~\eqref{eq:optim-lb-ztilde.app} and~\eqref{eq:optim-lb-z} (if there exists one).
We are interested in bounding the deviation between $u^\star(z,\theta^\star)$ and $v^\star(\theta^\star)$ as a function of $z$. 

Let us first define the following set of \emph{confusing} models:
\begin{align*}
    \wt{\Theta}_{alt} := \left\{\theta' \in\Theta_{alt} : \forall x\in\X, \mu^\star_{\theta^\star}(x) = \mu_{\theta'}(x,a^\star_x)\right\},
\end{align*}
where, for the sake of readability, we abbreviate $a^\star_x = a^\star_{\theta^\star}(x)$.
These models are indistinguishable from $\theta^\star$ by pulling only optimal arms. The following proposition, which was proved in \citep{degenne2020structure}, connects models in the alternative set $\Theta_{\mathrm{alt}}$ with the confusing ones in $\wt{\Theta}_{\mathrm{alt}}$.
\begin{proposition}[\citep{degenne2020structure}]\label{prop:altset.lipschitz}
        There exists a constant $c_{\Theta}>0$ such that, for all $\theta'\in\Theta_{\mathrm{alt}}$, there exists $\theta'' \in \wt{\Theta}_{\mathrm{alt}}$ such that,
\begin{align*}
        \forall x\in\X,a\in\A \qquad |\mu_{\theta'}(x,a) - \mu_{\theta''}(x,a)| \leq c_{\Theta}|\mu_{\theta^\star}^\star(x) - \mu_{\theta'}(x,a^\star_{\theta^\star}(x))|.
\end{align*}
\end{proposition}

We now prove the bound on $ u^\star(z,\theta)$ reported in Lem.~\ref{lem:z-opt-upperbound}.
\begin{proof}[Proof of Lem.~\ref{lem:z-opt-upperbound}]
We start from the Lagrangian version of~\eqref{eq:optim-lb-ztilde.app}.
\begin{align*}
        u^\star(z,\theta) = \min_{\eta \geq 0}\left\{\sum_{x\in\X}\rho(x)\sum_{a\in\A}\eta(x,a)\Delta_{\theta^\star}(x,a) + \lambda^\star(z,\theta^\star)\left(1 - \inf_{\theta' \in \Theta_{\mathrm{alt}}}\sum_{x\in\X}\rho(x)\sum_{a\in\A}\eta(x,a)d_{x,a}(\theta^\star,\theta') \right) \right\},
\end{align*}
subject to $\sum_{a\in\A} \eta(x,a) = z$ for each context $x\in\X$. Here $\lambda^\star(z,\theta^\star)$ is the optimal value of the Lagrange multiplier for the same problem.
We distinguish two cases.

\paragraph{Case 1: {$z < \max_{x\in\X}\frac{1}{\rho(x)}\sum_{a \neq a^\star_{\theta^\star}(x)} \eta^\star(x,a)$}.}

Let
\begin{align*}
    \wb{\eta}(x,a) = z \cdot \begin{cases} 
            \frac{\eta^\star(x,a)/\rho(x)}{\max_{x\in\X}\frac{1}{\rho(x)}\sum_{a \neq a^\star_{\theta^\star}(x)} \eta^\star(x,a)} & \text{if } a \neq a^\star_{\theta^\star}(x)\\
            1 - \frac{\sum_{a\neq a^\star_{\theta^\star} (x)}\eta^\star(x,a)/\rho(x)}{\max_{x\in\X}\frac{1}{\rho(x)}\sum_{a \neq a^\star_{\theta^\star}(x)} \eta^\star(x,a)} & \text{otherwise}
    \end{cases}
\end{align*}
where $\eta^\star$ is the optimal solution of~\eqref{eq:optim-lb}.
Since $\sum_{a} \wb{\eta}(x,a) = z$, we have that $u^\star(z,\theta^\star)$ is less or equal to the value of the Lagrangian for $\eta = \wb{\eta}$, i.e.,
\begin{align*}
        u^\star(z,\theta^\star) \leq v^\star(\theta^\star) + \lambda^\star(z,\theta^\star)\left(1 - \inf_{\theta' \in \Theta_{\mathrm{alt}}}\sum_{x\in\X}\rho(x)\sum_{a\in\A}\wb{\eta}(x,a)d_{x,a}(\theta^\star,\theta') \right),
\end{align*}
where we used the fact that
\begin{align*}
        \sum_{x\in\X}\rho(x)\sum_{a\in\A}\wb{\eta}(x,a)\Delta_{\theta^\star}(x,a) =  \underbrace{\frac{z}{\max_{x\in\X}\frac{1}{\rho(x)}\sum_{a \neq a^\star_{\theta^\star}(x)} \eta^\star(x,a)}}_{< 1}\underbrace{\sum_{x\in\X}\sum_{a\neq a^\star_{\theta^\star}(x)}{\eta^\star}(x,a)\Delta_{\theta^\star}(x,a)}_{= v^\star(\theta^\star)}
\end{align*}
since $\Delta_{\theta^\star}(x, a^\star_{\theta^\star}(x)) = 0$.
Since the KL divergence $d_{x,a}(\theta^\star,\theta')$ is lower-bounded by zero, in case 1 we have
\begin{align*}
        u^\star(z,\theta^\star) \leq v^\star(\theta^\star) + \lambda^\star(z,\theta^\star).
\end{align*}

\paragraph{Case 2: {$z \geq \max_{x\in\X}\frac{1}{\rho(x)}\sum_{a \neq a^\star_{\theta^\star}(x)} \eta^\star(x,a)$}.}
Let
\begin{align*}
    \wb{\eta}(x,a) = \begin{cases} \eta^\star(x,a)/\rho(x) & \text{if } a \neq a^\star_{\theta^\star}(x)\\
        z - \sum_{a \neq a^\star_{\theta^\star}(x)} \eta^\star(x,a)/\rho(x) & \text{otherwise}
    \end{cases}
\end{align*}
where, as before, $\eta^\star$ is the optimal solution of~\eqref{eq:optim-lb}. 
Since $z \geq \sum_{a \neq a^\star_{\theta^\star}(x)} \eta^\star(x,a)/\rho(x)$  for any $x \in \mathcal{X}$, $\wb{\eta}$ is well defined.
Since $\wb{\eta}$ also sums to $z$ for each context, we have that $u^\star(z,\theta)$ is less or equal to the value of the Lagrangian for $\eta = \wb{\eta}$, i.e.,
\begin{align*}
        u^\star(z,\theta^\star) \leq v^\star(\theta^\star) + \lambda^\star(z,\theta^\star)\left(1 - \inf_{\theta' \in \Theta_{\mathrm{alt}}}\sum_{x\in\X}\rho(x)\sum_{a\in\A}\wb{\eta}(x,a)d_{x,a}(\theta^\star,\theta') \right).
\end{align*}

We first lower bound the infimum on the right hand side. We have
\begin{align}
    \notag\inf_{\theta' \in \Theta_{\mathrm{alt}}}\sum_{x\in\X}\rho(x)\sum_{a\in\A}\wb{\eta}(x,a)d_{x,a}(\theta^\star,\theta')
    = \min\Bigg\{
            &\underbrace{\inf_{\theta' \in \wt{\Theta}_{\mathrm{alt}}}\sum_{x\in\X}\rho(x)\sum_{a\in\A}\wb{\eta}(x,a)d_{x,a}(\theta^\star,\theta')}_{
            I_{\wt\Theta_{\mathrm{alt}}}
            }, \\
            &\underbrace{\inf_{\theta' \in \Theta_{\mathrm{alt}} \setminus \wt{\Theta}_{\mathrm{alt}}}\sum_{x\in\X}\rho(x)\sum_{a\in\A}\wb{\eta}(x,a)d_{x,a}(\theta^\star,\theta')}_{
            I_{\Theta_{\mathrm{alt}}\setminus\wt\Theta_{\mathrm{alt}}}
            }
    \Bigg\}.
\end{align}
By definition of $\wb{\eta}$ and $\eta^\star$, the infimum over the set of confusing models can be written as
\begin{align}\label{eq:bound.confusing.app}
   I_{\wt\Theta_{\mathrm{alt}}} =
 \inf_{\theta' \in \wt{\Theta}_{\mathrm{alt}}}\sum_{x\in\X}\rho(x)\sum_{a\in\A}\wb{\eta}(x,a)d_{x,a}(\theta^\star,\theta') = \inf_{\theta' \in \wt{\Theta}_{\mathrm{alt}}}\sum_{x\in\X}\sum_{a\neq a^\star_x}{\eta^\star}(x,a)d_{x,a}(\theta^\star,\theta') \geq 1,
\end{align}
where the equality holds since the KLs are zero in the optimal arms, which are the only arms where the values of $\wb{\eta}$ differ from those of $\eta^\star$, and the inequality holds since $\eta^\star$ is feasible.
Regarding the infimum over the non-confusing models, 
\begin{align}
I_{\Theta_{\mathrm{alt}}\setminus\wt\Theta_{\mathrm{alt}}}=
    \inf_{\theta' \in \Theta_{\mathrm{alt}} \setminus \wt{\Theta}_{\mathrm{alt}}}\left(
    \underbrace{ \sum_{x\in\X}\rho(x)\wb{\eta}(x,a^\star_x)d_{x,a^\star_x}(\theta^\star,\theta')}_{(i)} +
    \underbrace{
    \sum_{x\in\X}\sum_{a\neq a^\star_x}{\eta^\star}(x,a)d_{x,a}(\theta^\star,\theta')}_{(ii)} \right).
\end{align}

We partition the set of non-confusing models in two subsets:
\begin{align}
        \wt{\Theta}_{\mathrm{alt}}^{(1)} &:= \Big\{\theta' \in \Theta_{\mathrm{alt}} \setminus \wt{\Theta}_{\mathrm{alt}} : \forall x\in\X, |\mu^\star_{\theta^\star}(x) - \mu_{\theta'}(x,a^\star_{\theta^\star}(x)) | < \epsilon_z \Big\},\\
        \wt{\Theta}_{\mathrm{alt}}^{(2)} &:= \Big\{\theta' \in \Theta_{\mathrm{alt}} \setminus \wt{\Theta}_{\mathrm{alt}} : \exists x\in\X, |\mu^\star_{\theta^\star}(x) - \mu_{\theta'}(x,a^\star_{\theta^\star}(x)) | \geq \epsilon_z \Big\}.
\end{align}
The value of $\epsilon_z$ will be specified later. We have, for $\theta'' \in \wt{\Theta}_{\mathrm{alt}}$,
\begin{align}
        \inf_{\theta' \in \wt{\Theta}^{(1)}_{\mathrm{alt}}} & \sum_{x\in\X}\rho(x)\sum_{a\in\A}\wb{\eta}(x,a)d_{x,a}(\theta^\star,\theta')
        \stackrel{(a)}{\geq} \inf_{\theta' \in \wt{\Theta}^{(1)}_{\mathrm{alt}}} \sum_{x\in\X}\sum_{a\neq a^\star_x}{\eta^\star}(x,a)d_{x,a}(\theta^\star,\theta')\\
    &\stackrel{(b)}{\geq} \inf_{\theta' \in \wt{\Theta}^{(1)}_{\mathrm{alt}}}\sum_{x\in\X}\sum_{a\neq a^\star_x}{\eta^\star}(x,a)\left(d_{x,a}(\theta^\star,\theta'') - \frac{1}{\sigma^2}|\mu_{\theta'}(x,a) - \mu_{\theta''}(x,a)|\right)\\
    &\stackrel{(c)}{\geq}  1 - \frac{1}{\sigma^2}\sup_{\theta' \in \wt{\Theta}^{(1)}_{\mathrm{alt}}}\sum_{x\in\X}\sum_{a\neq a^\star_x}{\eta^\star}(x,a)|\mu_{\theta'}(x,a) - \mu_{\theta''}(x,a)|\\
    &\stackrel{(d)}{\geq}1  - \frac{c_\Theta}{\sigma^2}\sup_{\theta' \in \wt{\Theta}^{(1)}_{\mathrm{alt}}}\sum_{x\in\X}\sum_{a\neq a^\star_x}{\eta^\star}(x,a)\underbrace{|\mu_{\theta^\star}^\star(x) - \mu_{\theta'}(x,a^\star_x)|}_{< \epsilon_z}\\
    &\stackrel{(e)}{\geq} 1 - \frac{c_\Theta \epsilon_z}{\sigma^2}\sum_{x\in\X}\sum_{a\neq a^\star_x}{\eta^\star}(x,a),
\end{align}
where $(a)$ uses the fact that $(i) \geq 0$ and the definition of $\wb\eta$, $(b)$ uses the Lipschitz property of the KL divergence between Gaussians, $(c)$ uses the fact that $\wb\eta$ is feasible for confusing models (see Eq.~\ref{eq:bound.confusing.app}), $(d)$ uses Prop.~\ref{prop:altset.lipschitz} and $(e)$ uses the definition of $ \wt{\Theta}_{\mathrm{alt}}^{(1)}$. 
Regarding the second set of alternative models,
\begin{align}
        \inf_{\theta' \in \wt{\Theta}^{(2)}_{\mathrm{alt}}} & \sum_{x\in\X}\rho(x)\sum_{a\in\A}\wb{\eta}(x,a)d_{x,a}(\theta^\star,\theta') \\
      &\stackrel{(f)}{\geq} 
      \inf_{\theta' \in \wt{\Theta}^{(2)}_{\mathrm{alt}}}\sum_{x\in\X}\rho(x)\left(z - \sum_{a \neq a^\star_{\theta^\star}(x)} \eta^\star(x,a)/\rho(x) \right)d_{x,a^\star_x}(\theta^\star,\theta')\\
    &\stackrel{(g)}{=} \inf_{\theta' \in \wt{\Theta}^{(2)}_{\mathrm{alt}}}\sum_{x\in\X}\rho(x)\left(z - \sum_{a \neq a^\star_{\theta^\star}(x)} \eta^\star(x,a)/\rho(x) \right) \frac{1}{2\sigma^2}\underbrace{(\mu_{\theta^\star}(x,a_x^\star)-\mu_{\theta'}(x,a_x^\star))^2}_{\geq \epsilon_z^2}\\
    &\stackrel{(k)}{=} \frac{\epsilon_z^2}{2\sigma^2}\left(z - \sum_{x\in\X}\sum_{a \neq a^\star_{\theta^\star}(x)} \eta^\star(x,a) \right).
\end{align}
where $(f)$ uses the fact that $(ii) \geq 0$ and the definition of $\wb\eta$, $(g)$ uses the definition of KL for Gaussian distributions and $(k)$ uses the definition of $\wt{\Theta}^{(2)}_{\mathrm{alt}}$. 
Let $\sumpull(\theta^\star) := \sum_{x\in\X}\sum_{a \neq a^\star_{\theta^\star}(x)} \eta^\star(x,a)$. Putting together the results so far, we have
\begin{align}
        \inf_{\theta' \in \Theta_{\mathrm{alt}}}\sum_{x\in\X}\rho(x)\sum_{a\in\A}\wb{\eta}(x,a)d_{x,a}(\theta^\star,\theta') \geq \min\left\{1,1 - \frac{c_\Theta \epsilon_z \sumpull(\theta^\star)}{\sigma^2}, \frac{\epsilon_z^2}{2\sigma^2}\Big(z - \sumpull(\theta^\star) \Big) \right\}.
\end{align}
Setting $\epsilon_z = \sqrt{\frac{2\sigma^2}{z}}$, 
%
\begin{align}
        \inf_{\theta' \in \Theta_{\mathrm{alt}}}\sum_{x\in\X}\rho(x)\sum_{a\in\A}\wb{\eta}(x,a)d_{x,a}(\theta^\star,\theta') 
        \geq \max\left\{
        \min\left\{
                1- \frac{c_\Theta \sqrt{2} \sumpull(\theta^\star)}{\sigma\sqrt{z}}, 1 - \frac{\sumpull(\theta^\star)}{z}
            \right\}
        , 0\right\},
\end{align}
Therefore, in case 2 we have
\begin{align*}
        u^\star(z,\theta^\star) \leq v^\star(\theta^\star) + \lambda^\star(z,\theta^\star) \min\left\{
            \max\left\{
                \frac{c_\Theta \sqrt{2} \sumpull(\theta^\star)}{\sigma\sqrt{z}}, \frac{\sumpull(\theta^\star)}{ z}
            \right\}
        , 1\right\}.
\end{align*}

\paragraph{Bounding $\lambda^\star(z,\theta^\star)$.}
Finally, we show that the optimal multiplier $\lambda^\star(z,\theta^\star)$ is bounded (regardless of which case $z$ falls into).
Let $\underline{\eta} = z\underline{\omega}$, where $\underline{\omega} = \omega^\star_{\underline{z},\theta^\star}$ is the pure-exploration solution obtained solving problem~\eqref{eq:optim-lb-z} with $\underline{z}(\theta^\star)$. Recall from the first statement of Lem.~\ref{lem:z-opt-upperbound} that
\begin{align*}
        \inf_{\theta' \in \Theta_{\mathrm{alt}}}\sum_{x\in\X}\rho(x)\sum_{a\in\A}\underline{\omega}(x,a)d_{x,a}(\theta^\star,\theta') = \frac{1}{\underline{z}(\theta^\star)}.
\end{align*}
Thus, \underline{$\eta$} is strictly feasible for problem ($\tilde{\mathrm{P}}_z$) and has constraint value
\begin{align}
        \inf_{\theta' \in \Theta_{\mathrm{alt}}}\sum_{x\in\X}\rho(x)\sum_{a\in\A}\underline{\eta}(x,a)d_{x,a}(\theta^\star,\theta') = \frac{z}{\underline{z}(\theta^\star)} > 1
\end{align}
since $z > \underline{z}(\theta^\star)$ by assumption.
Using the Slater's condition (see e.g., Lem. 3 in~\citep{nedic2009subgradient}),
\begin{align}
    0\leq \lambda^\star(z,\theta^\star) 
    &\leq \frac{\sum_{x\in\X}\rho(x)\sum_{a\in\A}\Delta_{\theta^\star}(x,a)(\underline{\eta}(x,a) - \eta^\star_z(x,a))}{\inf_{\theta' \in \Theta_{\mathrm{alt}}}\sum_{x\in\X}\rho(x)\sum_{a\in\A}\underline{\eta}(x,a)d_{x,a}(\theta^\star,\theta') - 1}\\
    &\leq \frac{z}{\frac{z}{\underline{z}(\theta^\star)} - 1}\sum_{x\in\X}\rho(x)\sum_{a\in\A}
    \underbrace{\Delta_{\theta^\star}(x,a)}_{\geq 0 }\Big(\underline{\omega}(x,a) - \underbrace{\eta^\star_z(x,a)/z}_{\geq 0} \Big)\\
    &\leq \frac{z}{\frac{z}{\underline{z}(\theta^\star)} - 1}\sum_{x\in\X}\rho(x)\sum_{a\in\A} \underbrace{\Delta_{\theta^\star}(x,a)}_{\in [0,2BL]}\underline{\omega}(x,a)
    \leq 2BL \frac{z\underline{z}(\theta^\star)}{z - \underline{z}(\theta^\star)}.
\end{align}
\end{proof}

\subsection{Discussion About Problem~\eqref{eq:optim-lb-z}}

In this section we provide more intuition about the effect of explicitly adding the context distribution in the formulation of the lower bound. As mentioned in Sect.~\ref{sec:lower.bound} the infimum in the original problem~\eqref{eq:optim-lb} may not be attainable, thus making it difficult to solve it and build a learning algorithm around it. A simple way to address this issue is to introduce a global constraint so that the sum of $\eta$ is constrained to a parameter $z$. This leads to the optimization
%
\begin{equation}\label{eq:optim-lb.z}\tag{$\wt P_z$}
\begin{aligned}
&\underset{\eta(x,a) \geq 0}{\inf}&& \sum_{x\in\X}\sum_{a\in\A}\eta(x,a)\Delta_{\theta^\star}(x,a)\\
&\mathrm{s.t.} && \inf_{\theta' \in \Theta_{\mathrm{alt}}}\sum_{x\in\X}\sum_{a\in\A}\eta(x,a)d_{x,a}(\theta^\star,\theta') \geq 1\\
&&& \sum_{x,a}\eta(x,a) = z
\end{aligned}
\end{equation}
Let $\wt\eta^\star_z$ be the optimal solution of~\eqref{eq:optim-lb.z} and $\wt u^\star_z$ be its associated optimal value. On the other hand, the problem~\eqref{eq:optim-lb-z} we propose can be easily rewritten as
\begin{equation}\label{eq:optim-lb.rho.z}\tag{$P_{z}$}
	\begin{aligned}
		&\underset{\eta(x,a) \geq 0}{\inf}&& \sum_{x\in\X}\sum_{a\in\A}\eta(x,a)\Delta_{\theta^\star}(x,a)\\
		&\mathrm{s.t.} && \inf_{\theta' \in \Theta_{\mathrm{alt}}}\sum_{x\in\X}\sum_{a\in\A}\eta(x,a)d_{x,a}(\theta^\star,\theta') \geq 1 \\
		&&& \sum_{a}\eta(x,a) = z\rho(x) \quad \forall x\in\X
	\end{aligned}
\end{equation}
where the constraint is now on each context and it depends on the context distribution ($\omega(x,a) = \frac{\eta(x,a)}{z\rho(x)}$).\footnote{Notice that the constraint directly implies $\sum_{x,a}\eta(x,a) = z$.} The crucial difference w.r.t.~\eqref{eq:optim-lb.z} is that now the number of samples prescribed by $\eta$ needs to be ``compatible'' with the amount of samples that can be collected within $z$ steps from each context $x$ depending on its probability $\rho(x)$. Let $\eta^\star_z$ be the optimal solution of~\eqref{eq:optim-lb.rho.z} and $u^\star_z$ be its associated objective value. In order to understand how this difference may translate into a different behavior when integrated in an actual algorithm, let compare the two solutions $\wt\eta^\star_z$ and $\eta^\star_z$ if executed for $z$ steps.\footnote{We recall that, as discussed in Sect.~\ref{sec:lower.bound}, $z$ introduces a more finite-time flavor into the lower bound, where pulls should now be allocated so as to satisfy the KL-information constraint within $z$ steps.} Since neither of them can be ``played'' (i.e., only one arm can be selected at each step), we need to define a specific \textit{execution strategy} to ``realize'' an allocation $\eta$. For the ease of exposition, let consider a simple strategy where in each context $x$, an arm $a$ is pulled at random proportionally to $\eta(x,a)$. Let $\wt \zeta_z(x,a)$ and $\zeta_z(x,a)$ the expected number of samples generated in each context-arm pair $(x,a)$ when sampling from $\wt\eta^\star_z$ and $\eta^\star_z$ respectively. Then we have
\begin{align}\label{eq:eta.mismatch}
	\wt\zeta_z(x,a) &=  \wt \eta^\star_z(x,a)\overbrace{\frac{z\rho(x)}{\sum_{a'} \wt\eta^\star_z(x,a')}}^{\text{mismatch $\alpha_z(x,a)$}} \\
	\zeta_{z}(x,a) &= z\rho(x) \frac{\eta^\star_{z}(x,a)}{\sum_{a'} \eta^\star_{z}(x,a')} = \eta^\star_{z}(x,a)
\end{align}
which reveals how $\wt \eta^\star_z(x,a)$, which was explicitly optimized under the constraint that the total number of samples was $z$, may not really be ``realizable'' in practice, since it ignores the context distribution and the number of samples that can be actually generated at each context $x$. On the other hand, on average the desired allocation $\eta^\star_z$ can always be realized within $z$ steps. Interestingly, the mismatch between $\wt \eta^\star_z(x,a)$ and $\wt\zeta_z(x,a)$ would no longer guarantee neither the performance $\wt u^\star_z$ ``promised'' by $\wt \eta^\star_z$ nor the feasibility for~\eqref{eq:optim-lb.z} (i.e., $\wt\zeta_z(x,a)$ may not satisfy the KL-information constraint). This would make considerably more difficult to build a learning algorithm on $\wt \eta^\star_z$ than on $\eta^\star_z$.

As it can be noticed in Eq.~\ref{eq:eta.mismatch}, the level mismatch is due to the execution strategy used to realize the allocation $\wt\eta^\star_z$ (in this case, a simple sampling approach) and better solutions may exist. We could even consider to directly optimize the execution strategy so as to achieve a mismatch $\alpha_z(x,a)$ that induce an allocation $\wt\zeta_z(x,a)$ that performs best in terms of regret minimization under the KL-information constraint. Given the $\wt\eta^\star_z$ obtained from~\eqref{eq:optim-lb.z}, we define the optimization problem
\begin{equation}\label{eq:optim-lb.alpha}\tag{$\wt P_{\alpha}$}
\begin{aligned}
&\underset{\alpha(x,a) \geq 0}{\inf}&& \sum_{x\in\X}\sum_{a\in\A}\wt \eta^\star_z(x,a) \alpha(x,a)\Delta_{\theta}(x,a)\\
&\mathrm{s.t.} && \inf_{\theta' \in \Theta_{\mathrm{alt}}}\sum_{x\in\X}\sum_{a\in\A}\wt \eta^\star_z(x,a) \alpha(x,a) d_{x,a}(\theta,\theta') \geq 1\\
&&& \sum_{a}\wt \eta^\star_z(x,a) \alpha(x,a) = z\rho(x)
\end{aligned}
\end{equation}
Interestingly, a simple change of variables reveals that~\eqref{eq:optim-lb.alpha} does coincide with~\eqref{eq:optim-lb.rho.z} that we originally introduced (i.e., $\alpha^\star(x,a) = \frac{\eta^\star_z(x,a)}{\wt\eta^\star_z(x,a)}$ minimizes the problem). This illustrates that solving~\eqref{eq:optim-lb.rho.z} indeed leads to the optimal allocation compatible with the context distribution and the constraint of $z$ realizations. 

\section{Lagrangian Formulation}\label{app:lagrangian}

We discuss in more details the Lagrangian formulation presented in Section \ref{sec:lower.bound}. Consider the following variant of \eqref{eq:optim-lb-z}:
\begin{equation}\label{eq:optim-lb-z-mu}\tag{$\wb{P}_z$}
\begin{aligned}
&\underset{\omega \in \Omega}{\max}&& \E_{\rho}\bigg[\sum_{a\in\A}\omega(x,a)\mu_{\theta^\star}(x,a)\bigg]
\quad \mathrm{s.t.} \quad
\inf_{\theta' \in \Theta_{\mathrm{alt}}}\E_{\rho}\bigg[\sum_{a\in\A}\omega(x,a)d_{x,a}(\theta^\star,\theta') \bigg] \geq 1/z
\end{aligned}
\end{equation}
This problem differs from \eqref{eq:optim-lb-z} since we replaced the action gaps with the means in the objective function and avoided scaling the latter by $z$. Let $\wb{\omega}^\star_{z,\theta^\star}$ the optimal solution of~\eqref{eq:optim-lb-z-mu} and $\wb{u}^\star(z, \theta^\star)$ be its associated value (if the problem is unfeasible we set $\wb{u}^\star(z, \theta^\star) = +\infty$). Since the feasibility set is equivalent in \eqref{eq:optim-lb-z} and \eqref{eq:optim-lb-z-mu} as we only changed the objective function, the following proposition is immediate.

\begin{proposition}\label{prop:pz-vs-pzmu}
The following properties hold:
\begin{enumerate}
\item Both \eqref{eq:optim-lb-z} and \eqref{eq:optim-lb-z-mu} are feasible for $z \geq \underline{z}(\theta^\star)$;
\item ${\omega}^\star_{z,\theta^\star} = \wb{\omega}^\star_{z,\theta^\star} $.
\item ${u}^\star(z, \theta^\star) = z\left(\mu^\star - \wb{u}^\star(z, \theta^\star)\right)$ where $\mu^\star = \E_{\rho}[\mu^\star_{\theta^\star}(x)]$;
\end{enumerate}
\end{proposition}
Due to the equivalence demonstrated in Prop. \ref{prop:pz-vs-pzmu}, in the remaining we shall occasionally write $\omega^\star_z$ to denote both ${\omega}^\star_{z,\theta^\star}$ and $\wb{\omega}^\star_{z,\theta^\star} $. 

We recall the Lagrangian relaxation problem of Sec. \ref{sec:lower.bound}. For any $\omega\in\Omega$, let $f(\omega; \theta^\star)$ denote the objective function and $g(\omega, z; \theta^\star)$ denote the KL constraint
\begin{align*}
f(\omega; \theta^\star) = \E_{\rho}\Big[\sum_{a\in\A}\omega(x,a)\mu_{\theta^\star}(x,a)\Big], \enspace g(\omega; z, \theta^\star) = \!\!\inf_{\theta' \in \Theta_{\mathrm{alt}}}\!\E_{\rho}\Big[\sum_{a\in\A}\omega(x,a)d_{x,a}(\theta^\star,\theta') \Big] - \frac{1}{z}.
\end{align*}
The Lagrangian relaxation problem of \eqref{eq:optim-lb-z-mu} is\footnote{In the main text we actually state that \eqref{eq:lagr-rel} is the Lagrangian relaxation of \eqref{eq:optim-lb-z} instead of \eqref{eq:optim-lb-z-mu}. This is motivated by the fact that \eqref{eq:lagr-rel} and \eqref{eq:optim-lb-z} have the same optimal solution (see Prop. \ref{prop:pz-vs-pzmu}), though different optimal objective values.}
\begin{align}\label{eq:lagr-rel}\tag{P$_\lambda$}
\min_{\lambda \geq 0} \max_{\omega \in \Omega} \Big\{ h(\omega, \lambda; z, \theta^\star) := f(\omega; \theta^\star) + \lambda g(\omega; z,\theta^\star) \Big\},
\end{align}
where $\lambda\in\Re_{\geq 0}$ is a multiplier. We denote by $\lambda^\star(z, \theta^\star)$ the optimal multiplier for problem \eqref{eq:lagr-rel}. We note that $f$ is linear in $\omega$, while $g$ is concave since it is an infimum of affine functions. Hence, the maximization in \eqref{eq:lagr-rel} is  a non-smooth concave optimization problem.

\paragraph{Strong duality.} We now verify that strong duality holds for the Lagrangian formulation \eqref{eq:lagr-rel}  (with respect to \eqref{eq:optim-lb-z-mu}) when $z > \underline{z}(\theta^\star)$. This is immediate from the existence of a Slater point, as shown in the following proposition.
\begin{proposition}[Slater Condition]\label{prop:slater}
       For any $z > \underline{z}(\theta^\star)$, there exists a \emph{strictly} feasible solution $\underline\omega$, i.e., $g(\underline\omega; z, \theta^\star) > 0$.
\end{proposition}
\begin{proof}
        This is a direct consequence of the fact that 
\begin{align}
\max\limits_{\omega\in\Omega}\inf\limits_{\theta' \in \Theta_{\mathrm{alt}}}\E_{\rho}\left[\sum_{a\in\A}{\omega}(x,a)d_{x,a}(\theta^\star,\theta') \right] = \frac{1}{\underline{z}(\theta^\star)} > \frac{1}{z}.
\end{align}        
See Lem.~\ref{lem:z-opt-upperbound} and App.~\ref{app:lower.bound}.
\end{proof}
Thus, the optimal solution of \eqref{eq:lagr-rel} is $\left(\lambda^\star(z, \theta^\star), \omega^\star_z\right)$.

\paragraph{Boundedness of the optimal multipliers.} We recall the following basic result.

\begin{lemma}[Lemma 3 of \citep{nedic2009subgradient}]\label{lemma:mult-bound-nedic}
For any $z > \underline{z}(\theta^\star)$, if $\wb{\omega}_z$ is a Slater point for \eqref{eq:optim-lb-z-mu},
\begin{align*}
\lambda^\star(z, \theta^\star) \leq \frac{f(\omega^\star_z; \theta^\star) - f(\wb{\omega}_z; \theta^\star)}{g(\wb{\omega}_z;z,\theta^\star)}
\end{align*}
\end{lemma}
Using Lemma \ref{lemma:mult-bound-nedic}, we can prove the following result which will be very useful for the regret analysis.

\begin{lemma}\label{lemma:mult-bound}
For any $z \geq 2\underline{z}(\theta^\star)$,
\begin{align}
\lambda^\star(z, \theta^\star) \leq 2BL\underline{z}(\theta^\star).
\end{align}
\end{lemma}
\begin{proof}
From Prop. \ref{prop:slater}, $\underline{\omega}$ (the solution of the associated pure-exploration problem) is a Slater point for problem  \eqref{eq:optim-lb-z}. Then, by Lemma \ref{lemma:mult-bound-nedic},
\begin{align*}
        \lambda^\star(z, \theta^\star) \leq \frac{f(\omega^\star_z; \theta^\star) - f(\underline{\omega}; \theta^\star)}{g(\underline{\omega};z,\theta^\star)}.
\end{align*}
Let $\mathrm{kl}(\omega)$ denote the expected KL of $\omega$, so that $g(\omega;z,\theta^\star) = \mathrm{kl}(\omega) - 1/z$. Then,
\begin{align}
  \frac{f(\omega^\star_z; \theta^\star) - f(\underline{\omega}; \theta^\star)}{\mathrm{kl}(\underline{\omega}) - 1/z} \leq
  \frac{f(\omega^\star_z; \theta^\star)}{\mathrm{kl}(\underline{\omega}) - 1/z}  
  \leq  \frac{BL}{\mathrm{kl}(\underline{\omega}) - 1/z} .
\end{align}
 Furthermore, since $\mathrm{kl}(\underline{\omega}) = 1/\underline{z}(\theta^\star)$,
\begin{align*}
    \lambda^\star(z, \theta^\star) \leq \frac{BL z\underline{z}(\theta^\star)}{z - \underline{z}(\theta^\star)} \leq 2BL\underline{z}(\theta^\star),
\end{align*}
where the last inequality holds for $z \geq 2 \underline{z}(\theta^\star)$. This concludes the proof.
\end{proof}


\section{Action Sampling}\label{app:sampling}

\algo does not use standard tracking approaches for action selection (e.g., cumulative tracking~\citep{garivier2016optimal,degenne2019non} or direct tracking~\citep{combes2017minimal,hao2019adaptive}) but a sampling strategy.
Despite being simpler and more practical than tracking, we show that sampling from $\omega_t$ enjoys nice theoretical guarantees.

In the following lemmas we define the filtration $\F_t$ as the $\sigma$-algebra generated by the $t$-step history, $H_t = (X_1,A_1,Y_1, \dots, X_t, A_t, Y_t)$.

\begin{lemma}\label{lemma:pd-sampling}
Let $\{\omega_t\}_{t\geq 1}$ be such that $\omega_t \in \Omega$ and $\omega_t$ is $\F_{t-1}$-measurable.
Let $\{X_t\}_{t\geq1}$ be a sequence of i.i.d.\ contexts distributed according to $\rho$ and $\{A_t\}_{t\geq1}$ be such that $A_t \sim \omega_t(X_t, \cdot)$. Then,
\begin{align*}
    \sum_{t\geq1}\sum_{x\in\X}\sum_{a\in\A}\prob{E_t, \left| N_{t}^E(x,a) - \rho(x)\sum_{s\leq t : E_s} \omega_s(x,a) \right| > \sqrt{\frac{S_t}{2}\log \left(S_t^2|\X||\A|\right)}} \leq \frac{\pi^2}{3}.
\end{align*}
\end{lemma}
\begin{proof}
Let $Z_t := \indi{X_t=x, A_t=a}$ and $\tau_s$ be a random variable such that the $s$-th exploration round occurs at time $\tau_s + 1$. Notice that $\{\tau_s\}_{s\geq1}$ is a strictly-increasing sequence (i.e., $\tau_{s+1} > \tau_s$) of stopping times w.r.t.\ $\{\F_t\}_{t \geq 1}$. Furthermore, define
\begin{align*}
    W_s := Z_{\tau_s+1} - \rho(x)\omega_{\tau_s + 1}(x,a)
\end{align*}
and let $\mathcal{G}_s := \F_{\tau_{s+1}}$. Using Lem. 10 in~\cite{jian2019exploration}, we have that $\{W_s, \mathcal{G}_s\}_{s\geq 1}$ is a martingale difference sequence. Therefore, by Azuma's inequality
\begin{align*}
    \prob{\left|\sum_{i=1}^s W_i\right| > \sqrt{\frac{s}{2} \log \frac{2}{\delta}}} \leq \delta.
\end{align*}
Let $a_t := \sqrt{\frac{S_t}{2}\log \left(S_t^2|\X||\A|\right)}$ and rewrite $N_{t}^E(x,a) = \sum_{s\leq t : E_s} Z_s$. Fix any $\wb{t} \geq 1$. Then,
\begin{align*}
        \sum_{t=1}^{\wb{t}}&\indi{E_t, \left| \sum_{s\leq t : E_s} \left(Z_s - \rho(x) \omega_s(x,a)\right) \right| > a_t}\\
        &\leq \sum_{s \geq 1} \indi{\left| \sum_{i=1}^s \left(Z_{\tau_i + 1} - \rho(x) \omega_{\tau_i + 1}(x,a)\right) \right| > a_{\tau_s + 1}, \tau_s + 1 \leq \wb{t}}\\ &\leq \sum_{s \geq 1} \indi{\left| \sum_{i=1}^s W_i \right| > \sqrt{\frac{s}{2}\log \left(s^2|\X||\A|\right)}}.
\end{align*}
In the last inequality, we used the fact that $a_{\tau_s + 1} = \sqrt{s\log s}$. Taking expectations and applying Azuma's inequality with $\delta = \frac{2}{s^2|\X||\A|}$,
\begin{align*}
    \sum_{t=1}^{\wb{t}}\prob{E_t, \left| \sum_{s\leq t : E_s} \left(Z_s - \rho(x) \omega_s(x,a)\right) \right| > a_t} \leq \frac{1}{|\X||\A|}\sum_{s \geq 1} \frac{2}{s^2} = \frac{\pi^2}{3|\X||\A|}.
\end{align*}
The results holds for all $\wb{t}$, and the proof is concluded by summing over contexts and arms.
\end{proof}

\begin{lemma}\label{lemma:pd-sampling-func}
Let $\{\omega_t\}_{t\geq 1}$ be such that $\omega_t \in \Omega$ and $\omega_t$ is $\mathcal{F}_{t-1}$-measurable. Let $\{X_t\}_{t\geq1}$ be a sequence of i.i.d. contexts distributed according to $\rho$ and $\{A_t\}_{t\geq1}$ be such that $A_t \sim \omega_t(X_t, \cdot)$. Let $\{\varphi_t^i\}_{t\geq 1,i\in[m]}$ be a sequence of functions $\varphi_t^i : \X \times \A \rightarrow [-b,b]$ such that $\varphi_t^i(x,a)$ is $\mathcal{F}_{t-1}$-measurable for all $i\in[m]$. Then,
\begin{align*}
    \sum_{t\geq1}\sum_{i=1}^m\prob{E_t, \left| \sum_{s\leq t : E_s}\left(\varphi_s^i(X_s,A_s) - \sum_{x\in\X}\rho(x)\sum_{a\in\A}\omega_s(x,a)\varphi_s^i(x,a) \right)\right| > b\sqrt{\frac{S_t}{2}\log (mS_t^2)}} \leq \frac{\pi^2}{3}.
\end{align*}
\end{lemma}
\begin{proof}
The proof follows the same steps as the one of Lemma \ref{lemma:pd-sampling}. Fix $i \in [m]$. Let $Z_t := \varphi_t^i(X_t,A_t)$ and $\tau_s$ be a random variable such that the $s$-th exploration round occurs at time $\tau_s + 1$. Notice that $\{\tau_s\}_{s\geq1}$ is a strictly-increasing sequence (i.e., $\tau_{s+1} > \tau_s$) of stopping times w.r.t.\ $\{\mathcal{F}_t\}_{t \geq 1}$. Furthermore, define
\begin{align*}
    W_s := Z_{\tau_s+1} - \sum_{x\in\X}\sum_{a\in\A}\rho(x)\omega_{\tau_s + 1}(x,a)\varphi_{\tau_s + 1}^i(x,a)
\end{align*}
and let $\mathcal{G}_s := \mathcal{F}_{\tau_{s+1}}$. Using Lem. 10 in~\cite{jian2019exploration}, we have that $\{W_s, \mathcal{G}_s\}_{s\geq 1}$ is a martingale difference sequence (with differences bounded by $b$). Therefore, by Azuma's inequality
\begin{align*}
    \prob{\left|\sum_{i=1}^s W_i\right| > b\sqrt{\frac{s}{2} \log \frac{2}{\delta}}} \leq \delta.
\end{align*}
Let $a_t := b\sqrt{\frac{S_t}{2}\log \left(mS_t^2\right)}$ and fix some $\bar{t}\geq 1$. Then,
\begin{align*}
        \sum_{t=1}^{\bar{t}}&\indi{E_t, \left| \sum_{s\leq t : E_s} \left(Z_s - \sum_{x\in\X}\sum_{a\in\A}\rho(x) \omega_s(x,a)\varphi_s^i(x,a)\right) \right| > a_t}\\
        &\leq \sum_{s \geq 1} \indi{\left| \sum_{j=1}^s \left(Z_{\tau_j + 1} - \sum_{x\in\X}\sum_{a\in\A}\rho(x) \omega_{\tau_j + 1}(x,a)\varphi_{\tau_j+1}^i(x,a)\right) \right| > a_{\tau_s + 1}, \tau_s + 1 \leq \bar{t}}\\ &\leq \sum_{s \geq 1} \indi{\left| \sum_{j=1}^s W_j \right| > b\sqrt{\frac{s}{2}\log (ms^2)}}.
\end{align*}
In the last inequality, we used the fact that $a_{\tau_s + 1} = b\sqrt{\frac{s}{2}\log (ms^2)}$. Taking expectations and applying Azuma's inequality with $\delta = \frac{2}{ms^2}$,
\begin{align*}
    \sum_{t=1}^{\bar{t}}&\prob{E_t, \left| \sum_{s\leq t : E_s} \left(Z_s - \sum_{x\in\X}\sum_{a\in\A}\rho(x) \omega_s(x,a)\varphi_s^i(x,a)\right) \right| > a_t}\leq \sum_{s \geq 1} \frac{2}{ms^2} = \frac{\pi^2}{3m}.
\end{align*}
The results holds for all $\bar{t}$ and the proof follows by summing over all $i\in[m]$.
\end{proof}

\paragraph{Discussion.}

Lemma \ref{lemma:pd-sampling} provides an analogous result to those obtained by tracking strategies, where the empirical pull counts are shown close to the sequence of conditional probabilities computed by the optimizer. Despite being simpler, our sampling rule achieves similar efficiency as existing tracking rules. In particular, our bound scales with $\log |\A|$, a factor that appears in the tightest known analysis of cumulative tracking \citep{degenne2020structure}. The factor $\sqrt{S_t\log S_t}$ is not typically found in tracking strategies for MABs. However, we note that such dependency would naturally appear when generalizing these strategies to the contextual case.

Lemma \ref{lemma:pd-sampling-func} extends Lemma \ref{lemma:pd-sampling} to bound the deviation between expectations of measurable functions under the sequence of conditional probabilities and the same functions evaluated at the observed contexts/arms. This result will be very useful in the regret analysis to avoid undesirable linear dependencies on the number of arms.

\newpage

\section{High-Probability Events}\label{app:good.events}
In this section, we report the high-probability events used through the paper.
Refer to App.~\ref{app:concentrations} for concentration inequalities.

Let $\Phi_{x,a} := \phi(x,a)\phi(x,a)^T$. We define the following events:
\begin{align}
%
        &\textit{\small\color{gray} true regret close to objective values}\notag\\
        G_t^\Delta &:= \left\{ \left| \sum_{s\leq t : E_s}\left(\Delta_{\theta^\star}(X_s,A_s) - \sum_{x\in\X}\rho(x)\sum_{a\in\A}\omega_s(x,a)\Delta_{\theta^\star}(x,a) \right)\right| \leq 2LB\sqrt{S_t\log S_t} \right\},\\
        &\textit{\small\color{gray} true confidence intervals close to expected confidence intervals}\notag\\
        G_t^\phi &:= \left\{ \left| \sum_{s\leq t : E_s}\left(\| \phi(X_s,A_s)\|_{\bar{V}_{s-1}^{-1}} - \sum_{x\in\X}\rho(x)\sum_{a\in\A}\omega_s(x,a)\| \phi(x,a)\|_{\bar{V}_{s-1}^{-1}} \right)\right| \leq \frac{L}{\nu}\sqrt{S_t\log S_t} \right\},\\
                &\textit{\small\color{gray} true design matrix close to expected design matrix}\notag\\
            G_t^d &:= \left\{ \left\| \sum_{s\leq t : E_s} \left( \Phi_{X_s,A_s} - \sum_{x\in\X}\rho(x) \sum_{a\in\A}\omega_s(x,a)\Phi_{x,a} \right) \right\|_\infty  \leq L^2\sqrt{S_t\log \left(dS_t\right)} \right\},\\
                   &\textit{\small\color{gray} well-estimated context distribution}\notag\\
        G_t^\rho   &:= \left\{\forall x\in\X : |\wh{\rho}_{t-1}(x) - \rho(x)| \leq 2\max \left( \sqrt{\frac{\log(|\X|S_t^2)}{2S_t}}, \frac{2}{t}\right) \right\},\\
                   &\textit{\small\color{gray} well-estimated parameters}\notag\\
        G_t^\theta    &:= \left\{ \|\wh{\theta}_{t-1} - \theta^\star\|_{\wb{V}_{t-1}} \leq \sqrt{\gamma_t} \right\}.
\end{align}
Furthermore, we define $G_t := \{ G_t^\Delta, G_t^\phi, G_t^d, , G_t^\rho, G_t^\theta \}$ as the \emph{``good'' event} and let $M_t = \sum_{s=1}^t \indi{E_s, \neg G_s}$ be the number of exploration rounds in which the good event does not hold. This can be bounded in expectation as follows.

\begin{lemma}\label{lemma:pdlin-bad-event}
Let $M_t = \sum_{s=1}^t \indi{E_s, \neg G_s}$ be the number of exploration rounds in which the good event does not hold, then
\begin{align*}
    \expec{M_t} \leq  \frac{3\pi^2}{2}.
\end{align*}
\end{lemma}
\begin{proof}
Using the definition of $G_s$ together with the union bound,
\begin{align*}
    \expec{M_t} = \sum_{s=1}^t \prob{E_s, \neg G_s} \leq  \sum_{s=1}^t \prob{E_s, \neg G_s^\Delta} &+ \sum_{s=1}^t \prob{E_s, \neg G_s^\phi} + \sum_{s=1}^t \prob{E_s, \neg G_s^d}\\ &+ \sum_{s=1}^t \prob{E_s, \neg G_s^\rho} + \sum_{s=1}^t \prob{E_s, \neg G_s^\theta}.
\end{align*}
The first and second term can be bounded by Lemma \ref{lemma:pd-sampling-func} by noticing that $\Delta_{\theta^\star}(x,a) \leq 2LB$ and that $\| \phi(x,a)\|_{\bar{V}_{s-1}^{-1}}$ is $\F_{s-1}$-measurable and upper-bounded by $\frac{L}{\nu}$ at all time steps. Thus,
\begin{align*}
\sum_{s=1}^t \prob{E_s, \neg G_s^\Delta} + \sum_{s=1}^t \prob{E_s, \neg G_s^\phi} \leq \frac{2\pi^2}{3}.
\end{align*}
Similarly, the third term can be bounded by Lemma \ref{lemma:pd-sampling-func} by taking a union bound over all elements of $\Phi_{x,a}$ (for a total of $d^2$ elements) and noting that each term is bounded by $L^2$. Thus,
\begin{align*}
\sum_{s=1}^t \prob{E_s, \neg G_s^d} \leq \frac{\pi^2}{3}.
\end{align*}
The fourth term is
\begin{align*}
    \sum_{s=1}^t \prob{E_s, \neg G_s^\rho} &\leq \sum_{x\in\X}\sum_{s\geq 1} \prob{E_s, |\wh{\rho}_{s-1}(x) - \rho(x)| > 2\max \left( \sqrt{\frac{\log(|\X|S_s^2)}{2S_s}}, \frac{2}{s}\right)}\\ &\leq \sum_{x\in\X}\sum_{s\geq 1} \prob{E_s, |\wh{\rho}_{s-1}(x) - \wh{\rho}_{s}(x)| + |\wh{\rho}_{s}(x) - \rho(x)| > 2\max \left( \sqrt{\frac{\log(|\X|S_s^2)}{2S_s}}, \frac{2}{s}\right)}\\ & \leq \sum_{x\in\X}\sum_{s\geq 1} \prob{E_s, |\wh{\rho}_{s-1}(x) - \wh{\rho}_{s}(x)| > \frac{2}{s}} + \sum_{x\in\X}\sum_{s\geq 1} \prob{E_s, |\wh{\rho}_{s}(x) - \rho(x)| > \sqrt{\frac{\log(|\X|S_s^2)}{2S_s}}} \\ &\leq \frac{\pi^2}{3}.
\end{align*}
Here we used the fact that the absolute difference between two consecutive empirical means with samples bounded by $1$ cannot be larger than $\frac{2}{s}$. We also used Lemma \ref{lemma:conc-rho} to bound the second term.
Finally, the fifth term can be directly bounded by Lemma \ref{lemma:ci-exploration}:
\begin{align*}
    \sum_{s=1}^t \prob{E_s, \neg G_s^\theta} &\leq \frac{\pi^2}{6}.
\end{align*}
Combining the five bounds concludes the proof.
\end{proof}

\section{Regret Proof}\label{app:regret}

We start decomposing the regret based on whether $E_t$ holds or not:
\begin{align*}
    R_n = \sum_{t=1}^n \Delta_{\theta^\star}(X_t, A_t)\indi{\neg E_t} + \sum_{t=1}^n \Delta_{\theta^\star}(X_t, A_t)\indi{E_t} = R_n^{\mathrm{exploit}} +  R_n^{\mathrm{explore}}.
\end{align*}

Throughout the proof, as stated in the main theorem, we use $\beta_{t-1} := c_{n,1/n}$ and $\gamma_t := c_{n,1/S_t^2}$.

\subsection{Outline}\label{app:regret.outline}

An outline of our proof is as follows.
{
\renewcommand{\theenumi}{\textbf{Step \arabic{enumi}}}
\begin{enumerate}
\item (App. \ref{app:regret.exploitation}) Using the confidence set derived in App. \ref{app:conf.set}, we show that the regret suffered when the algorithm enters the exploitation step is finite;
\item (App. \ref{app:regret.exploration.first}) Using the properties of our action sampling strategy, we reduce the regret incurred during exploration rounds to the sum of objective values of the policies computed incrementally by primal-dual gradient ascent;
\item (App. \ref{app:regret.exploration.obj}) By combining standard tools from convex optimization with the properties of our confidence intervals, we relate the sum of objective values at each phase to the corresponding optimal value and constraint violations;
\item (App. \ref{app:regret.exploration.constr}) We relate the sum of constraints to the exploitation test used by SOLID. In particular, using the fact that the algorithm is not in the exploitation step, we show that the sum of constraints cannot be larger than $\mathcal{O}(\log n)$;
\item (App. \ref{app:regret.exploration.reg}) We combine the results obtained in the previous steps to show a first bound on the expected regret suffered during the exploration rounds. Our bound has the optimal dependency on $v^\star(\theta^\star)\log n$ but scales with the expected number $\expec{K_n}$ of phases executed by the algorithm;
\item (App. \ref{app:regret.exploration.phases}) By relating the upper bound on the sum of constraints computed at Step 3 and a lower bound on the same quantity, we obtain an upper bound on $K_n$ as a function of the chosen sequences $p_k, z_k$;
\item (App. \ref{app:regret.exploration.final}) We derive the final result by combining the bound on $K_n$ of Step 5 using the exponential schedule for $p_k, z_k$ with the partial regret bound of Step 4.
\end{enumerate}
}

\subsection{Regret during Exploitation} \label{app:regret.exploitation}

We show that the regret suffered when exploitation occurs is finite. Let $\beta_{t-1} := c_{n, 1/n}$, where $c_{n,\delta}$ was defined in Thm. \ref{th:conf-theta}. Then $F_t := \indi{\| \wh{\theta}_{t-1} - \theta^\star \|_{\wb{V}_{t-1}}^2 \leq c_{n,1/n}}$ is the event under which the true model belongs to the confidence set, which holds with probability at least $1-1/n$ by the same theorem.
We leverage this to decompose the regret during exploitation as:
\begin{align*}
    R_n^{\mathrm{exploit}} = \sum_{t=1}^n \Delta_{\theta^\star}(X_t, A_t)\indi{\neg E_t, F_t} + \sum_{t=1}^n \Delta_{\theta^\star}(X_t, A_t)\indi{\neg E_t, \neg F_t}.
\end{align*}
The expectation of the second term is bounded by
\begin{align*}
    \expec{\sum_{t=1}^n \underbrace{\Delta_{\theta^\star}(X_t, A_t)}_{\leq 2LB}\indi{\neg E_t, \neg F_t}} \leq 2LB \cdot \expec{\sum_{t=1}^n \indi{\neg F_t}} \leq 2LB\sum_{t=1}^n \underbrace{\prob{\neg F_t}}_{\leq 1/n} \leq 2LB,
\end{align*}
where we bounded $\prob{\neg F_t} \leq \frac{1}{n}$ by using Thm.~\ref{th:conf-theta} with $\delta = 1/n$.
Regarding the first term, we have two possible cases. If $a^\star_{\wt{\theta}_{t-1}}(X_t) = a^\star_{\theta^\star}(X_t)$, then the algorithm suffers no regret since by definition it pulls the empirically optimal arm (which is the optimal arm in this case). If $a^\star_{\wt{\theta}_{t-1}}(X_t) \neq a^\star_{\theta^\star}(X_t)$, then it must be that $\theta^\star \in \wb{\Theta}_{t-1}$, that is, the true model is in the set of alternative models for the current context. Under $\neg E_t$, this implies that
\begin{align*}
     \| \wh{\theta}_{t-1} - \theta^\star \|_{\wb{V}_{t-1}}^2 \geq  \| \wt{\theta}_{t-1} - \theta^\star \|_{\wb{V}_{t-1}}^2 \geq  \inf_{\theta' \in \wb{\Theta}_{t-1}} \| \wh{\theta}_{t-1} - \theta' \|_{\wb{V}_{t-1}}^2 > \beta_{t-1} = c_{n,1/n},
\end{align*}
where the first inequality is due to the fact that the
good event $F_t$ holds and Cor.~\ref{cor:proj-theta-star}. This is a contradiction with respect to $F_t$. Therefore, $\neg E_t$ and $F_t$ cannot hold at the same time and the algorithm suffers no regret. Combining these results, we conclude
\begin{align*}
    \expec{R_n^{\mathrm{exploit}}} \leq 2LB.
\end{align*}





\subsection{Regret under Exploration}\label{app:regret.exploration}

The key challenge is to bound the regret during the exploration rounds. We proceed by following the steps outlined in App. \ref{app:regret.outline}.

\subsubsection{From Regret to Objective Values}\label{app:regret.exploration.first}

We decompose the regret incurred during exploration as 
\begin{align*}
    R_n^{\mathrm{explore}} &:= \sum_{t=1}^n \Delta_{\theta^\star}(X_t, A_t)\indi{E_t} \leq \sum_{t=1}^n \Delta_{\theta^\star}(X_t, A_t)\indi{E_t, G_t} +
2LB\underbrace{
    \sum_{t=1}^n \indi{E_t, \neg G_t}
}_{:= M_n}
.  
\end{align*}
Refer to App.~\ref{app:good.events} for the definition of $G_t$.
The second term is $M_n$, the number of exploration rounds in which the good event does not hold, and can be bounded in expectation by using Lem.~\ref{lemma:pdlin-bad-event}. The first one can be bounded by using the good event. Suppose, without loss of generality, that $E_n$ and $G_n$ hold (if they do not, the following reasoning can be repeated for the last time step at which these events hold). Then, using $G_t^\Delta$ (see App. \ref{app:good.events}),
\begin{align*}
        \sum_{t=1}^n \Delta_{\theta^\star}(X_t, A_t)\indi{E_t, G_t} &= \sum_{t\leq n : E_t} \Delta_{\theta^\star}(X_t, A_t)\\ &\leq \sum_{t\leq n : E_t}\sum_{x\in\X}\rho(x)\sum_{a\in\A}{\omega}_t(x,a)\Delta_{\theta^\star}(x,a) + 2LB\sqrt{S_n\log S_n}.
\end{align*}
Using the definition of phase, we can rewrite the first summation as
\begin{align*}
\sum_{t\leq n : E_t}\sum_{x\in\X}\rho(x)\sum_{a\in\A}{\omega}_t(x,a)\Delta_{\theta^\star}(x,a) = \sum_{k=0}^{K_n} \sum_{t\in\mathcal{T}_k^E}\sum_{x\in\X}\rho(x)\sum_{a\in\A}{\omega}_t(x,a)\Delta_{\theta^\star}(x,a).
\end{align*}
Recall that $K_t$ is the (random) phase index at time $t$, while $\mathcal{T}_k^E$ is the set of exploration rounds in phase $k$. See App. \ref{app:notation.definition} for a summary of notation.
Let $\underline{k} := \min\{k\in\mathbb{N} | z_k \geq 2\underline{z}(\theta^\star)\}$. We split the sum into phases before and after $\underline{k}$. For those before, we have
\begin{align*}
 \sum_{k<\underline{k}} \sum_{t\in\mathcal{T}_k^E}\sum_{x\in\X}\rho(x)\sum_{a\in\A}{\omega}_t(x,a)\Delta_{\theta^\star}(x,a) \leq 2LB\sum_{k<\underline{k}} |\mathcal{T}_k^E| \leq 2LB\sum_{k<\underline{k}} p_k,
\end{align*}
which yields at most finite regret since $\{p_k\}$ is increasing. Let us now fix a phase $k\geq\underline{k}$ and bound the regret during its exploration rounds ($\mathcal{T}_k^E$). Note that the optimization problem in each phase $k \geq \underline{k}$ is feasible (see App. \ref{app:lagrangian}). We have
\begin{align*}
&\sum_{t\in\mathcal{T}_k^E}\sum_{x\in\X}\rho(x)\sum_{a\in\A}\omega_t(x,a) \Delta_{\theta^\star}(x,a)\\
        &\quad = \sum_{t\in\mathcal{T}_k^E : G_t}\sum_{x\in\X}\rho(x)\sum_{a\in\A}\omega_t(x,a) \Delta_{\theta^\star}(x,a) + \sum_{t\in\mathcal{T}_k^E : \neg G_t}\sum_{x\in\X}\rho(x)\sum_{a\in\A}\omega_t(x,a) (\mu^\star_{\theta^\star}(x) - \mu_{\theta^\star}(x,a))\\
        &\quad \leq \sum_{t\in\mathcal{T}_k^E : G_t}\sum_{x\in\X}\rho(x)\sum_{a\in\A}\omega_t(x,a) \Delta_{\theta^\star}(x,a) + M_{n,k} \mu^\star - \sum_{t\in\mathcal{T}_k^E : \neg G_t}\sum_{x\in\X}\rho(x)\sum_{a\in\A}\omega_t(x,a) \mu_{\theta^\star}(x,a).
\end{align*}
Here we defined $\mu^\star := \sum_{x\in\X}\rho(x)\mu^\star_{\theta^\star}(x)$ and $M_{n,k}$ as the number of exploration rounds during phase $k$ where the good event does not hold. The last term can be bounded by $M_{n,k}BL$. Regarding the remaining two,
\begin{align*}
    &\sum_{t\in\mathcal{T}_k^E : G_t}\sum_{x\in\X}\rho(x)\sum_{a\in\A}\omega_t(x,a) \Delta_{\theta^\star}(x,a) + M_{n,k}\mu^\star\\  &= (p_k - M_{n,k}) \mu^\star + M_{n,k} \mu^\star - \sum_{t\in\mathcal{T}_k^E : G_t}\sum_{x\in\X}\rho(x)\sum_{a\in\A}\omega_t(x,a) \mu_{\theta^\star}(x,a)\\ &= p_k \mu^\star + \underbrace{\sum_{t\in\mathcal{T}_k^E : G_t}\sum_{x\in\X}(\hat{\rho}_{t-1}(x) - \rho(x))\sum_{a\in\A}\omega_t(x,a) \mu_{\theta^\star}(x,a)}_{(a)} - \underbrace{\sum_{t\in\mathcal{T}_k^E : G_t}\sum_{x\in\X}\hat{\rho}_{t-1}(x)\sum_{a\in\A}\omega_t(x,a) \mu_{\theta^\star}(x,a)}_{(b)}.
\end{align*}
Term (a) can be bounded by
\begin{align*}
(a) \leq LB \underbrace{\sum_{t\in\mathcal{T}_k^E : G_t}\sum_{x\in\X}|\hat{\rho}_{t-1}(x) - \rho(x)|}_{\zeta_{n,k}}.
\end{align*}
The second term $\zeta_{n,k}$ will be bounded shortly over all phases by means of Lemma \ref{lemma:pdlin-rho-dev}.
We now provide a lower bound to term (b). The first step is to relate this to the objective function optimized by the algorithm. Using the definition of $G_t$ and Lem. \ref{lemma:good-event-mu},
\begin{align}
    (b) &\geq \notag \sum_{t\in\mathcal{T}_k^E : G_t}\sum_{x\in\X}\hat{\rho}_{t-1}(x)\sum_{a\in\A}\omega_t(x,a) \left( {\mu}_{\wt{\theta}_{t-1}}(x,a) - \sqrt{\gamma_t}\| \phi(x,a)\|_{\bar{V}_{t-1}^{-1}} \right) \\ \notag
        &\pm \sum_{t\in\mathcal{T}_k^E : \neg G_t}\sum_{x\in\X}\hat{\rho}_{t-1}(x)\sum_{a\in\A}\omega_t(x,a) \underbrace{{\mu}_{\wt{\theta}_{t-1}}(x,a)}_{|\cdot|\leq LB} \pm \sum_{t\in\mathcal{T}_k^E}\sum_{x\in\X}\hat{\rho}_{t-1}(x)\sum_{a\in\A}\omega_t(x,a) \sqrt{\gamma_t}\| \phi(x,a)\|_{\bar{V}_{t-1}^{-1}}\\ &\geq \notag \sum_{t\in\mathcal{T}_k^E} f_t(\omega_t) - M_{n,k}BL - 2\sum_{t\in\mathcal{T}_k^E}\sum_{x\in\X}\hat{\rho}_{t-1}(x)\sum_{a\in\A}\omega_t(x,a) \sqrt{\gamma_t}\| \phi(x,a)\|_{\bar{V}_{t-1}^{-1}}\\ &\geq \sum_{t\in\mathcal{T}_k^E} f_t(\omega_t) - M_{n,k}BL - 2\sqrt{\gamma_n}\Psi_{n,k}. \label{eq:reg-exp-obj}
\end{align}
In the last step, we used $\sqrt{\gamma_t} \leq \sqrt{\gamma_n}$ (which is by definition $\mathcal{O}(\log S_n)$) and defined $\Psi_{n,k} := \sum_{t\in\mathcal{T}_k^E}\sum_{x\in\X}\hat{\rho}_{t-1}(x)\sum_{a\in\A}\omega_t(x,a) \| \phi(x,a)\|_{\bar{V}_{t-1}^{-1}}$. 

To wrap-up the regret bound we have obtained so far, summing over all phases,
\begin{align*}
R_n^{\mathrm{explore}} &\leq 2LB\sum_{k<\underline{k}} p_k + \sum_{k \geq \underline{k}}^{K_n} p_k \mu^\star +  LB\underbrace{\sum_{k \geq \underline{k}}^{K_n} \zeta_{n,k}}_{\leq \zeta_n} -  \sum_{k \geq \underline{k}}^{K_n} \sum_{t\in\mathcal{T}_k^E} f_t(\omega_t)\\ &+  2LB\underbrace{\sum_{k \geq \underline{k}}^{K_n} M_{n,k}}_{\leq M_n} + 2LBM_n + 2\sqrt{\gamma_n} \underbrace{\sum_{k \geq \underline{k}}^{K_n} \Psi_{n,k}}_{\leq \Psi_n} + 2LB\sqrt{S_n\log S_n}.
\end{align*}
Here we defined
\begin{align*}
\zeta_n := \sum_{t\leq n : E_t, G_t}\sum_{x\in\X}|\hat{\rho}_{t-1}(x) - \rho(x)|
\end{align*}
and
\begin{align*}
\Psi_{n} := \sum_{t\leq n : E_t}\sum_{x\in\X}\hat{\rho}_{t-1}(x)\sum_{a\in\A}\omega_t(x,a) \| \phi(x,a)\|_{\bar{V}_{t-1}^{-1}}.
\end{align*}
$\zeta_n$ can be bounded by Lemma \ref{lemma:pdlin-rho-dev} and $\Psi_n$ by Lemma \ref{lemma:pdlin-Lt}. Both terms are of order $\mathcal{O}(\sqrt{S_n \log S_n})$. In order to simplify notation, we keep the specific bounds implicit in the remaining. Therefore, our partial regret bound is
\begin{align}\label{eq:pdlin-phased-reg}
\notag R_n^{\mathrm{explore}} &\leq 2LB\sum_{k<\underline{k}} p_k + \sum_{k \geq \underline{k}}^{K_n} p_k \mu^\star -  \sum_{k \geq \underline{k}}^{K_n} \sum_{t\in\mathcal{T}_k^E} f_t(\omega_t)\\ &+ 4LB M_n + 2\sqrt{\gamma_n} \Psi_n + LB\zeta_n + 2LB\sqrt{S_n\log S_n}.
\end{align}

\subsubsection{Bounding the Sum of Objective Values}\label{app:regret.exploration.obj}

Our goal here is to lower bound the sum of objective values. As before, fix some phase index $k\geq \underline{k}$ and let $\lambda \geq 0$ be arbitrary. By recalling that the optimization process is reset at the beginning of each phase and using Corollary \ref{cor:rec-pd} with $\alpha_k^\lambda = \alpha_k^\omega = 1/\sqrt{p_k}$ and $\omega = \omega^\star_{z_k}$ (the optimal solution of problem (P$_{z_k}$)),
\begin{align}\label{eq:pdlin-optim}
    \sum_{t\in\mathcal{T}_k^E} f_t(\omega_t)  \geq \sum_{t\in\mathcal{T}_k^E} h_t(\omega^\star_{z_k}, \lambda_t, z_k) - \lambda \sum_{t\in\mathcal{T}_k^E}g_t(\omega_t, z_k) - \left(\log |\A| + \frac{b_\omega^2 + b_\lambda^2}{2} + \frac{(\lambda - \lambda_1)^2}{2} \right)\sqrt{p_k}.
\end{align}
We recall that $b_\lambda$ and $b_\omega$ are the maximum sub-gradients in $\lambda$ and $\omega$, respectively. We now lower-bound the first term on the right-hand side. Since $h_t(\omega^\star_{z_k}, \lambda_t, z_k) = f_t(\omega^\star_{z_k}) + \lambda_t g_t(\omega^\star_{z_k}, z_k)$, $f_t(\omega^\star_{z_k}) \geq -LB$, $g_t(\omega^\star_{z_k}, z_k) \geq -\frac{1}{z_k}$, and $\lambda_t \leq \lambda_{\max}$, this term, evaluated on those steps where $G_t$ does not hold, can be lower-bounded by $\sum_{t\in\mathcal{T}_k^E : \neg G_t} h_t(\omega^\star_{z_k}, \lambda_t, z_k) \geq -(LB + \lambda_{\max}/z_k)M_{n,k}$. For any step $t\in\mathcal{T}_k^E$ in which $G_t$ holds, the optimism property (Lemma \ref{lemma:pdlin-optimism}) yields
\begin{align*}
    f_t(\omega^\star_{z_k}) &\geq \sum_{x\in\X} (\hat{\rho}_{t-1}(x) - \rho(x)) \underbrace{\sum_{a\in\A} \omega^\star_{z_k}(x,a) \mu_{\theta^\star}(x,a)}_{|\cdot| \leq LB} + f(\omega^\star_{z_k}) \\ &\geq f(\omega^\star_{z_k}) - LB\sum_{x\in\X}|\hat{\rho}_{t-1}(x) - \rho(x)|,
\end{align*}
and
\begin{align*}
   g_t(\omega^\star_{z_k}, z_k) &\geq \inf_{\theta' \in {\Theta}_{alt}}\sum_{x\in\X}\hat{\rho}_{t-1}(x)\sum_{a\in\A}\omega^\star_{z_k}(x,a) d_{x,a}(\theta^\star,\theta') - \frac{1}{z_k} \pm g(\omega^\star_{z_k})\\ &\geq \inf_{\theta' \in {\Theta}_{alt}}\sum_{x\in\X}(\hat{\rho}_{t-1}(x)-\rho(x))\sum_{a\in\A}\omega^\star_{z_k}(x,a) d_{x,a}(\theta^\star,\theta') + g(\omega^\star_{z_k}) \\ &\geq g(\omega^\star_{z_k}) - \frac{2L^2B^2}{\sigma^2}\sum_{x\in\X}|\hat{\rho}_{t-1}(x) - \rho(x)|.
\end{align*}
Combining these two and using $\lambda_t \leq \lambda_{\max}$,
\begin{align*}
    \sum_{t\in\mathcal{T}_k^E : G_t} h_t(\omega^\star_{z_k}, \lambda_t, z_k) \geq \sum_{t\in\mathcal{T}_k^E : G_t} \left(f(\omega^\star_{z_k}) + \lambda_t g(\omega^\star_{z_k})\right) - LB\left(1 + \frac{2LB\lambda_{\mathrm{max}}}{\sigma^2}\right)\zeta_{n,k}.
\end{align*}
Note that $g(\omega^\star_{z_k}) \geq 0$ since by assumption $\omega^\star_{z_k}$ is feasible for the optimization problem $(P_{z_k})$. Furthermore, $\sum_{t\in\mathcal{T}_k^E : G_t} f(\omega^\star_{z_k}) = \sum_{t\in\mathcal{T}_k^E} f(\omega^\star_{z_k}) - \sum_{t\in\mathcal{T}_k^E : \neg G_t} \underbrace{f(\omega^\star_{z_k})}_{|\cdot|\leq LB} \geq p_k f(\omega^\star_{z_k}) - LBM_{n,k}$. Therefore, we obtain the following lower-bound on the sum of optimal objective values:
\begin{align*}
    \sum_{t\in\mathcal{T}_k^E} h_t(\omega^\star_{z_k}, \lambda_t, z_k) \geq p_k f(\omega^\star_{z_k}) - LB\left(1 + \frac{2LB\lambda_{\mathrm{max}}}{\sigma^2}\right)\zeta_{n,k}  -(2LB + \lambda_{\max}/z_k)M_{n,k}.
\end{align*}

Plugging this back into \eqref{eq:pdlin-optim},
\begin{align}
    \notag\sum_{t\in\mathcal{T}_k^E} f_t(\omega_t)  \geq p_k f(\omega^\star_{z_k}) &- \lambda \sum_{t\in\mathcal{T}_k^E}g_t(\omega_t, z_k) - a_\lambda \sqrt{p_k}\\ &- LB\left(1 + \frac{2LB\lambda_{\mathrm{max}}}{\sigma^2}\right)\zeta_{n,k}  -(2LB + \lambda_{\max}/z_k)M_{n,k},\label{eq:pdlin-phased-f-k}
\end{align}
where, for simplicity, we defined $a_\lambda:= \left(\log |\A| + \frac{b_\omega^2 + b_\lambda^2}{2} + \frac{(\lambda - \lambda_1)^2}{2} \right)$. Summing over all phases,
\begin{align}
    \notag\sum_{k \geq \underline{k}}^{K_n} \sum_{t\in\mathcal{T}_k^E} f_t(\omega_t) \geq \sum_{k \geq \underline{k}}^{K_n}  p_k f(\omega^\star_{z_k}) &- \lambda \sum_{k \geq \underline{k}}^{K_n} \sum_{t\in\mathcal{T}_k^E}g_t(\omega_t, z_{k}) - a_\lambda \sum_{k \geq \underline{k}}^{K_n} \sqrt{p_k} \\ & -LB\left(1 + \frac{2LB\lambda_{\mathrm{max}}}{\sigma^2}\right)\zeta_{n}  -(2LB + \lambda_{\max})M_{n},\label{eq:pdlin-phased-f}
\end{align}
where we used $\sum_{k \geq \underline{k}}^{K_n}  M_{n,k} \leq M_n$, $\sum_{k \geq \underline{k}}^{K_n}  \zeta_{n,k} \leq \zeta_n$, and $z_k \geq 1$.

\subsubsection{Bounding the sum of constraints}\label{app:regret.exploration.constr}

Our next step is to upper bound $\sum_{k \geq \underline{k}}^{K_n} \sum_{t\in\mathcal{T}_k^E}g_t(\omega_t, z_{k})$, the sum of constraints of the policies played by the algorithm during feasible phases (those with $z_k \geq 2\underline{z}(\theta^\star)$). The intuition is that this term cannot be large (i.e., it cannot be above $\mathcal{O}(\log n)$), otherwise the exploitation test would trigger and we would not be exploring at step $n$. Using the definition of $g_t(\omega, z_{k})$ (Eq.~\ref{eq:pdlin.gt}) and splitting the sum based on the good event 
\begin{align*}
        &\sum_{k \geq \underline{k}}^{K_n} \sum_{t\in\mathcal{T}_k^E}g_t(\omega_t, z_{K_t})\\ &\leq \sum_{t\leq n: E_t} \inf_{\theta' \in \bar{\Theta}_{t-1}}\sum_{x\in\X}\hat{\rho}_{t-1}(x)\sum_{a\in\A}\omega_t(x,a) {d}_{x,a}(\wt{\theta}_{t-1},\theta') + \frac{2LB}{\sigma^2}\sqrt{\gamma_n}\Psi_n - \sum_{k \geq \underline{k}}^{K_n} \sum_{t\in\mathcal{T}_k^E}\frac{1}{z_{k}} \\
                                          &\leq \underbrace{\sum_{t\leq n: E_t, G_t} \inf_{\theta' \in \bar{\Theta}_{t-1}}\sum_{x\in\X}\hat{\rho}_{t-1}(x)\sum_{a\in\A}\omega_t(x,a) {d}_{x,a}(\wt{\theta}_{t-1},\theta')}_{\textcircled{1}} + \frac{2L^2B^2}{\sigma^2}M_n + \frac{2LB}{\sigma^2}\sqrt{\gamma_n}\Psi_n - \sum_{k \geq \underline{k}}^{K_n} \frac{p_k}{z_k}.
\end{align*}
Note that in the first step above we implicitly upper bounded the sum of KLs on the feasible phases with the sum of KLs over all exploration rounds. We can use the definition of $G_t$ and the optimism (Lemma \ref{lemma:pdlin-optimism}) to upper bound the first sum by
\begin{align*}
        \textcircled{1} &\leq \sum_{t\leq n: E_t, G_t} \inf_{\theta' \in \bar{\Theta}_{t-1}}\sum_{x\in\X}\hat{\rho}_{t-1}(x)\sum_{a\in\A}\omega_t(x,a) d_{x,a}(\theta^\star,\theta') + \frac{2LB}{\sigma^2}\sqrt{\gamma_n}\Psi_n\\ &\leq \sum_{t\leq n: E_t, G_t} \underbrace{\inf_{\theta' \in \bar{\Theta}_{t-1}}\sum_{x\in\X}{\rho}(x)\sum_{a\in\A}\omega_t(x,a) d_{x,a}(\theta^\star,\theta')}_{\textcircled{2}} + \frac{2L^2B^2}{\sigma^2}\underbrace{\sum_{t\leq n: E_t, G_t}\sum_{x\in\X}|{\rho}(x) - \hat{\rho}_{t-1}(x)|}_{= \zeta_n}\\ & \quad + \frac{2LB}{\sigma^2}\sqrt{\gamma_n}\Psi_n.
\end{align*}
Furthermore, the first term can be upper bounded by replacing each set $\bar{\Theta}_{t-1}$ over which the infimum is taken by ${\Theta}_{alt}$ (if the two sets were different, such term would be zero). Therefore,
\begin{align}\label{eq:pdlin-sum-constr}
        \notag\textcircled{2} &\leq \sum_{t\leq n: E_t, G_t} \inf_{\theta' \in {\Theta}_{alt}}\sum_{x\in\X}{\rho}(x)\sum_{a\in\A}\omega_t(x,a) d_{x,a}(\theta^\star,\theta') 
        \\ &\leq  \underbrace{\inf_{\theta' \in {\Theta}_{alt}} \sum_{t\leq n: E_t} \sum_{x\in\X}{\rho}(x)\sum_{a\in\A}\omega_t(x,a) d_{x,a}(\theta^\star,\theta')}_{\textcircled{3}},
\end{align}
where we moved the infimum outside the outer sum and added the remaining steps where $G_t$ does not hold. Let $\Phi_{x,a} := \phi(x,a)\phi(x,a)^T$ and $V_{n,e} := \sum_{t\leq n: E_t}\Phi_{X_t,A_t}$ be the design matrix of the exploration rounds. Using the definition of $d_{x,a}$,
\begin{align*}
\textcircled{3} &= \frac{1}{2\sigma^2} \inf_{\theta' \in {\Theta}_{alt}} (\theta^\star - \theta')^T \left(\sum_{t\leq n: E_t} \sum_{x\in\X}{\rho}(x)\sum_{a\in\A}\omega_t(x,a) \Phi_{x,a} \pm V_{n,e} \right)(\theta^\star - \theta')\\ & \leq  \inf_{\theta' \in {\Theta}_{alt}} \left\{\frac{1}{2\sigma^2} (\theta^\star - \theta')^T  V_{n,e}(\theta^\star - \theta') +  \frac{1}{2\sigma^2} \|\theta^\star - \theta'\|_2^2 \left\|\sum_{t\leq n: E_t} \sum_{x\in\X}{\rho}(x)\sum_{a\in\A}\omega_t(x,a) \Phi_{x,a}  - V_{n,e} \right\|_2\right\}\\ &\leq \inf_{\theta' \in {\Theta}_{alt}} \sum_{x\in\X}\sum_{a\in\A}N_{n}^E(x,a) d_{x,a}(\theta^\star,\theta') + \frac{2B^2}{\sigma^2}\left\|\sum_{t\leq n: E_t} \sum_{x\in\X}{\rho}(x)\sum_{a\in\A}\omega_t(x,a) \Phi_{x,a} - V_{n,e} \right\|_2\\ &\leq \inf_{\theta' \in {\Theta}_{alt}} \sum_{x\in\X}\sum_{a\in\A}N_{n}^E(x,a) d_{x,a}(\theta^\star,\theta') + \frac{2B^2\sqrt{d}}{\sigma^2}\left\|\sum_{t\leq n: E_t} \sum_{x\in\X}{\rho}(x)\sum_{a\in\A}\omega_t(x,a) \Phi_{x,a}  - V_{n,e} \right\|_\infty.
\end{align*}

Recall that $G_n$ holds. Then, by using the definition of $G^d$ to bound the norm,
\begin{align*}
\textcircled{3} \leq \underbrace{\inf_{\theta' \in {\Theta}_{alt}} \sum_{x\in\X}\sum_{a\in\A}N_{n-1}^E(x,a) d_{x,a}(\theta^\star,\theta') }_{\textcircled{4}} + \frac{2B^2L^2}{\sigma^2} + \frac{2B^2L^2}{\sigma^2}\sqrt{dS_n\log \left(dS_n\right)}.
\end{align*}
Here we used $N_n(x,a) = N_{n-1}(x,a) + \indi{X_n=x,A_n=a}$ and upper bounded the KL at round $n$ by its maximum value.
Moreover, similarly to Lem.~\ref{lemma:pdlin-optimism} we can show that
\begin{align*}
\textcircled{4} \leq \inf_{\theta' \in {\Theta}_{alt}} \sum_{x\in\X}\sum_{a\in\A}N_{n-1}^E(x,a) {d}_{x,a}(\wt{\theta}_{n-1},\theta') + \frac{2LB\sqrt{\gamma_n}}{\sigma^2}\underbrace{\sum_{x\in\X}\sum_{a\in\A}N_{n-1}^E(x,a)\|\phi(x,a)\|_{\bar{V}_{n-1}^{-1}}}_{\leq \Psi_n}.
\end{align*}
The upper bound on the second term can be extracted from the proof of Lemma \ref{lemma:pdlin-Lt}. The first term can be finally related to the exploitation test:
\begin{align*}
        \inf_{\theta' \in {\Theta}_{alt}} \sum_{x\in\X}\sum_{a\in\A}N_{n-1}^E(x,a) {d}_{x,a}(\wt{\theta}_{n-1},\theta') &\leq \inf_{\theta' \in \bar{\Theta}_{n-1}} \sum_{x\in\X}\sum_{a\in\A}N_{n-1}^E(x,a) {d}_{x,a}(\wt{\theta}_{n-1},\theta') \\ &= \frac{1}{2\sigma^2}\inf_{\theta'\in\bar{\Theta}_{n-1}} \| \wt{\theta}_{n-1} - \theta' \|_{{V}_{n-1}}^2  \\ &\leq \frac{1}{2\sigma^2}\inf_{\theta'\in\bar{\Theta}_{n-1}} \| \wt{\theta}_{n-1} - \theta' \|_{\bar{V}_{n-1}}^2
        \leq \frac{\beta_{n-1}}{2\sigma^2},
\end{align*}
where the second-last inequality holds since $\bar{V}_{n-1} \succeq V_{n-1}$, and the last inequality holds since the algorithm is exploring at step $n$. By gathering all the results together, we get 
\begin{align}
    \notag \sum_{k \geq \underline{k}}^{K_n} \sum_{t\in\mathcal{T}_K^E : E_t}g_t(\omega_t, z_{K_t}) \leq \frac{\beta_{n-1}}{2\sigma^2} &- \sum_{k \geq \underline{k}}^{K_n} \frac{p_k}{z_k} + \frac{2L^2B^2}{\sigma^2}M_n+ \frac{6LB}{\sigma^2}\sqrt{\gamma_n}\Psi_n + \frac{2B^2L^2}{\sigma^2}\zeta_n \\ &+  \frac{2B^2L^2}{\sigma^2}\left(\sqrt{dS_n\log \left(dS_n\right)} + 1 \right).\label{eq:pdlin-phased-g}
\end{align}

\subsubsection{Back to the regret during exploration}\label{app:regret.exploration.reg}

So far we have (1) reduced the total regret during exploration to the sum of objective values (Eq. \ref{eq:pdlin-phased-reg}), (2) related this quantity to the optimal values of each phase (Eq. \ref{eq:pdlin-phased-f}), and (3) derived an upper bound to the total sum of constraints (Eq. \ref{eq:pdlin-phased-g}). We now combine all these results. If we first plug \eqref{eq:pdlin-phased-f} into \eqref{eq:pdlin-phased-reg},
\begin{align}
\notag R_n^{\mathrm{explore}} &\leq 2LB\sum_{k<\underline{k}} p_k + \sum_{k \geq \underline{k}}^{K_n} p_k \mu^\star -  \sum_{k \geq \underline{k}}^{K_n}  p_k f(\omega^\star_{z_k}) + \lambda \sum_{k \geq \underline{k}}^{K_n} \sum_{t\in\mathcal{T}_k^E}g_t(\omega_t, z_{k}) + a_\lambda \sum_{k \geq \underline{k}}^{K_n} \sqrt{p_k}\\ &+ (6LB + \lambda_{\max})M_n + 2\sqrt{\gamma_n} \Psi_n + LB\left(2 + \frac{2LB\lambda_{\mathrm{max}}}{\sigma^2}\right)\zeta_n + 2LB\sqrt{S_n\log S_n}.
\end{align}
Then, plugging \eqref{eq:pdlin-phased-g} into this inequality,
\begin{align}
\notag R_n^{\mathrm{explore}} &\leq 2LB\sum_{k<\underline{k}} p_k + \sum_{k \geq \underline{k}}^{K_n} p_k \mu^\star -  \sum_{k \geq \underline{k}}^{K_n}  p_k f(\omega^\star_{z_k}) + \lambda \frac{\beta_{n-1}}{2\sigma^2} - \lambda \sum_{k \geq \underline{k}}^{K_n} \frac{p_k}{z_k} + a_\lambda \sum_{k \geq \underline{k}}^{K_n} \sqrt{p_k}\\ \notag &+ \left(\lambda\frac{2L^2B^2}{\sigma^2} + 6LB + \lambda_{\max}\right)M_n + \left(2 + \frac{6LB\lambda}{\sigma^2}\right)\sqrt{\gamma_n} \Psi_n + 2LB\sqrt{S_n\log S_n}\\ &+ LB\left(2 + \frac{2LB(\lambda_{\mathrm{max}} + \lambda)}{\sigma^2}\right)\zeta_n +  \frac{2\lambda B^2L^2}{\sigma^2}\left(\sqrt{dS_n\log \left(dS_n\right)} + 1\right).
\end{align}
Let us simplify this expression so that it becomes more readable. First, we note that
\begin{align*}
\sum_{k \geq \underline{k}}^{K_n} p_k \mu^\star -  \sum_{k \geq \underline{k}}^{K_n}  p_k f(\omega^\star_{z_k}) =  \sum_{k\geq\underline{k}}^{K_n} \frac{p_k}{z_k} \underbrace{z_k(\mu^\star - f(\omega^\star_{z_k}))}_{=u^\star(z_k,\theta^\star)} = \sum_{k\geq\underline{k}}^{K_n} \frac{p_k}{z_k}u^\star(z_k,\theta^\star).
\end{align*}
Taking the expectation of both sides, we obtain
\begin{align*}
    \expec{ R_n^{\mathrm{explore}}} \leq 2LB\sum_{k<\underline{k}} p_k  &+ \expec{\sum_{k\geq\underline{k}}^{K_n} \frac{p_k}{z_k}u^\star(z_k,\theta^\star)} + \lambda \frac{\beta_{n-1}}{2\sigma^2} - \lambda\expec{\sum_{k\geq\underline{k}}^{K_n}\frac{p_k}{z_k}}\\ &+ a_\lambda \expec{\sum_{k\geq\underline{k}}^{K_n}\sqrt{p_k}}+ \expec{\mathcal{O}(\sqrt{S_n\log S_n})}.
\end{align*}
The remaining expectations on the right-hand side are due to the fact that $K_n$ (hence $S_n$) is still random. 
Setting $\lambda = v^\star(\theta^\star)$ and combining the second and fourth terms, we get
\begin{align*}
    \sum_{k\geq\underline{k}}^{K_n} \frac{p_k}{z_k}u^\star(z_k,\theta^\star) &- \lambda{\sum_{k\geq\underline{k}}^{K_n}\frac{p_k}{z_k}} = \sum_{k\geq\underline{k}}^{K_n} \frac{p_k}{z_k}\left(u^\star(z_k,\theta^\star) - v^\star(\theta^\star)\right) \\ &= \sum_{k\geq\underline{k} : z_k < \bar{z}(\theta^\star)} \frac{p_k}{z_k}\left(u^\star(z_k,\theta^\star) - v^\star(\theta^\star)\right) + \sum_{k : z_k \geq \bar{z}(\theta^\star)}^{K_n} \frac{p_k}{z_k}\left(u^\star(z_k,\theta^\star) - v^\star(\theta^\star)\right),
\end{align*}
where $\bar{z}(\theta^\star) := \max_{x\in\X}\sum_{a \neq a^\star_{\theta^\star}(x)} \frac{\eta^\star(x,a)}{\rho(x)}$ was defined in Lem.~\ref{lem:z-opt-upperbound}. For $k\geq\underline{k}$, we can use the perturbation bound (Lem.~\ref{lem:z-opt-upperbound}) on both terms. We obtain,
\begin{align*}
\sum_{k\geq\underline{k} : z_k < \bar{z}(\theta^\star)} \frac{p_k}{z_k}\left(u^\star(z_k,\theta^\star) - v^\star(\theta^\star)\right) \leq BL\underline{z}(\theta^\star) \sum_{k\geq\underline{k} : z_k < \bar{z}(\theta^\star)} \frac{p_k}{z_k - \underline{z}(\theta^\star)}
\end{align*}
and
\begin{align*}
\sum_{k\geq\underline{k} : z_k \geq \bar{z}(\theta^\star)}^{K_n} \frac{p_k}{z_k}\left(u^\star(z_k,\theta^\star) - v^\star(\theta^\star)\right) \leq BL\underline{z}(\theta^\star) \sumpull(\theta^\star) \sum_{k\geq\underline{k} : z_k \geq \bar{z}(\theta^\star)}^{K_n} \frac{p_k}{z_k - \underline{z}(\theta^\star)}\max\left\{
                \frac{c_\Theta \sqrt{2}}{\sigma\sqrt{z_k}}, \frac{1}{z_k}
            \right\}
\end{align*}
\paragraph{Partial regret bound} 

\newcommand{\RomanNumeralCaps}[1]
    {\MakeUppercase{\romannumeral #1}}

Plugging these bounds into the expected regret,
\begin{align}\label{eq:pdlin-phased-reg-partial}
    &\notag\expec{ R_n^{\mathrm{explore}}} \leq \underbrace{2BL\sum_{k<\underline{k}} p_k}_{\text{\RomanNumeralCaps{1}}}  +  \underbrace{BL\underline{z}(\theta^\star)\sum_{k\geq\underline{k} : z_k < \bar{z}(\theta^\star)} \frac{p_k}{z_k - \underline{z}(\theta^\star)}}_{\text{\RomanNumeralCaps{2}}} + \underbrace{v^\star(\theta^\star) \frac{\beta_{n-1}}{2\sigma^2}}_{\text{\RomanNumeralCaps{3}}} + \underbrace{a_\lambda \expec{\sum_{k\geq\underline{k}}^{K_n}\sqrt{p_k}}}_{\text{\RomanNumeralCaps{4}}} \\ &+ \underbrace{BL\underline{z}(\theta) \sumpull(\theta^\star) \expec{\sum_{k : z_k \geq \bar{z}(\theta^\star)}^{K_n} \frac{p_k}{z_k - \underline{z}(\theta^\star)}\max\left\{
                \frac{c_\Theta \sqrt{2}}{\sigma\sqrt{z_k}}, \frac{1}{z_k}
            \right\}}}_{\text{\RomanNumeralCaps{5}}} + \underbrace{\expec{\mathcal{O}(\sqrt{S_n\log S_n})}}_{\text{\RomanNumeralCaps{6}}}.
\end{align}
The six terms constituting the bound are (from left to right):
{
\renewcommand{\theenumi}{\Roman{enumi}}
\begin{enumerate}
\item finite regret suffered in the phases where the optimization problem is infeasible;
\item finite regret suffered in the phases in which we do not know much about the convergence rate of $u^\star(z,\theta^\star)$ to $v^\star(\theta^\star)$. This term is likely an artefact of the analysis;
\item asymptotically-optimal regret rate;
\item regret suffered due to the incremental gradient updates and inversely proportional to the step sizes;
\item regret suffered due to the fact that we solve \eqref{eq:optim-lb-z} instead of \eqref{eq:optim-lb};
\item other low-order terms mostly due to the concentration bounds.
\end{enumerate}
}

Note that, since $\beta_{n-1} = c_{n,1/n}$ and $c_{n,1/n} \rightarrow 2\sigma^2\log n$ as $n \rightarrow \infty$,
\begin{align*}
    \limsup_{n \rightarrow \infty} \frac{v^\star(\theta^\star)\beta_{n-1}}{2\sigma^2\log n} = v^\star(\theta^\star),
\end{align*}
which is the asymptotically-optimal regret rate as prescribed by \eqref{eq:optim-lb}.

\subsubsection{Bounding the total number of phases}\label{app:regret.exploration.phases}

So far we proved an upper bound on the regret incurred during exploration which depends on the (random) number of phases. We now upper bound this random variable as a function of $z_k$ and $p_k$. In particular, we achieve this by focusing on the constraints only. The intuition is that, if the primal-dual algorithm works, then the sequence of policies played cannot violate the constraints \emph{at each phase} too much. At the same time, these policies cannot satisfy the constraints too much, otherwise the exploitation test would trigger and the algorithm would not be exploring at step $n$. Relating these two we obtain a bound on $K_n$.

Recall that, as we assumed before, $n$ is an exploration step in which the good event $G_n$ holds. Using \eqref{eq:pdlin-sum-constr} and the equations thereafter, we have

\begin{align}\label{eq:pdlin-g-upper}
    \notag\inf_{\theta' \in {\Theta}_{alt}} &\sum_{t\leq n: E_t} \sum_{x\in\X}{\rho}(x)\sum_{a\in\A}\omega_t(x,a) d_{x,a}(\theta^\star,\theta')\\ & \leq \frac{\beta_{n-1}}{2\sigma^2} + \frac{2LB}{\sigma^2}\sqrt{\gamma_n}\Psi_n +  \frac{2B^2L^2}{\sigma^2}\left(\sqrt{dS_n\log \left(dS_n\right)} + 1\right).
\end{align}
where the last two terms are $\mathcal{O}(\sqrt{S_n\log S_n})$.

We now provide a lower-bound on the same quantity. Fix a phase index $k\geq \underline{k}$. From \eqref{eq:pdlin-phased-f-k}, we have
\begin{align}
    \notag \sum_{t\in\mathcal{T}_k^E}\left( f_t(\omega_t) + \lambda g_t(\omega_t, z_k)\right)  \geq p_k f(\omega^\star_{z_k}) & - a_\lambda \sqrt{p_k} -(2LB + \lambda_{\max})M_{n,k}\\ &- LB\left(1 + \frac{2LB\lambda_{\mathrm{max}}}{\sigma^2}\right)\zeta_{n,k},
\end{align}
The left-hand side can be upper-bounded by using the optimism property to obtain the true objective and constraint. Regarding the objective function, we have
\begin{align*}
   & \sum_{t\in\mathcal{T}_k^E} f_t(\omega_t)  =  \sum_{t\in\mathcal{T}_k^E : G_t} f_t(\omega_t) + \sum_{t\in\mathcal{T}_k^E : \neg G_t} f_t(\omega_t)\\ &\leq \sum_{t\in\mathcal{T}_k^E : G_t}\sum_{x\in\X}\hat{\rho}_{t-1}(x)\sum_{a\in\A}\omega_t(x,a) {\mu}_{\wt{\theta}_{t-1}}(x,a) + \sum_{t\in\mathcal{T}_k^E : \neg G_t}\sum_{x\in\X}\hat{\rho}_{t-1}(x)\sum_{a\in\A}\omega_t(x,a) {\mu}_{\wt{\theta}_{t-1}}(x,a) + \sqrt{\gamma_n}\Psi_{n,k}\\ &\leq \underbrace{\sum_{t\in\mathcal{T}_k^E : G_t}\sum_{x\in\X}\hat{\rho}_{t-1}(x)\sum_{a\in\A}\omega_t(x,a) {\mu}_{\wt{\theta}_{t-1}}(x,a)}_{(a)} + BLM_{n,k} + \sqrt{\gamma_n}\Psi_{n,k}.
\end{align*}
Regarding the sum over the good events, using Lem.~\ref{lemma:pdlin-optimism},
\begin{align}
     (a) &\leq  \sum_{t\in\mathcal{T}_k^E: G_t}  \sum_{x\in\X}\hat{\rho}_{t-1}(x)\sum_{a\in\A}\omega_t(x,a) \mu_{{\theta^\star}}(x,a) + \sqrt{\gamma_n}\Psi_{n,k}\\ &\leq \sum_{t\in\mathcal{T}_k^E} \underbrace{\sum_{x\in\X}{\rho}(x)\sum_{a\in\A}\omega_t(x,a) \mu_{{\theta^\star}}(x,a)}_{= f(\omega_t)} + BL\underbrace{\sum_{t\in\mathcal{T}_k^E: G_t}  \sum_{x\in\X}|\hat{\rho}_t(x)-\rho(x)|}_{\zeta_{n,k}} + \sqrt{\gamma_n}\Psi_{n,k}.
\end{align}
Therefore,
\begin{align*}
    \sum_{t\in\mathcal{T}_k^E} f_t(\omega_t) \leq \sum_{t\in\mathcal{T}_k^E} f(\omega_t) + BL\zeta_{n,k} + 2\sqrt{\gamma_n}\Psi_{n,k} + BLM_{n,k}.
\end{align*}
We can follow the same reasoning to upper bound the sum of constraints. Since the KLs are upper-bounded by $2B^2L^2/\sigma^2$,
\begin{align*}
    \sum_{t\in\mathcal{T}_k^E} g_t(\omega_t, z_k) \leq \sum_{t\in\mathcal{T}_k^E} g(\omega_t, z_k) + \frac{2B^2L^2}{\sigma^2}\zeta_{n,k} + \frac{4BL}{\sigma^2}\sqrt{\gamma_n}\Psi_{n,k} + \frac{2B^2L^2}{\sigma^2}M_{n,k}.
\end{align*}
Combining the bounds on $f$ and $g$,
\begin{align*}
 \sum_{t\in\mathcal{T}_k^E}&\left( f(\omega_t) + \lambda g(\omega_t, z_k)\right) \geq p_k f(\omega^\star_{z_k}) -\left(3BL + \lambda_{\mathrm{max}} + \lambda \frac{2B^2L^2}{\sigma^2} \right)M_{n,k} - a_{\lambda} \sqrt{p_k}\\ &- 2BL\left(1 + \frac{(\lambda_{\mathrm{max}} + \lambda)BL}{\sigma^2} \right)\zeta_{n,k} - \left(2 + \frac{4BL\lambda}{\sigma^2}\right)\sqrt{\gamma_n}\Psi_{n,k}.
\end{align*}
Let $\bar{\omega}_{t,k} := \frac{1}{p_k}\sum_{t\in\mathcal{T}_{k}^E} \omega_t$ be the average policy played in phase $k$. Since $f$ is linear and $g$ is concave, $\sum_{t\in\mathcal{T}_k^E}\left( f(\omega_t) + \lambda g(\omega_t, z_k)\right) \leq p_k f(\bar{\omega}_{t,k}) + \lambda p_k g(\bar{\omega}_{t,k}, z_k)$.
We now set 
\[
        \lambda = \begin{cases} 2 \lambda_{\max} & \text{if } [g(\bar{\omega}_{t,k}, z_k)]_{-} \neq 0\\
        0 & \text{otherwise} \end{cases}
\]
where $[x]_- = \min\{x, 0\}$. Therefore,
\begin{align*}
 p_k \left(f(\bar{\omega}_{t,k}) - f(\omega^\star_{z_k})+ 2\lambda_{\max} [g(\bar{\omega}_{t,k}, z_k)]_-\right) \geq &-\left(3BL + \lambda_{\mathrm{max}} + \lambda_{\mathrm{max}} \frac{4B^2L^2}{\sigma^2} \right)M_{n,k} - a_{\lambda_{\mathrm{max}}} \sqrt{p_k}\\ &- 2BL\left(1 + \frac{3\lambda_{\mathrm{max}}BL}{\sigma^2} \right)\zeta_{n,k} - \left(2 + \frac{8BL\lambda_{\mathrm{max}}}{\sigma^2}\right)\sqrt{\gamma_n}\Psi_{n,k}.
\end{align*}
Lemma \ref{lemma:mult-bound} together with Asm. \ref{asm:lambda_max} ensures that, for $k\geq\underline{k}$, $\lambda^\star(z_k,\theta^\star) \leq \lambda_{\mathrm{max}}$. Thus, we can apply Theorem 42 of \cite{efroni2020exploration} and obtain
\begin{align*}
    p_k g(\bar{\omega}_{t,k}, z_k) \geq p_k[g(\bar{\omega}_{t,k}, z_k)]_- \geq &-\left(3BL + \lambda_{\mathrm{max}} + \lambda_{\mathrm{max}} \frac{4B^2L^2}{\sigma^2} \right)\frac{M_{n,k}}{2\lambda_{\mathrm{max}}} - \frac{a_{\lambda_{\mathrm{max}}} \sqrt{p_k}}{2\lambda_{\mathrm{max}}}\\ &- 2BL\left(1 + \frac{3\lambda_{\mathrm{max}}BL}{\sigma^2} \right)\frac{\zeta_{n,k}}{2\lambda_{\mathrm{max}}} - \left(2 + \frac{8BL\lambda_{\mathrm{max}}}{\sigma^2}\right)\frac{\sqrt{\gamma_n}\Psi_{n,k}}{2\lambda_{\mathrm{max}}}.
\end{align*}
Summing both sides over all phases,
\begin{align*}
    \sum_{k\geq\underline{k}}^{K_n} p_k g(\bar{\omega}_{t,k}, z_k) &= \sum_{k\geq\underline{k}}^{K_n} p_k \left( \inf_{\theta' \in {\Theta}_{alt}}  \sum_{x\in\X}{\rho}(x)\sum_{a\in\A}\bar{\omega}_{t,k}(x,a) d_{x,a}(\theta^\star,\theta') - \frac{1}{z_k} \right)\\ &= 
    \sum_{k\geq\underline{k}}^{K_n}\left( \inf_{\theta' \in {\Theta}_{alt}}  \sum_{t\in\mathcal{T}_k^E}\sum_{x\in\X}{\rho}(x)\sum_{a\in\A}{\omega}_{t}(x,a) d_{x,a}(\theta^\star,\theta') - \frac{p_k}{z_k} \right)\\ &\leq \inf_{\theta' \in {\Theta}_{alt}} \sum_{t\leq n : E_t} \sum_{x\in\X}{\rho}(x)\sum_{a\in\A}{\omega}_{t}(x,a) d_{x,a}(\theta^\star,\theta') - \sum_{k\geq\underline{k}}^{K_n}\frac{p_k}{z_k}.
\end{align*}
Therefore,
\begin{align*}
    \inf_{\theta' \in {\Theta}_{alt}} &\sum_{t\leq n : E_t} \sum_{x\in\X}{\rho}(x)\sum_{a\in\A}{\omega}_{t}(x,a) d_{x,a}(\theta^\star,\theta') \geq \sum_{k\geq\underline{k}}^{K_n}\frac{p_k}{z_k} -\left(3BL + \lambda_{\mathrm{max}} + \lambda_{\mathrm{max}} \frac{4B^2L^2}{\sigma^2} \right)\frac{M_{n}}{2\lambda_{\mathrm{max}}}\\ & - \frac{a_{\lambda_{\mathrm{max}}} \sqrt{p_k}}{2\lambda_{\mathrm{max}}}- 2BL\left(1 + \frac{3\lambda_{\mathrm{max}}BL}{\sigma^2} \right)\frac{\zeta_n}{2\lambda_{\mathrm{max}}} - \left(2 + \frac{8BL\lambda_{\mathrm{max}}}{\sigma^2}\right)\frac{\sqrt{\gamma_n}\Psi_{n}}{2\lambda_{\mathrm{max}}}.
\end{align*}
Combining this with \eqref{eq:pdlin-g-upper}, we obtain the following inequality:
\begin{align*}
    \sum_{k\geq\underline{k}}^{K_n}\frac{p_k}{z_k} &\leq \frac{\beta_{n-1}}{2\sigma^2}  + \frac{a_{\lambda_{\mathrm{max}}} \sqrt{p_k}}{2\lambda_{\mathrm{max}}} + \frac{2LB}{\sigma^2}\sqrt{\gamma_n}\Psi_n +  \frac{2B^2L^2}{\sigma^2}\sqrt{dS_n\log \left(dS_n\right)}\\& +\left(3BL + \lambda_{\mathrm{max}} + \lambda_{\mathrm{max}} \frac{4B^2L^2}{\sigma^2} \right)\frac{M_{n}}{2\lambda_{\mathrm{max}}}\\ & + 2BL\left(1 + \frac{3\lambda_{\mathrm{max}}BL}{\sigma^2} \right)\frac{\zeta_n}{2\lambda_{\mathrm{max}}} + \left(2 + \frac{8BL\lambda_{\mathrm{max}}}{\sigma^2}\right)\frac{\sqrt{\gamma_n}\Psi_{n}}{2\lambda_{\mathrm{max}}}.
\end{align*}
Recall that, by definition, $S_n = \sum_{k=0}^{K_n} p_k$. Furthermore, by Cauchy-Schwartz inequality, $\sum_{k=0}^{K_n} \sqrt{p_k} \leq \sqrt{K_n\sum_{k=0}^{K_n}p_k}$. Simplifying this a little,
\begin{align}
    \sum_{k\geq\underline{k}}^{K_n}\frac{p_k}{z_k} &\leq \frac{\beta_{n-1}}{2\sigma^2} + \mathcal{O}\left({\sqrt{K_n\sum_{k=0}^{K_n}{p_k}}}\right) + \mathcal{O}\left({\sqrt{\left(\sum_{k=0}^{K_n}p_k\right) \log\left(\sum_{k=0}^{K_n}p_k\right)}}\right).\label{eq:pdlin-phased-kn}
\end{align}

\subsubsection{Choosing $z_k$ and $p_k$}\label{app:regret.exploration.final}

We choose the exponential schedule $z_k = z_0 e^k$ and $p_k = z_ke^{rk}$, where $r$ will be specified later. The left-hand side of \eqref{eq:pdlin-phased-kn} is
\begin{align*}
  \sum_{k\geq\underline{k}}^{K_n}\frac{p_k}{z_k} =  \sum_{k\geq\underline{k}}^{K_n}e^{rk} \geq e^{rK_n},
\end{align*}
while the right-hand side is
\begin{align*}
    \frac{\beta_{n-1}}{2\sigma^2} &+ \mathcal{O}\left({\sqrt{K_n\sum_{k=0}^{K_n}{e^{(r+1)k}}}}\right) + \mathcal{O}\left({\sqrt{\left(\sum_{k=0}^{K_n}e^{(r+1)k}\right) \log\left(\sum_{k=0}^{K_n}e^{(r+1)k}\right)}}\right)\\ &\leq \frac{\beta_{n-1}}{2\sigma^2} + \mathcal{O}\left(\sqrt{K_n^{2}e^{(r+1)K_n}}\right).
\end{align*}
For $r>1$, the resulting inequality yields $K_n \leq \mathcal{O}(\frac{1}{r}\log \beta_{n-1})$, i.e., $K_n \leq \mathcal{O}(\frac{1}{r}\log\log n)$ by definition of $\beta_{n-1}$. Let us recall \eqref{eq:pdlin-phased-reg-partial}:
\begin{align}
    &\notag\expec{ R_n^{\mathrm{explore}}} \leq \underbrace{2BL\sum_{k<\underline{k}} p_k}_{\text{\RomanNumeralCaps{1}}}  +  \underbrace{BL\underline{z}(\theta^\star)\sum_{k\geq\underline{k} : z_k < \bar{z}} \frac{p_k}{z_k - \underline{z}(\theta^\star)}}_{\text{\RomanNumeralCaps{2}}} + \underbrace{v^\star(\theta^\star) \frac{\beta_{n-1}}{2\sigma^2}}_{\text{\RomanNumeralCaps{3}}} + \underbrace{a_\lambda \expec{\sum_{k\geq\underline{k}}^{K_n}\sqrt{p_k}}}_{\text{\RomanNumeralCaps{4}}} \\ &+ \underbrace{BL\underline{z}(\theta) \sumpull(\theta^\star) \expec{\sum_{k : z_k \geq \bar{z}(\theta^\star)}^{K_n} \frac{p_k}{z_k - \underline{z}(\theta^\star)}\max\left\{
                \frac{c_\Theta \sqrt{2}}{\sigma\sqrt{z_k}}, \frac{1}{z_k}
            \right\}}}_{\text{\RomanNumeralCaps{5}}} + \underbrace{\expec{\mathcal{O}(\sqrt{S_n\log S_n})}}_{\text{\RomanNumeralCaps{6}}}.
\end{align}
We bound the remaining terms separately.
\paragraph{Term \RomanNumeralCaps{1}}
\begin{align*}
\sum_{k<\underline{k}} p_k  = z_0\sum_{k<\underline{k}} e^{(r+1)k} \leq z_0 e^{(r+1)\log(\frac{2\underline{z}(\theta^\star)}{z_0})}\log(2\underline{z}(\theta^\star)/z_0) = z_0 (2\underline{z}(\theta^\star)/z_0)^{r+1}\log (2\underline{z}(\theta^\star)/z_0),
\end{align*}
where we used that, from the definition of $\underline{k}$ and $z_k$, it must be that $k < \log(2\underline{z}(\theta^\star)/z_0)$. Thus,
\begin{align*}
\text{\RomanNumeralCaps{1}} \leq 2BL z_0 (2\underline{z}(\theta^\star)/z_0)^{r+1}\log (2\underline{z}(\theta^\star)/z_0).
\end{align*}
\paragraph{Term \RomanNumeralCaps{2}}
\begin{align*}
\sum_{k\geq\underline{k} : z_k < \bar{z}(\theta^\star)} \frac{p_k}{z_k - \underline{z}(\theta^\star)} &= \sum_{\log\left(\frac{2\underline{z}(\theta^\star)}{z_0}\right) \leq k < \log\left(\frac{\bar{z}(\theta^\star)}{z_0}\right) } \frac{z_0e^{(r+1)k}}{z_0e^k - \underline{z}(\theta^\star)}\\ &= \sum_{\log\left(\frac{2\underline{z}(\theta^\star)}{z_0}\right) \leq k < \log\left(\frac{\bar{z}(\theta^\star)}{z_0}\right) } \underbrace{\frac{z_0e^{k}}{z_0e^k - \underline{z}(\theta^\star)} }_{\leq 2}e^{rk}\leq 2(\bar{z}(\theta^\star)/z_0)^r\log(\bar{z}(\theta^\star)/z_0).
\end{align*}
Thus,
\begin{align*}
\text{\RomanNumeralCaps{2}} \leq 2BL\underline{z}(\theta^\star)(\bar{z}(\theta^\star)/z_0)^r\log(\bar{z}(\theta^\star)/z_0).
\end{align*}

\paragraph{Term \RomanNumeralCaps{4}}
The total number of exploration rounds is
\begin{align*}
    S_n = \sum_{k=0}^{K_n} p_k = z_0 \sum_{k=0}^{K_n} e^{(r+1)k} \leq z_0 e^{(r+1)(K_n + 1)} \leq \mathcal{O}((\log n)^{\frac{r+1}{r}}).
\end{align*}
Therefore,
\begin{align*}
    \text{\RomanNumeralCaps{4}} \leq \sqrt{K_n\sum_{k=0}^{K_n}{p_k}} \leq \mathcal{O}((\log\log n)^{1/2}(\log n)^{\frac{r+1}{2r}}).
\end{align*}

\paragraph{Term \RomanNumeralCaps{5}}
We consider two cases, based on which of the inner terms is the maximum. In the first case, we need to bound
\begin{align*}
\sum_{k : z_k \geq \bar{z}(\theta^\star)}^{K_n} \frac{p_k}{(z_k - \underline{z}(\theta^\star))\sqrt{z_k}} &= 
\sum_{k \geq \log\left(\frac{\bar{z}(\theta^\star)}{z_0}\right)}^{K_n} \frac{z_0e^{(r+1)k}}{(z_0e^k - \underline{z}(\theta^\star))\sqrt{z_0e^k}}\\ &= \frac{1}{\sqrt{z_0}}\sum_{k \geq \log\left(\frac{\bar{z}(\theta^\star)}{z_0}\right)}^{K_n} \underbrace{\frac{z_0e^{k}}{(z_0e^k - \underline{z}(\theta^\star))}}_{\leq 2}e^{(r-1/2)k} \leq  \frac{2}{\sqrt{z_0}}\sum_{k \geq \log\left(\frac{\bar{z}(\theta^\star)}{z_0}\right)}^{K_n} e^{(r-1/2)k}\\ &\leq \frac{2}{\sqrt{z_0}}\int_{\log\left(\frac{\bar{z}(\theta^\star)}{z_0}\right)}^{K_n+1} e^{(r-1/2)k}\mathrm{d}k = \frac{2}{\sqrt{z_0}}\left[\frac{e^{(r-1/2)k}}{r-1/2}\right]_{\log\left(\frac{\bar{z}(\theta^\star)}{z_0}\right)}^{K_n+1}\\ &= \frac{2}{(r-1/2)\sqrt{z_0}}\left(e^{(r-1/2)(K_n+1)} - (\bar{z}(\theta^\star)/z_0)^{r-1/2}\right).
\end{align*}
Since $K_n \leq \mathcal{O}(\frac{1}{r}\log\log n)$, this term is $\mathcal{O}((\log n)^{\frac{r-1/2}{r}})$. If the other term is the maximum, then the same procedure yields a $\mathcal{O}((\log n)^{\frac{r-1}{r}})$ dependency. Thus,
\begin{align*}
 \text{\RomanNumeralCaps{5}} \leq \mathcal{O}((\log n)^{\frac{r-1/2}{r}}).
\end{align*}

\paragraph{Term \RomanNumeralCaps{6}} We have $\text{ \RomanNumeralCaps{6}} \leq \mathcal{O}((\log n)^{\frac{r+1}{2r}})$ as in Term \RomanNumeralCaps{4}.

\paragraph{Final Bound}

Using $r=2$, we obtain the following bound on the expected regret during exploration:
\begin{align*}
    \notag\expec{ R_n^{\mathrm{explore}}} &\leq 2BLz_0 (2\underline{z}(\theta^\star)/z_0)^{3}\log (2\underline{z}(\theta^\star)/z_0)\\ &+ 2BL\underline{z}(\theta^\star)(\bar{z}(\theta^\star)/z_0)^2\log(\bar{z}(\theta^\star)/z_0) + v^\star(\theta^\star) \frac{\beta_{n-1}}{2\sigma^2} + \mathcal{O}((\log\log n)^{\frac{1}{2}}(\log n)^{\frac{3}{4}}),
\end{align*}
which is asymptotically optimal.

\section{Worst-case Analysis (Proof of Thm.~\ref{th:worst-case})}\label{app:worst-case}

\subsection{Outline}

The proof follows a similar argument as the one of Thm.~\ref{thm:regret.bound} but it is considerably simpler and shorter. In particular, the main simplifications come from two worst-case arguments. (1) While bounding the regret during exploration rounds, we use the naive bound $S_n \leq n$. This is equivalent to assuming that \algo never enters the exploitation step and it allows us to entirely avoid the bound on the number of phases of App.~\ref{app:regret.exploration.phases}. (2) We completely ignore the sequence $z_k$ and proceed as if the optimization problem \eqref{eq:optim-lb-z} was infeasible in all phases. This makes the multiplier saturate to $\lambda_{\max}$ and facilitate the analysis of the resulting Lagrangian\footnote{Recall that the regret of \algo is not defined in terms of the optimization problem \eqref{eq:optim-lb-z} or its Lagrangian, but only in terms of the rewards of the chosen arms compared to those of the optimal arms. This makes it possible to obtain good regret guarantees even when solving an infeasible optimization problem.}. An outline of the proof, together with the main differences w.r.t. the one of Thm.~\ref{thm:regret.bound}, is as follows.

\begin{enumerate}
\item We decompose the regret suffered during exploitation and exploration rounds. Using the same steps as in App.~\ref{app:regret}, we bound the former by a constant and reduce the latter to the sum of objective values.
\item Instead of relating to the objective values of the optimal policies $\omega^\star_{z_k}$ at each phase $k$ (as was done in App.~\ref{app:regret.exploration.obj}, we reduce our bound to the optimal solution of our bandit problem, i.e., the policy that only pulls optimal arms. This makes the sum of objective values cancel since the optimal policy achieves zero regret.
\item Using the results of App.~\ref{app:regret.exploration.constr}, we show that the sum of constraints is $\mathcal{O}(\log n)$.
\item We use the naive bound $S_n \leq n$ to conclude the proof.
\end{enumerate}

\subsection{Proof}

We start from the same regret decomposition as in App.~\ref{app:regret},
\begin{align*}
    R_n = \sum_{t=1}^n \Delta_{\theta^\star}(X_t, A_t)\indi{\neg E_t} + \sum_{t=1}^n \Delta_{\theta^\star}(X_t, A_t)\indi{E_t} = R_n^{\mathrm{exploit}} +  R_n^{\mathrm{explore}}.
\end{align*}
The regret suffered during the exploitation rounds was bounded in App.~\ref{app:regret.exploitation} as $\expec{R_n^{\mathrm{exploit}}} \leq 2LB$. Regarding the regret suffered during the exploration rounds, we have
\begin{align}\label{eq:wc1}
    R_n^{\mathrm{explore}} &:= \sum_{t=1}^n \Delta_{\theta^\star}(X_t, A_t)\indi{E_t} \leq \sum_{t=1}^n \Delta_{\theta^\star}(X_t, A_t)\indi{E_t, G_t} +
2LB\underbrace{
    \sum_{t=1}^n \indi{E_t, \neg G_t}
}_{:= M_n}.
\end{align}
Refer to App.~\ref{app:good.events} for the definition of $G_t$.
The second term is $M_n$, the number of exploration rounds in which the good event does not hold, and can be bounded in expectation by using Lem.~\ref{lemma:pdlin-bad-event}. The first one can be bounded by using the good event. Suppose, without loss of generality, that $E_n$ and $G_n$ hold (if they do not, the following reasoning can be repeated for the last time step at which these events hold). Then, using $G_t^\Delta$ (see App. \ref{app:good.events}),
\begin{align}\label{eq:wc2}
        \notag\sum_{t=1}^n \Delta_{\theta^\star}(X_t, A_t)\indi{E_t, G_t} &\leq \sum_{t\leq n : E_t} \Delta_{\theta^\star}(X_t, A_t)\\ &\leq \sum_{t\leq n : E_t}\sum_{x\in\X}\rho(x)\sum_{a\in\A}{\omega}_t(x,a)\Delta_{\theta^\star}(x,a) + 2LB\sqrt{S_n\log S_n}.
\end{align}
We now proceed using similar steps as in App.~\ref{app:regret.exploration.first}, except that we ignore the phases. We decompose the first term as
\begin{align*}
&\sum_{t\leq n : E_t}\sum_{x\in\X}\rho(x)\sum_{a\in\A}\omega_t(x,a) \Delta_{\theta^\star}(x,a)\\
        &\quad = \sum_{t\leq n : E_t,G_t}\sum_{x\in\X}\rho(x)\sum_{a\in\A}\omega_t(x,a) \Delta_{\theta^\star}(x,a) + \sum_{t\leq n : E_t, \neg G_t}\sum_{x\in\X}\rho(x)\sum_{a\in\A}\omega_t(x,a) (\mu^\star_{\theta^\star}(x) - \mu_{\theta^\star}(x,a))\\
        &\quad \leq \sum_{t\leq n : E_t,G_t}\sum_{x\in\X}\rho(x)\sum_{a\in\A}\omega_t(x,a) \Delta_{\theta^\star}(x,a) + M_{n} \mu^\star - \sum_{t\leq n : E_t,\neg G_t}\sum_{x\in\X}\rho(x)\sum_{a\in\A}\omega_t(x,a) \mu_{\theta^\star}(x,a).
\end{align*}
Here we defined
\begin{equation}\label{eq:worstcase.mustar}
    \mu^\star :=\sum_{x\in\X}\rho(x)\mu^\star_{\theta^\star}(x).
\end{equation}
The last term can be bounded by $M_{n}BL$. Regarding the remaining two,
\begin{align*}
    &\sum_{t\leq n : E_t,G_t}\sum_{x\in\X}\rho(x)\sum_{a\in\A}\omega_t(x,a) \Delta_{\theta^\star}(x,a) + M_{n}\mu^\star\\  &= (S_n - M_{n}) \mu^\star + M_{n} \mu^\star - \sum_{t\leq n : E_t, G_t}\sum_{x\in\X}\rho(x)\sum_{a\in\A}\omega_t(x,a) \mu_{\theta^\star}(x,a)\\ &= S_n \mu^\star + \underbrace{\sum_{t\leq n : E_t,G_t}\sum_{x\in\X}(\wh{\rho}_{t-1}(x) - \rho(x))\sum_{a\in\A}\omega_t(x,a) \mu_{\theta^\star}(x,a)}_{(a)} - \underbrace{\sum_{t\leq n : E_t, G_t}\sum_{x\in\X}\wh{\rho}_{t-1}(x)\sum_{a\in\A}\omega_t(x,a) \mu_{\theta^\star}(x,a)}_{(b)}.
\end{align*}
Term (a) can be bounded as
\begin{align*}
(a) \leq LB \underbrace{\sum_{t\leq n : E_t, G_t}\sum_{x\in\X}|\wh{\rho}_{t-1}(x) - \rho(x)|}_{\zeta_n}.
\end{align*}
For the sake of readability, we keep the dependence on $\zeta_n$ explicit. We will bound this term by Lem.~\ref{lemma:pdlin-rho-dev} at the end of the proof. Regarding term (b), using the definition of $G_t$ and Lem. \ref{lemma:good-event-mu},
\begin{align*}
    (b) &\geq \notag \sum_{t\leq n : E_t, G_t}\sum_{x\in\X}\wh{\rho}_{t-1}(x)\sum_{a\in\A}\omega_t(x,a) \left( {\mu}_{\wt{\theta}_{t-1}}(x,a) - \sqrt{\gamma_t}\| \phi(x,a)\|_{\bar{V}_{t-1}^{-1}} \right) \\ \notag
        &\pm \sum_{t\leq n : E_t,\neg G_t}\sum_{x\in\X}\wh{\rho}_{t-1}(x)\sum_{a\in\A}\omega_t(x,a) \underbrace{{\mu}_{\wt{\theta}_{t-1}}(x,a)}_{|\cdot|\leq LB} \pm \sum_{t\leq n : E_t}\sum_{x\in\X}\wh{\rho}_{t-1}(x)\sum_{a\in\A}\omega_t(x,a) \sqrt{\gamma_t}\| \phi(x,a)\|_{\bar{V}_{t-1}^{-1}}\\ &\geq \notag \sum_{t\leq n : E_t} f_t(\omega_t) - M_{n}BL - 2\sum_{t\leq n : E_t}\sum_{x\in\X}\wh{\rho}_{t-1}(x)\sum_{a\in\A}\omega_t(x,a) \sqrt{\gamma_t}\| \phi(x,a)\|_{\bar{V}_{t-1}^{-1}}\\ &\geq \sum_{t\leq n : E_t} f_t(\omega_t) - M_{n}BL - 2\sqrt{\gamma_n}\Psi_{n}.
\end{align*}
We recall that $\sqrt{\gamma_t} \leq \sqrt{\gamma_n}$ and $\Psi_{n} := \sum_{t\leq n : E_t}\sum_{x\in\X}\wh{\rho}_{t-1}(x)\sum_{a\in\A}\omega_t(x,a) \| \phi(x,a)\|_{\bar{V}_{t-1}^{-1}}$. As for $\zeta_n$, we keep the dependence on $\Psi_n$ explicit and defer bounding this term to the end of the proof. Using the bounds on (a) and (b) and plugging everything back into \eqref{eq:wc2} and then into \eqref{eq:wc1}, we obtain
\begin{align}\label{eq:wc-reg-partial}
R_n^{\mathrm{explore}} &\leq S_n\mu^\star - \sum_{t\leq n : E_t} f_t(\omega_t) + 4M_{n}BL + \zeta_n BL + 2\sqrt{\gamma_n}\Psi_{n} + 2BL\sqrt{S_n\log S_n}.
\end{align}

We now lower bound the sum of objective values. Here we proceed in a slightly different way with respect to the proof of the asymptotically optimal regret bound. Instead of relating to the objective values of the optimal policies $\omega^\star_{z_k}$ at each phase $k$, we reduce our bound to the optimal solution of our bandit problem, i.e., the policy that only pulls optimal arms. Let
\begin{align}
    \label{eq:omega.optimal.bandit}
\omega^\star_{\theta^\star}(x,a) := \begin{cases} 1 & \text{if } a = a^\star_{\theta^\star}(x)\\
        0 & \text{otherwise}
    \end{cases}
\end{align}
Recall that $\sum_{t\leq n : E_t} f_t(\omega_t) = \sum_{k=0}^{K_n}\sum_{t\in\mathcal{T}_k^E} f_t(\omega_t)$. Fix some phase index $k\geq 0$ and let $\lambda \geq 0$ be arbitrary. Using Corollary \ref{cor:rec-pd} with $\alpha_k^\lambda = \alpha_k^\omega = 1/\sqrt{p_k}$ and $\omega = \omega^\star_{\theta^\star}$,
\begin{align}\label{eq:wc3}
    \sum_{t\in\mathcal{T}_k^E} f_t(\omega_t)  \geq \sum_{t\in\mathcal{T}_k^E} h_t(\omega^\star_{\theta^\star}, \lambda_t, z_k) - \lambda \sum_{t\in\mathcal{T}_k^E}g_t(\omega_t, z_k) - a_\lambda\sqrt{p_k},
\end{align}
where 
\begin{equation}
    \label{eq:alambda}
    a_\lambda:= \left(\log |\A| + \frac{b_\omega^2 + b_\lambda^2}{2} + \frac{(\lambda - \lambda_1)^2}{2} \right)
\end{equation}
and $b_\lambda$ and $b_\omega$ are the maximum sub-gradients in $\lambda$ and $\omega$, respectively. Note that, since we apply Corollary \ref{cor:rec-pd} to bound the sum of objective values over the whole phase, we have $S_{n,k} = p_k$. We now lower-bound the first term on the right-hand side. 
We have
\begin{align}\label{eq:wc4}
\notag\sum_{t\in\mathcal{T}_k^E}  h_t(\omega^\star_{\theta^\star}, \lambda_t, z_k) 
&\stackrel{(c)}{=} \sum_{t\in\mathcal{T}_k^E}  f_t(\omega^\star_{\theta^\star}) + \sum_{t\in\mathcal{T}_k^E}  \lambda_t g_t(\omega^\star_{\theta^\star}, z_k) \stackrel{(d)}{\geq} \sum_{t\in\mathcal{T}_k^E}  f_t(\omega^\star_{\theta^\star}) - \sum_{t\in\mathcal{T}_k^E} \frac{\lambda_t}{z_k}\\
&\stackrel{(e)}{\geq} \sum_{t\in\mathcal{T}_k^E}  f_t(\omega^\star_{\theta^\star}) - \frac{\lambda_{\max}S_{n,k}}{z_k},
\end{align}
where (c) uses the definition of $h_t$ and $g_t$ (see Eq.~\ref{eq:pdlin.gt} and Eq.~\ref{eq:pdlin.ht}), (d) uses the positivity of KL divergences and confidence intervals, and (e) uses $\lambda_t \leq \lambda_{\max}$ and $S_{n,k} := |\mathcal{T}_k^E|$. Let us focus on the sum of objective values. Since $f_t(\omega^\star_{\theta^\star}) \geq -LB$, we have $\sum_{t\in\mathcal{T}_k^E : \neg G_t} f_t(\omega^\star_{\theta^\star}) \geq  -M_{n,k}BL$. For any step $t\in\mathcal{T}_k^E$ in which $G_t$ holds, the optimism property (see App.~\ref{app:good.events} and Lem.~\ref{lemma:pdlin-optimism}) yields
\begin{align*}
    \sum_{t\in\mathcal{T}_k^E : G_t} f_t (\omega^\star_{\theta^\star}) &\geq \sum_{t\in\mathcal{T}_k^E : G_t} \sum_{x\in\X} \wh{\rho}_{t-1}(x) \sum_{a\in\A} \omega^\star_{\theta^\star}(x,a) \mu_{\theta^\star}(x,a)\\  &=  \sum_{t\in\mathcal{T}_k^E : G_t} \sum_{x\in\X} (\wh{\rho}_{t-1}(x) - \rho(x)) \underbrace{\sum_{a\in\A} \omega^\star_{\theta^\star}(x,a) \mu_{\theta^\star}(x,a)}_{|\cdot| \leq BL} + \sum_{t\in\mathcal{T}_k^E : G_t}  \underbrace{f(\omega^\star_{\theta^\star})}_{= \mu^\star} \\ &\geq (S_{n,k} - M_{n,k})\mu^\star - BL\underbrace{\sum_{t\in\mathcal{T}_k^E : G_t} \sum_{x\in\X}|\wh{\rho}_{t-1}(x) - \rho(x)|}_{:= \zeta_{n,k}},
\end{align*}
where we used the fact that $f(\omega^\star_{\theta^\star}) = \mu^\star$ by definition~\eqref{eq:omega.optimal.bandit} and~\eqref{eq:worstcase.mustar} and $\sum_{t\in\mathcal{T}_k^E} \indi{G_t} = \sum_{t\in \mathcal{T}_k} \indi{E_t} - \sum_{t\in \mathcal{T}_k} \indi{E_{t}, \neg G_{t}} = S_{n,k} - M_{n,k}$.
Plugging this back into \eqref{eq:wc4} and then into \eqref{eq:wc3},
\begin{align*}
\sum_{t\in\mathcal{T}_k^E} f_t(\omega_t)  \geq (S_{n,k} - M_{n,k})\mu^\star &- BL\zeta_{n,k} - \frac{\lambda_{\max}S_{n,k}}{z_k} - \lambda \sum_{t\in\mathcal{T}_k^E}g_t(\omega_t, z_k) - a_\lambda\sqrt{p_k} -M_{n,k}BL.
\end{align*}
Summing over all phases and recalling that $\sum_{k=0}^{K_n}S_{n,k} = S_n$, $\sum_{k=0}^{K_n}M_{n,k} = M_n$, and $\sum_{k=0}^{K_n}\zeta_{n,k} = \zeta_n$, we obtain
\begin{align}\label{eq:wc5}
\sum_{t\leq n: E_t} f_t(\omega_t)  \geq (S_{n} - M_{n})\mu^\star &- BL\zeta_{n} - \sum_{k=0}^{K_n}\frac{\lambda_{\max}S_{n,k}}{z_k} - \lambda \sum_{t\leq n : E_t}g_t(\omega_t, z_{K_t}) - a_\lambda\sum_{k=0}^{K_n}\sqrt{p_k} -M_{n}BL.
\end{align}
Using the definition of $g_t$ (see Eq.~\ref{eq:pdlin.gt}),
\begin{align*}
\sum_{t\leq n : E_t}g_t(\omega_t, z_{K_t})&:= \sum_{t\leq n : E_t}\inf_{\theta' \in \wb{\Theta}_{t-1}}\sum_{x\in\X}\wh{\rho}_{t-1}(x)\sum_{a\in\A}\omega_t(x,a)
        \bigg( d_{x,a} \left(\wt{\theta}_{t-1},\theta' \right) + \frac{2LB}{\sigma^2}\sqrt{\gamma_t}\| \phi(x,a)\|_{\wb{V}_{t-1}^{-1}} \bigg) \\ &- \sum_{t\leq n : E_t}\frac{1}{z_{K_t}}.
\end{align*}
By the definition of phase, the second term is $\sum_{t\leq n : E_t}\frac{1}{z_{K_t}} = \sum_{k=0}^{K_n}\frac{S_{n,k}}{z_k}$. The first term can be bounded using exactly the same steps as in App.~\ref{app:regret.exploration.constr}.\footnote{Note that the bound on the sum of constraints of App.~\ref{app:regret.exploration.constr} uses only the properties of the confidence intervals and of the exploitation test. Thus, it is applicable regardless of the feasibility of the optimization problems at each phase.} We obtain
\begin{align}\label{eq:wc6}
\notag\sum_{t\leq n : E_t}g_t(\omega_t, z_{K_t}) \leq \frac{\beta_{n-1}}{2\sigma^2} &- \sum_{k=0}^{K_n}\frac{S_{n,k}}{z_k} + \frac{2L^2B^2}{\sigma^2}M_n+ \frac{6LB}{\sigma^2}\sqrt{\gamma_n}\Psi_n + \frac{2L^2B^2}{\sigma^2}\zeta_n \\ &+  \frac{2B^2L^2}{\sigma^2}\left(\sqrt{dS_n\log \left(dS_n\right)} + 1\right).
\end{align}
If we now set $\lambda = \lambda_{\max}$ and plug \eqref{eq:wc6} into \eqref{eq:wc5},
\begin{align*}
\sum_{t\leq n: E_t} & f_t(\omega_t)  \geq (S_{n} - M_{n})\mu^\star - BL\left(1 + \frac{2\lambda_{\max}BL}{\sigma^2}\right)(\zeta_{n} + M_n) + \underbrace{\sum_{k=0}^{K_n}\frac{\lambda_{\max}S_{n,k}}{z_k} - \sum_{k=0}^{K_n}\frac{\lambda_{\max}S_{n,k}}{z_k}}_{=0}\\ &-\frac{\lambda_{\max}\beta_{n-1}}{2\sigma^2} - a_{\lambda_{\max}}\sum_{k=0}^{K_n}\sqrt{p_k} - \frac{6\lambda_{\max}LB}{\sigma^2}\sqrt{\gamma_n}\Psi_n - \frac{2\lambda_{\max}B^2L^2}{\sigma^2}\left(\sqrt{dS_n\log \left(dS_n\right)} + 1\right).
\end{align*}
We can finally plug this into \eqref{eq:wc-reg-partial}, thus obtaining
\begin{align*}
R_n^{\mathrm{explore}} &\leq M_{n}\underbrace{\mu^\star}_{|\cdot|\leq BL} + BL\left(5 + \frac{2\lambda_{\max}BL}{\sigma^2}\right)(\zeta_{n} + M_n) +\frac{\lambda_{\max}\beta_{n-1}}{2\sigma^2} + a_{\lambda_{\max}}\sum_{k=0}^{K_n}\sqrt{p_k}\\ &+ \left(2 + \frac{6\lambda_{\max}LB}{\sigma^2}\right)\sqrt{\gamma_n}\Psi_n + \frac{2\lambda_{\max}B^2L^2}{\sigma^2}\left(\sqrt{dS_n\log \left(dS_n\right)} + 1\right) + 2BL\sqrt{S_n\log S_n}.
\end{align*}
Let $\bar{k}_n := \min\{k : p_k \geq n\}$, then $K_n \leq \bar{k}_n$. Using the exponential schedule $p_k = e^{rk}$, $\bar{k}_n = \lceil \frac{1}{r}\log n \rceil$ and
\begin{align*}
\sum_{k=0}^{K_n}\sqrt{p_k} \leq \sum_{k=0}^{\bar{k}_n}e^{\frac{r}{2}k} \leq \int_{0}^{\bar{k}_n + 1} e^{\frac{r}{2}x}\mathrm{d}x = \left[\frac{2}{r} e^{\frac{r}{2}x}\right]_{0}^{\bar{k}_n + 1} = \frac{2}{r}e^{\frac{r}{2}( \lceil \frac{1}{r}\log n \rceil + 1)} - \frac{2}{r} \leq \frac{2e^r}{r}\sqrt{n}.
\end{align*}
Taking expectations of both sides of the regret bound above and using $S_n \leq n$ and $\expec{M_n} \leq \frac{3\pi^2}{2}$ by Lem.~\ref{lemma:pdlin-bad-event},
\begin{align*}
&\expec{R_n^{\mathrm{explore}}} \leq  \frac{3BL\pi^2}{2}\left(6 + \frac{2\lambda_{\max}BL}{\sigma^2}\right) +\frac{\lambda_{\max}\beta_{n-1}}{2\sigma^2} + \frac{2e^r a_{\lambda_{\max}}}{r}\sqrt{n} + 2BL\sqrt{n\log n}\\ &+ \underbrace{\left(2 + \frac{6\lambda_{\max}LB}{\sigma^2}\right)\expec{\sqrt{\gamma_n}\Psi_n}}_{(i)}) + \underbrace{\frac{2\lambda_{\max}B^2L^2}{\sigma^2}\left(\sqrt{nd\log \left(nd\right)} + 1\right)}_{(ii)} + \underbrace{BL\left(5 + \frac{2\lambda_{\max}BL}{\sigma^2}\right)\expec{\zeta_{n}}}_{(iii)}.
\end{align*}
{After bounding $S_n \leq n$, by Lem.~\ref{lemma:pdlin-Lt}, $\Psi_n \leq \wt{\mathcal{O}}\left(L|\X|\sqrt{n\log n} + \sqrt{nd\log n}\right)$,  while, by Lem.~\ref{lemma:pdlin-rho-dev}, $\zeta_n \leq \wt{\mathcal{O}}\left(|\X|\sqrt{n\log n} \right)$\footnote{Here we hide logarithmic terms in $|\X|$ and $d$.}. Moreover, both $\gamma_n$ and $\beta_n$ are $\mathcal{O}(\log n + d\log\log n )$ by definition of the confidence set. Introducing $C_{\lambda_{\max}} := \left(1 + \frac{\lambda_{\max}BL}{\sigma^2}\right)$,
\begin{align*}
(i) &\leq C_{\lambda_{\max}} \wt{\mathcal{O}}\Big((L|\X|+\sqrt{d})\log(n)\sqrt{dn} \Big),\\
(ii) & \leq BL C_{\lambda_{\max}} \wt{\mathcal{O}}\Big(\sqrt{dn\log(n)} \Big),\\
(iii) & \leq BL C_{\lambda_{\max}} \wt{\mathcal{O}}\left(|\X|\sqrt{n\log n} \right).
\end{align*}
Recalling that the regret during exploitation rounds was bounded by $2BL$ and noting that $2 < \frac{3\pi^2}{2}$, the $\mathcal{O}(1)$ regret term (first term of the bound above plus $2BL$) can be bounded by $12BL\pi^2C_{\lambda_{\max}}$. Hence, the final regret bound can be written as
\begin{align*}
\expec{R_n} \leq 12BL\pi^2C_{\lambda_{\max}} + \frac{2e^r\left(\lambda_{\max}^2 + \log|\A|\right)}{r}\sqrt{n} + C_{\lambda_{\max}} C_{\mathrm{sqrt}} \log (n) \sqrt{dn},
\end{align*}
where $C_{\mathrm{sqrt}} = lin_{\geq 0}(|\X| + \sqrt{d}, B, L)$. Here we included $\frac{\lambda_{\max}\beta_{n-1}}{2\sigma^2} $, $2BL\sqrt{n\log n}$, and the components of $a_{\lambda_{\max}}$ (except $\lambda_{\max}^2$ and $\log|\A|$, which are kept explicit, see Eq.~\ref{eq:alambda}) into the last term above. This concludes the proof.}

\section{Auxiliary Results}\label{app:auxiliary.results}
\subsection{Concentration Inequalities} \label{app:concentrations}

\begin{lemma}[Concentration of $\rho$ during exploration]\label{lemma:conc-rho}
For any context $x\in\X$,
\begin{align}
    \sum_{t \geq 1}\sum_{x\in\X} \prob{E_t, |\wh{\rho}_{t}(x) - \rho(x)| > \sqrt{\frac{\log(|\X|S_t^2)}{2S_t}}} \leq \frac{\pi^2}{3}.
\end{align}
\end{lemma}
\begin{proof}
The proof follows Lem. B.1 in~\citep{combes2014unimodal}. Fix some $\wb{t} \geq 1$ and $x\in\X$. Then,
\begin{align*}
    \sum_{t=1}^{\wb{t}} \indi{E_t, |\wh{\rho}_{t}(x) - \rho(x)| > \sqrt{\frac{\log(|\X|S_t^2)}{2S_t}}} \leq \sum_{s\geq1} \indi{ |\wh{\rho}_{\tau_s}(x) - \rho(x)| > \sqrt{\frac{\log(|\X|s^2)}{2s}}, \tau_s \leq \wb{t}}.
\end{align*}
where $\tau_s$ is the random time the $s$-th exploration round occurs.
Thus, by taking the expectation of both sides,
\begin{align*}
    \sum_{t=1}^{\wb{t}} \prob{E_t, |\wh{\rho}_{t}(x) - \rho(x)| > \sqrt{\frac{\log(|\X|S_t^2)}{2S_t}}}\leq \sum_{s\geq1} \prob{ |\wh{\rho}_{\tau_s}(x) - \rho(x)| > \sqrt{\frac{\log(|\X|s^2)}{2s}}, \tau_s \leq \wb{t}}.
\end{align*}
Since $\tau_s$ is a stopping-time upper bounded by $\wb{t}$ and the number of samples used to compute $\wh{\rho}_{\tau_s}(x)$ is at least $s$, we can apply Lemma 4.3 of~\citep{combes2014unimodal}:
\begin{align*}
    \sum_{t=1}^{\wb{t}} \prob{E_t, |\wh{\rho}_{t}(x) - \rho(x)| >\sqrt{\frac{\log(|\X|S_t^2)}{2S_t}}} \leq \sum_{s\geq 1} 2e^{-2s\frac{\log(|\X|s^2)}{2s}} = \frac{2}{|\X|} \sum_{s \geq 1} \frac{1}{s^2} = \frac{\pi^2}{3|\X|}.
\end{align*}
The reasoning above holds for any $\wb{t}$ and $x\in\X$. Summing over $\X$ concludes the proof.
\end{proof}

\begin{lemma}[Confidence set for exploration]\label{lemma:ci-exploration}
With some abuse of notation, let $\gamma_t := c_{n,1/S_t^2}$. Then, under the same conditions as in Theorem \ref{th:conf-theta},
\begin{align*}
    \sum_{t = 1}^n \prob{E_t, \|\wh{\theta}_{t-1} - \theta^\star\|_{\wb{V}_{t-1}} > \sqrt{\gamma_t}} \leq \frac{\pi^2}{6}.
\end{align*}
\end{lemma}
\begin{proof}
Let $\{\tau_s\}_{s\geq 1}$ be a sequence of stopping times with respect to $\mathcal{F}$ such that if $\tau_s = t$, then the $s$-th exploration round occurs at time $t+1$. Then,
\begin{align}
    \sum_{t = 1}^n \indi{E_t, \|\wh{\theta}_{t-1} - \theta^\star\|_{\wb{V}_{t-1}} > \sqrt{\gamma_t}} \leq \sum_{s \geq 1} \indi{\|\wh{\theta}_{\tau_s} - \theta^\star\|_{\wb{V}_{\tau_s}} > \sqrt{\gamma_{\tau_s+1}}, \tau_s \leq n}.
\end{align}
Since $S_{\tau_s + 1} = s$, we have $\gamma_{\tau_s + 1} = c_{n,1/s^2}$.  Taking expectations and applying Theorem \ref{th:conf-theta},
\begin{align*}
     \sum_{s \geq 1} \prob{\|\wh{\theta}_{\tau_s} - \theta^\star\|_{\wb{V}_{\tau_s}} > \sqrt{\gamma_{\tau_s+1}}, \tau_s \leq n} \leq \sum_{s\geq 1} \frac{1}{s^2} = \frac{\pi^2}{6}.
\end{align*}
\end{proof}

\subsection{Supporting Lemmas}

\red{The following result shows that any projection onto a non-empty convex set using a norm weighted by a positive definite matrix is a \emph{non-expansion}. That is, the distance (in the chosen weighted norm) between the projected vector and any point in the set cannot increase w.r.t. the unprojected vector. We are not sure about a suitable citation for this result, so we include its proof.

\begin{lemma}[Non-expansion of weighted projection]\label{lemma:non-exp-proj}
Let $\wh{\theta}\in\R^d$ be any vector, $V \in \R^{d\times d}$ be a positive definite matrix, and $\mathcal{B} \subset \R^d$ be a non-empty convex set. Let $\wt{\theta}$ be the weighted projection of $\wh{\theta}$ onto $\mathcal{B}$,
\begin{align}
\wt{\theta} := \argmin_{\theta \in \mathcal{B}}\| \theta - \wh{\theta} \|_V.
\end{align}
Then, for all $\theta \in \mathcal{B}$,
\begin{align}
\| \wt{\theta} - \theta \|_V \leq \| \wh{\theta} - \theta \|_V.
\end{align}
\end{lemma}
\begin{proof}
Let $f : \R^d \rightarrow \R$ be defined as $f(x) := \| x - \wh{\theta} \|_V^2$, so that $\wt{\theta} = \argmin_{x\in\mathcal{B}}f(x)$. Note that $f$ is a convex function that is differentiable on $\R^d$. Therefore, using the first-order optimality conditions for convex functions (see, e.g., Theorem 2.8 in \cite{orabona2019modern}), we have $\wt{\theta} = \argmin_{x\in\mathcal{B}}f(x)$ if and only if
\begin{align}
\forall \theta \in \mathcal{B} : \langle \nabla f(\wt{\theta}), \theta - \wt{\theta} \rangle \geq 0.
\end{align}
Since $\nabla f(x) = 2V(x - \wh{\theta})$,
\begin{align}
\forall \theta \in \mathcal{B} : \langle V(\wt{\theta} - \wh{\theta}), \theta - \wt{\theta}) \rangle \geq 0.
\end{align}
Fix any $\theta \in \mathcal{B}$. We have
\begin{align*}
\| \wh{\theta} - \theta \|_V^2 = \| \wh{\theta} \pm \wt{\theta} - \theta \|_V^2 = \| \wh{\theta} - \wt{\theta} \|_V^2 + \| \wt{\theta} - \theta \|_V^2 + 2 (\wh{\theta} - \wt{\theta})^T V (\wt{\theta} - \theta) \geq \| \wt{\theta} - \theta \|_V^2.
\end{align*}
This concludes the proof.
\end{proof}
 
\begin{corollary}\label{cor:proj-theta-star}
Let $t\in[n]$ be any time step in which the good event $G_t$ holds. Then,
\begin{align}
\| \wt{\theta}_{t-1} - \theta^\star \|_{\wb{V}_{t-1}} \leq \| \wh{\theta}_{t-1} - \theta^\star \|_{\wb{V}_{t-1}}.
\end{align}
\end{corollary}
\begin{proof}
If $G_t$ holds, then $\theta^\star \in \mathcal{C}_{t-1}$. Since $\|\theta^\star\|_2 \leq B$ by definition, the set $\mathcal{C}_{t-1} \cap \Theta$ is non-empty (it contains $\theta^\star$ itself). Then, the result follows from Lem.~\ref{lemma:non-exp-proj}.
\end{proof}

The following result is immediate from the definition of good event and the non-expansion property of the projection used to compute $\wt{\theta}_t$.
\begin{lemma}\label{lemma:good-event-mu}
Let $t\in[n]$ be any time step in which the good event $G_t$ holds. Then,
\begin{align*}
 \forall x\in\X, a\in\A : |{\mu}_{\wt{\theta}_{t-1}}(x,a) - \mu_{\theta^\star}(x,a)| \leq \sqrt{\gamma_t}\| \phi(x,a)\|_{\wb{V}_{t-1}^{-1}}.
\end{align*}
\end{lemma}
\begin{proof}
Fix any $x\in\X$ and $a\in\A$. Then,
\begin{align*}
|\mu_{\wt{\theta}_{t-1}}(x,a) - \mu_{\theta^\star}(x,a)| &= |\phi(x,a)^T(\wt{\theta}_{t-1} - \theta^\star)| = |\phi(x,a)^T\bar{V}_{t-1}^{-1/2}\bar{V}_{t-1}^{1/2}(\wt{\theta}_{t-1} - \theta^\star)|\\ &\stackrel{(a)}{\leq} \|\phi(x,a)\|_{\bar{V}_{t-1}^{-1}}\|\wt{\theta}_{t-1}-\theta^\star\|_{\bar{V}_{t-1}}\\ &\stackrel{(b)}{\leq} \|\phi(x,a)\|_{\bar{V}_{t-1}^{-1}}\|\wh{\theta}_{t-1}-\theta^\star\|_{\bar{V}_{t-1}} \stackrel{(c)}{\leq} \sqrt{\gamma_t}\|\phi(x,a)\|_{\bar{V}_{t-1}^{-1}},
\end{align*}
where (a) is from Cauchy-Schwartz inequality, (b) from Cor.~\ref{cor:proj-theta-star}, and (c) from the definition of $G_t$.
\end{proof}
}

\begin{lemma}\label{lemma:pdlin-optimism}
Let $\gamma_t := c_{n,1/S_t^2}$ and $n \geq 3$. Then, for any time step $t$ in which the good event $G_t$ (see App. \ref{app:good.events}) holds,
\begin{equation}
\begin{aligned}
        f_t(\omega):= \sum_{x\in\X} \wh{\rho}_{t-1}(x) \sum_{a\in\A} \omega(x,a) \bigg( {\mu}_{\wt{\theta}_{t-1}}(x,a) &+  \sqrt{\gamma_t}\| \phi(x,a)\|_{\wb{V}_{t-1}^{-1}} \bigg)
        \\
       &\geq \sum_{x\in\X} \wh{\rho}_{t-1}(x) \sum_{a\in\A} \omega(x,a) \mu_{\theta^\star}(x,a),
\end{aligned}
\end{equation}
and
\begin{equation}
\begin{aligned}
        g_t(\omega) := \inf_{\theta' \in \wb{\Theta}_{t-1}}\sum_{x\in\X}\wh{\rho}_{t-1}(x)\sum_{a\in\A}\omega(x,a)
        \bigg( {d}_{x,a}& \left(\wt{\theta}_{t-1},\theta' \right) + \frac{2LB}{\sigma^2}\sqrt{\gamma_t}\| \phi(x,a)\|_{\wb{V}_{t-1}^{-1}} \bigg)
        \\
                      &\geq \inf_{\theta' \in {\Theta}_{alt}}\sum_{x\in\X}\wh{\rho}_{t-1}(x)\sum_{a\in\A}\omega(x,a) d_{x,a}(\theta^\star,\theta').
\end{aligned}.
\end{equation}
\end{lemma}
\begin{proof}
        Since $\wh \rho$ and $\omega$ are non-negative, the first inequality is trivial by upper bounding the true mean $\mu_{\theta^\star}(x,a)$ for each $x,a$ by using the definition of $G_t^\theta$ and Lemma \ref{lemma:good-event-mu}.
        Let us prove the second one. Fix any model $\theta' \in \Theta$. 
        By using the definition of KL divergence of Gaussians with fixed variance, we have that: 
\begin{align*}
        d_{x,a}(\theta^\star,\theta') = \frac{(\mu_{\theta'}(x,a) - \mu_{\theta^\star}(x,a))^2}{2\sigma^2}
        &\leq {d}_{x,a}(\wt{\theta}_{t-1},\theta') + \frac{2LB}{\sigma^2}|{\mu}_{\wt{\theta}_{t-1}}(x,a) - \mu_{\theta^\star}(x,a)| \\
        &\leq {d}_{x,a}(\wt{\theta}_{t-1},\theta') + \frac{2LB}{\sigma^2}\sqrt{\gamma_t}\| \phi(x,a)\|_{\wb{V}_{t-1}^{-1}},
\end{align*}
where the first inequality is from $|(a - c)^2 - (b - c)^2| = |(a + b - 2c)(a - b)| \leq 4LB|a-b|$ and the second one is once again from the definition of $G_t$ and Lemma \ref{lemma:good-event-mu}.
Therefore,
\begin{align}
        \inf_{\theta' \in {\Theta}_{alt}}\sum_{x\in\X}\wh{\rho}_{t-1}&(x)
        \sum_{a\in\A}\omega(x,a) d_{x,a}(\theta^\star,\theta')\\
        &\leq \inf_{\theta' \in \Theta_{alt}}\sum_{x\in\X}\wh{\rho}_{t-1}(x)\sum_{a\in\A}\omega(x,a)\left( {d}_{x,a}(\wt{\theta}_{t-1},\theta') + \frac{2LB}{\sigma^2}\sqrt{\gamma_t}\| \phi(x,a)\|_{\wb{V}_{t-1}^{-1}} \right).
        \notag
\end{align}
We now upper bound the infimum over models in the alternative set. Note that such set can be fully specified once we assign an optimal arm to each context. Let $\{a_x\}_{x\in\X}$ and define
\begin{align*}
    \Theta(\{a_x\}_{x\in\X}) = \{ \theta' \in \Theta | \exists x\in\X : a^\star_{\theta'}(x) \neq a_x\}.
\end{align*}
Note that $\Theta_{alt} = \Theta(\{a^\star_{\theta^\star}(x)\}_{x\in\X})$. Then,
\begin{align}
        \inf_{\theta' \in \Theta_{alt}}\sum_{x\in\X}\wh{\rho}_{t-1}(x)
        &\sum_{a\in\A}\omega(x,a) {d}_{x,a}(\wt{\theta}_{t-1},\theta')\\
        &\leq \max_{\{a_x\}_{x\in\X}}\inf_{\theta' \in \Theta(\{a_x\}_{x\in\X})}\sum_{x\in\X}\wh{\rho}_{t-1}(x)\sum_{a\in\A}\omega(x,a) {d}_{x,a}(\wt{\theta}_{t-1},\theta') \\ &\leq \inf_{\theta' \in \wb{\Theta}_{t-1}}\sum_{x\in\X}\wh{\rho}_{t-1}(x)\sum_{a\in\A}\omega(x,a) {d}_{x,a}(\wt{\theta}_{t-1},\theta').
\end{align}
To see the last inequality, note that for all $\{a_x\}_{x\in\X}$ which do not contain only the optimal arms of $\wt{\theta}_{t-1}$ (i.e., $\{a_x\}_{x\in\X} \neq \{a^\star_{\wt{\theta}_{t-1}}(x)\}_{x\in\X}$), we have $\wt{\theta}_{t-1} \in \Theta(\{a_x\}_{x\in\X})$\footnote{Recall that, by definition, $\wt{\theta}_{t-1} \in \Theta$.}, and therefore the infimum is zero. Thus, the maximum must be attained by $\{a^\star_{\wt{\theta}_{t-1}}(x)\}_{x\in\X}$, which yields $\Theta(\{a^\star_{\wt{\theta}_{t-1}}(x)\}_{x\in\X}) = \wb{\Theta}_{t-1}$. This concludes the proof.
\end{proof}

\begin{lemma}\label{lemma:pdlin-rho-dev}
For all time steps $t$,
\begin{align}
    \sum_{s\leq t: E_s,G_s}\sum_{x\in\X} \left| \hat{\rho}_{s-1}(x) - \rho(x) \right| \leq 4|\X|\left(\sqrt{S_t \log(|\X|S_t^2)} + \log S_t + 1\right).
\end{align}
\end{lemma}
\begin{proof}
Using the definition of $G_s$,
\begin{align*}
    \sum_{s\leq t: E_s,G_s}\sum_{x\in\X} \left| \hat{\rho}_{s-1}(x) - \rho(x) \right| &\leq  |\X|\sum_{s\leq t: E_s,G_s} 2\max \left( \sqrt{\frac{\log(|\X|S_s^2)}{2S_s}}, \frac{2}{s}\right)\\ &\leq  2|\X|\sum_{s=1}^{S_t} \max \left( \sqrt{\frac{\log(|\X|s^2)}{2s}}, \frac{2}{s}\right)  \\ &\leq 2|\X|\sqrt{\frac{\log(|\X|S_t^2)}{2}} \sum_{s=1}^{S_t} \frac{1}{\sqrt{s}} + 4|\X|\sum_{s=1}^{S_t} \frac{1}{s}\\ &\leq 4|\X|\left(\sqrt{S_t \log(|\X|S_t^2)} + \log S_t + 1\right),
\end{align*}
where the last inequality holds since
\begin{align*}
    \sum_{t=1}^m \sqrt{\frac{1}{t}} \leq 1 + \int_1^m x^{-1/2}dx = 1 + [2x^{1/2}]_1^m = 2\sqrt{m} - 1 < 2\sqrt{m}
\end{align*}
and $\sum_{t=1}^m \frac{1}{t} \leq \log m + 1$.
\end{proof}

\begin{lemma}\label{lemma:pdlin-Lt}
Let $t$ be such that both $E_t$ and $G_t$ occur and suppose $\nu \geq 1$. Define 
\begin{align*}
    \Psi_t := \sum_{s\leq t : E_s}\sum_{x\in\X}\hat{\rho}_{s-1}(x)\sum_{a\in\A}\omega_s(x,a) \| \phi(x,a)\|_{\bar{V}_{s-1}^{-1}}.
    \end{align*}
Then,
\begin{align*}
    \Psi_t \leq \frac{4L|\X|}{\sqrt{\nu}}\left(\sqrt{S_t \log(|\X|S_t^2)} + \log S_t + 1\right) + \frac{M_tL}{\sqrt{\nu}} +  \frac{L}{\nu}\sqrt{S_t\log S_t}. + \sqrt{2dS_t \log \frac{\nu + S_tL^2/d}{\nu}}.
\end{align*}
\end{lemma}
\begin{proof}
We start by noticing that, for all $x,a$ and $s\geq 0$,
\begin{align*}
    \| \phi(x,a) \|_{\bar{V}_{s-1}^{-1}}^2 = \phi(x,a)^T \bar{V}_{s-1}^{-1} \phi(x,a) \leq \sigma_{\max}(\bar{V}_{s-1}^{-1})\underbrace{\|\phi(x,a)\|_2^2}_{\leq L} \leq \frac{L^2}{\sigma_{\min}(\bar{V}_{s-1})} \leq \frac{L^2}{\nu},
\end{align*}
and thus $\| \phi(x,a) \|_{\bar{V}_{s-1}^{-1}} \leq L/\sqrt{\nu}$. Here $\sigma_{\max}(\cdot)$ and $\sigma_{\min}(\cdot)$ denote the maximum and minimum eigenvalue of a matrix, respectively. Splitting the steps where the good event does and does not hold,
\begin{align*}
    \Psi_t &= \sum_{s\leq t : E_s, G_s}\sum_{x\in\X}\hat{\rho}_{s-1}(x)\sum_{a\in\A}\omega_s(x,a) \| \phi(x,a)\|_{\bar{V}_{s-1}^{-1}} + \sum_{s\leq t : E_s, \neg G_s}\sum_{x\in\X}\hat{\rho}_{s-1}(x)\sum_{a\in\A}\omega_s(x,a) \| \phi(x,a)\|_{\bar{V}_{s-1}^{-1}}\\ &\leq \sum_{s\leq t : E_s, G_s}\sum_{x\in\X}(\hat{\rho}_{s-1}(x) - \rho(x))\sum_{a\in\A}\omega_s(x,a) \| \phi(x,a)\|_{\bar{V}_{s-1}^{-1}} + \frac{M_tL}{\sqrt{\nu}}\\ & \quad + \sum_{s\leq t : E_s, G_s}\sum_{x\in\X}{\rho}(x)\sum_{a\in\A}\omega_s(x,a) \| \phi(x,a)\|_{\bar{V}_{s-1}^{-1}} \\ & \leq \frac{L}{\sqrt{\nu}}\sum_{s\leq t : E_s, G_s}\sum_{x\in\X}|\hat{\rho}_{s-1}(x) - \rho(x)| + \frac{M_tL}{\sqrt{\nu}} + \sum_{s\leq t : E_s, G_s}\sum_{x\in\X}{\rho}(x)\sum_{a\in\A}\omega_s(x,a) \| \phi(x,a)\|_{\bar{V}_{s-1}^{-1}}\\ &\leq \frac{4L|\X|}{\sqrt{\nu}}\left(\sqrt{S_t \log(|\X|S_t^2)} + \log S_t + 1\right) + \frac{M_tL}{\sqrt{\nu}} + \sum_{s\leq t : E_s, G_s}\sum_{x\in\X}{\rho}(x)\sum_{a\in\A}\omega_s(x,a) \| \phi(x,a)\|_{\bar{V}_{s-1}^{-1}},
\end{align*}
where in the first and second inequality we bounded the expected feature-norms by their maximum value and added/subtracted the first term with the true context distribution. In the last step we applied Lemma \ref{lemma:pdlin-rho-dev}. We now focus exclusively on the third term. 
Using the fact that the good event holds at time $t$,
\begin{align*}
    \sum_{s\leq t : E_s, G_s}\sum_{x\in\X}{\rho}(x)\sum_{a\in\A}\omega_s(x,a) \| \phi(x,a)\|_{\bar{V}_{s-1}^{-1}} & \leq \sum_{s\leq t : E_s}\sum_{x\in\X}{\rho}(x)\sum_{a\in\A}\omega_s(x,a) \| \phi(x,a)\|_{\bar{V}_{s-1}^{-1}}\\ & \leq \sum_{s\leq t: E_s} \| \phi(X_s,A_s)\|_{\bar{V}_{s-1}^{-1}} + \frac{L}{\nu}\sqrt{S_t\log S_t}.
\end{align*}
Finally, let $\bar{V}_{e,t} := \sum_{s\leq t: E_s} \phi(X_s,A_s)\phi(X_s,A_s)^T + \nu I$ denote the regularized design matrix computed using only the exploration rounds. Then, we have $\bar{V}_t \succeq \bar{V}_{e,t}$ (since sum of rank-one matrices), which implies $\bar{V}_t^{-1} \preceq \bar{V}_{e,t}^{-1}$ and thus $\| \phi(x,a)\|_{\bar{V}_{s-1}^{-1}} \leq \| \phi(x,a)\|_{\bar{V}_{e,s-1}^{-1}}$.
Here $\succeq$ denotes the Loewner ordering, i.e., for two symmetric matrices $A,B$ we have $A \succeq B$ ($A \succ B$) if $A - B$ is positive semi-definite (positive definite). Therefore,
\begin{align*}
    \sum_{s\leq t: E_s} \| \phi(X_s,A_s)\|_{\bar{V}_{s-1}^{-1}}  &\leq \sum_{s\leq t: E_s} \| \phi(X_s,A_s)\|_{\bar{V}_{e,s-1}^{-1}} \stackrel{(a)}{\leq} \sqrt{S_t \sum_{s\leq t: E_s} \| \phi(X_s,A_s)\|_{\bar{V}_{e,s-1}^{-1}}^2}\\ &\stackrel{(b)}{\leq} \sqrt{2S_t\log \frac{\det(\bar{V}_{e,t})}{\nu^d}} \stackrel{(d)}{\leq} \sqrt{2dS_t \log \frac{\nu + S_tL^2/d}{\nu}},
\end{align*}
where in (a) we equivalently rewritten the first term as a sum over exploration rounds, (b) is from Cauchy-Schwartz inequality, in (c) we used Lemma 11 of \cite{abbasi2011improved}, and in (d) we used the determinant-trace inequality (Lemma 10 of \cite{abbasi2011improved}) to bound the determinant of $\bar{V}_{e,t}$ by $(\nu + S_tL^2/d)^d$. The final statement follows by combining the previous bounds.
\end{proof}

\subsection{Online Convex Optimization}

Here we recall some basic results from online convex optimization. See \citep[e.g.,][]{beck2003mirror} for detailed proofs and discussion of these results.

\begin{lemma}[Recursion bound for subgradient descent]\label{lemma:rec-subgrad}
Let $\sup_{t\geq 1: E_t} |g_t(\omega_t, z_k)|^2 \leq b_{\lambda}$. For any phase $k \geq 0$, $t\in\mathcal{T}_k^E$, and $\lambda \in \mathbb{R}_+$, the incremental updates to the Lagrange multiplier $\{\lambda_t\}_{t\in\mathcal{T}_k^E}$ of Algorithm \ref{alg:primal-dual-alg-phased-nostop} satisfy
\begin{align*}
    \sum_{s\leq t : s\in\mathcal{T}_k^E} g_s(\omega_s,z_k)(\lambda_s - \lambda) \leq \frac{1}{2\alpha_k^\lambda} (\lambda - \lambda_1)^2 + \frac{\alpha_k^\lambda b_{\lambda}^2}{2}S_{t,k}.
\end{align*}
\end{lemma}
\begin{proof}
Recall that the optimization process is reset at the beginning of each phase. Let $\tau_{s,k}$ be a random variable indicating the time at which the $s$-th exploration round of phase $k$ occurs. Note that $\lambda_{\tau_{1,k}} = \lambda_1$. In order to simplify the exposition, and with some abuse of notation, let $\lambda_s = \lambda_{\tau_{s,k}}$ and $g_s = g_{\tau_{s,k}}(\omega_{\tau_{s,k}}, z_k)$. By definition of the update rule, for each $s\geq 1$,
\begin{align*}
    (\lambda_{s+1} - \lambda)^2 &= (\min\{[\lambda_s - \alpha_k^\lambda g_s]_+, \lambda_{\max}\} - \lambda)^2 = \min\{[\lambda_s - \alpha_k^\lambda g_s]_+ - \lambda, \lambda_{\max} -\lambda\}^2\\ &\leq (\lambda_s - \alpha_k^\lambda g_s - \lambda)^2 = (\lambda_s - \lambda)^2 + (\alpha_k^\lambda g_s)^2 + 2 \alpha_k^\lambda (\lambda - \lambda_s) g_s.
\end{align*}
Dividing by $2\alpha_k^\lambda$ and rearranging,
\begin{align*}
    (\lambda_s - \lambda) g_s \leq \frac{(\lambda_s - \lambda)^2 - (\lambda_{s+1} - \lambda)^2}{2\alpha_k^\lambda} + \frac{\alpha_k^\lambda}{2} g_s^2. 
\end{align*}
Summing over all $s$ up to $S_t$ and noting that the first sum on the right-hand side is telescopic,
\begin{align*}
    \sum_{s=1}^{S_t}(\lambda_s - \lambda) g_s &\leq \frac{1}{2\alpha_k^\lambda}(\lambda_1 - \lambda)^2 - \frac{1}{2\alpha_k^\lambda}(\lambda_{S_t + 1} - \lambda)^2  + \frac{\alpha_k^\lambda}{2}\sum_{s=1}^{S_t} g_s^2.
\end{align*}
The proof is concluded by upper-bounding the second term by zero and mapping the exploration counter $s$ back to time steps.
\end{proof}

\begin{lemma}\label{lemma:rec-omd}[Recursion bound for Online Mirror Descent (OMD)]
Let $\omega_1$ be the uniform distribution over actions for each context and $\sup_{t \geq 1 : E_t} \| q_t \|_{\infty} \leq b_\omega$. For any phase $k \geq 0$, $t\in\mathcal{T}_k^E$, and $\omega \in \Omega$, the OMD updates of Algorithm \ref{alg:primal-dual-alg-phased-nostop} satisfy
\begin{align*}
    \sum_{s\leq t : s\in\mathcal{T}_k^E} h_s(\omega_s, \lambda_s, z_k) - \sum_{s\leq t : s\in\mathcal{T}_k^E} h_s(\omega, \lambda_s, z_k) \geq -\frac{\log |\A|}{\alpha_k^\omega} - \frac{\alpha_k^\omega b_\omega^2}{2}S_{t,k}.
\end{align*}
\end{lemma}
\begin{proof}
We can follow the same steps as before, mapping time steps to exploration counters and then applying the standard recursion bound for OMD \citep[e.g.,][]{beck2003mirror}.
\end{proof}
\begin{corollary}\label{cor:rec-pd}[Recursion bound for primal-dual algorithm]
For any phase $k \geq 0$, $t\in\mathcal{T}_k^E$, $\omega \in \Omega$, and $\lambda \in \mathbb{R}_+$, under the same conditions as in Lemma \ref{lemma:rec-omd} and \ref{lemma:rec-subgrad},
\begin{align*}
    \sum_{s\leq t : s\in\mathcal{T}_k^E} f_s(\omega_s) &\geq  \sum_{s\leq t : s\in\mathcal{T}_k^E} h_s(\omega, \lambda_s, z_k) - \lambda\sum_{s\leq t: s\in\mathcal{T}_k^E}g_s(\omega_s,z_k) -\frac{\log |\A|}{\alpha_k^\omega} - \frac{\alpha_k^\omega b_\omega^2}{2}S_{t,k} \\ &- \frac{1}{2\alpha_k^\lambda} (\lambda - \lambda_1)^2 - \frac{\alpha_k^\lambda b_{\lambda}^2}{2}S_{t,k}.
\end{align*}
\end{corollary}
\begin{proof}
The proof is straightforward by expanding $\sum_{s\leq t : s\in\mathcal{T}_k^E} h_s(\omega_s, \lambda_s, z_k) = \sum_{s\leq t : s\in\mathcal{T}_k^E}( f_s(\omega_s) + \lambda_s g_s(\omega_s,z_k))$ and combining Lemma \ref{lemma:rec-omd} with Lemma \ref{lemma:rec-subgrad}.
\end{proof}

\section{Confidence Set for Regularized Least-Squares (Proof of Thm.~\ref{th:conf-theta})}\label{app:conf.set}

The following theorem is the extended version of Thm.~\ref{th:conf-theta}. It provides a refined confidence set for the parameters estimated by regularized least-squares.

\begin{theorem}[Confidence set over parameters]\label{th:conf-theta.extended}
	Let $\delta\in(0,1)$ and $n\geq3$. Then,
	\begin{align*}
	\prob{\exists t\in[n] : \|\wh{\theta}_t - \theta^\star\|_{\wb{V}_t} \geq \sqrt{c_{n,\delta}}} \leq \delta,
	\end{align*}
	where $\sqrt{c_{n,\delta}} := \frac{\gamma_n}{1 - \frac{1}{\log n}} \sqrt{\kappa_{n,\delta}}$, $\gamma_n := 1 + \frac{1}{\log n}$, and
	\begin{align*}
	\sqrt{\kappa_{n,\delta}} = B\sqrt{\nu}+  \sqrt{\frac{2\sigma^2 \log\left(\frac{2+\frac{2nL^2}{d\nu}}{\delta}\right)}{(\log n)^2}} + \sqrt{2\sigma^2\gamma_n^3\log\left(\frac{2(1 + \log (n/\chi_n) \log(n))}{\delta}\right) + 2\gamma_n^3\Upsilon_n}.
	\end{align*}
		Finally, we set $\Upsilon_n := d\log\left(\frac{5}{2} + 2\log n\sqrt{d}\right) + d\log\left(2 + 4d\log\left(4\gamma_nd(\log n)^2\sqrt{\frac{\nu + L^2n}{d\nu}}\right)\log n \right)$ and $\chi_n :=  \frac{\nu^2 v_{\min}^2}{16dL^2(\nu + L^2n)(\log n)^4\gamma_n^4}$.
\end{theorem}

\paragraph{Asymptotic dependence} It is important to note that $\lim_{n\rightarrow\infty}\frac{c_{n,1/n}}{2\sigma^2\log n} = 1$.

\subsection{Proof of Thm.~\ref{th:conf-theta.extended}}

The proof can be summarized in three main steps:
\begin{enumerate}
    \item We reduce the problem of bounding $\|\wh{\theta}_t - \theta^\star\|_{\wb{V}_t}$ to one in which we need to bound $(\wh{\theta}_t - \theta^\star)^T\wb{V}_t^{1/2}v$ for any $v \in \mathcal{C}_1$, where $\mathcal{C}_1 \subset \mathbb{R}^d$ is a (finite) $\epsilon_1$-cover of the $d$-dimensional Euclidean unit ball. We build this cover in such a way that all its elements have norm bounded from below by a strictly positive constant and from above.
    \item We extend Theorem 8 of \cite{lattimore2017end} to bound $(\wh{\theta}_t - \theta^\star)^T\wb{V}_t^{1/2}v$ uniformly over all $v\in\mathcal{C}_1$, instead of the prediction errors $(\wh{\theta}_t - \theta^\star)^T\phi(x,a)$ uniformly over all contexts/arms.
            This requires a second $\epsilon_2$-cover (we shall call it $\mathcal{C}_2$) of the set $\{\wb{V}_t^{-1/2}v : t \in [n], v\in\mathcal{C}_1\}$. The result is reported in Lemma~\ref{lemma:conf-cover}. 
    \item The resulting bound is of order $\mathcal{O}(\log(1/\delta) + d\log (1/\epsilon_1))$, which requires tuning $\epsilon_1 = \frac{1}{\log n}$ to cancel the bias of the first cover asymptotically without compromising the size of the cover itself.
\end{enumerate}

\paragraph{Step 1.} We start from the fact that 
\begin{align}
        \label{eq:norm_decomposition}
    \|\wh{\theta}_t- \theta^\star\|_{\wb{V}_t} = \frac{(\wh{\theta}_t- \theta^\star)^T\wb{V}_t (\wh{\theta}_t- \theta^\star)}{\|\wh{\theta}_t- \theta^\star\|_{\wb{V}_t}} = (\wh{\theta}_t- \theta^\star)^T \wb{V}_t^{1/2}z_t,
\end{align}
where $z_t = \frac{\wb{V}_t^{1/2} (\wh{\theta}_t- \theta^\star)}{\|\wh{\theta}_t- \theta^\star\|_{\wb{V}_t}}$ is such that $\|z_t\|_2 = 1$. 
To handle the fact that $z_t$ is random, we build a \emph{linear} $(\epsilon_1 > 0)$-cover of the space $\mathcal{Z} = \{z \in\mathbb{R}^d : \|z\|_2 \leq 1\}$, which includes $z_t$ for all $t=1,\dots,n$
Let $\epsilon_1' > 0$, $\{e_1,e_2,\dots,e_d\}$ be the canonical basis of $\mathbb{R}^d$, and define
\begin{align*}
    \wt{\mathcal{C}}_1 := \left\{ \sum_{i=1}^d a_i e_i : a_i \in \left\{\pm \epsilon_1' \Big(\frac{1}{2} +  j\Big) : j = 0,1,\dots,\wb{j}\right\} \forall i \in [d] \right\},
\end{align*}
where $\wb{j} := \left\lceil \frac{1}{\epsilon_1'} - \frac{1}{2} \right\rceil$. 
For any vector $z \in \mathcal{Z}$, we can find a vector in $\wt{\mathcal{C}}_1$ with at most $\epsilon_1'$ error on each component of $z$, which leads to $\min_{v \in \wt{\mathcal{C}}_1}\|v-z\|_2 \leq \epsilon_1'\sqrt{d}$~\citep[see e.g.,][Chap. 27]{shalev2014understanding}.
Setting $\epsilon_1' = \epsilon_1 / \sqrt{d}$ gives an $\epsilon_1$-cover of the unit ball in $\ell_2$-norm. The only problem with this cover is that it contains vectors with norm bigger than $1$ and scaling with $d$,\footnote{Consider the vector with all components equal to 1, whose norm is $\sqrt{d}$.} which may lead to an undesirable dependency later on. However, we can safely remove the vectors with large norm without affecting the desired accuracy of the cover.
Without loss of generality, select $z\in \mathcal{Z}$ in the positive orthant (i.e., $z_i \geq 0$, for any $i \in [d]$) such that we make an error of $\epsilon_1'$ on each component (i.e., the worst-case) and let $w = z + \epsilon_1'$. Then
\begin{align*}
    \|w\|_2^2 = \sum_{i=1}^d (z_i + \epsilon_1')^2 = \underbrace{\|z\|_2^2}_{\leq 1} + d(\epsilon_1')^2 + 2\epsilon_1'\underbrace{\sum_{i=1}^d z_i}_{\leq \|z\|_1 \leq \sqrt{d}} \leq 1 + \epsilon_1^2 + 2\epsilon_1 = (1+\epsilon_1)^2.
\end{align*}
Hence vectors with norm at most $(1+\epsilon_1)$ actually suffice and thus we can set $\mathcal{C}_1 = \wt{\mathcal{C}}_1 \setminus \{v\in\wt{\mathcal{C}}_1 : \|v\|_2 > (1+\epsilon_1)\}$. Then we upper bound the size of this cover as
\begin{align*}
    |\mathcal{C}_1| \leq |\wt{\mathcal{C}}_1| = 2^d(1 + \wb{j})^d \leq \left(\frac{5}{2} + \frac{2}{\epsilon_1'}\right)^d = \left(\frac{5}{2} + \frac{2\sqrt{d}}{\epsilon_1}\right)^d.
\end{align*}
To recap, our cover $\mathcal{C}_1$ has the following properties:
\begin{enumerate}
    \item $\forall z \in \mathcal{Z} = \{z\in\mathbb{R}^d : \|z\|_2 \leq 1\},\,\exists v \in \mathcal{C}_1: \|z - v\|_2 \leq \epsilon_1$ \label{item:c1.p1}
    \item $|\mathcal{C}_1| \leq \left(\frac{5}{2} + \frac{2\sqrt{d}}{\epsilon_1}\right)^d$ \label{item:c1.p2}
    \item $\forall v \in \mathcal{C}_1 : \|v\|_2 \leq v_{\max} := 1 + \epsilon_1$ \label{item:c1.p3}
    \item $\forall v \in \mathcal{C}_1, i \in [d] : |v_i| \geq v_{\min} := \frac{\epsilon_1}{2\sqrt{d}}$ (this follows from the discretization used in $\wt{\mathcal{C}}_1$ and it implies that $\|v\|_2 \geq v_{\min}\sqrt{d} = \frac{\epsilon_1}{2}$) \label{item:c1.p4}
\end{enumerate}

\paragraph{Step 2.}
We use an extension of Thm.~8 of~\citep{lattimore2017end} to bound the prediction error at vectors in the cover $\mathcal{C}_1$ after applying the linear transformation $\wb{V}_t^{1/2}$. 
\begin{lemma}\label{lemma:conf-cover}
	Let $\mathcal{C} \subset \mathbb{R}^d$ be a finite set such that, for any $v \in \mathcal{C}$, $\|v\|_2 \leq v_{\max} < \infty$ and $|v_i| \geq v_{\min} > 0$, $\forall i\in[d]$.
	Suppose that $n \geq 2$. Then, for any $\delta \in (0,1)$,
	\begin{align*}
	\prob{\exists t\leq n, v\in\mathcal{C} :  (\wh{\theta}_t - \theta^\star)^T\wb{V}_t^{1/2}v \geq \sqrt{\kappa_{n,\delta}}\|v\|_{2}} \leq \delta,
	\end{align*}
	where
	\begin{align*}
	\sqrt{\kappa_{n,\delta}} = B\sqrt{\nu} +  \sqrt{\frac{2\sigma^2 \log\left(\frac{2+\frac{2nL^2}{d\nu}}{\delta}\right)}{(\log n)^2}} + \sqrt{2\sigma^2\gamma_n^3\log\left(\frac{2(1 + \log (n/\chi_n) \log(n))}{\delta}\right) + 2\gamma_n^3\Upsilon_n}
	\end{align*}
	and $\Upsilon_n = \log(|\mathcal{C}|) + d\log\left(2 + 4d\log\left(2d\log n\frac{v_{\max}}{v_{\min}}\sqrt{\frac{\nu + L^2n}{d\nu}}\right)\log n \right)$ and $\chi_n =  \frac{\nu^2 v_{\min}^2}{4L^2(\nu + L^2n)(\log n)^2v_{\max}^2\gamma_n^2}$.
\end{lemma}

The specific shape of the bound is obtained by exploiting the properties of the cover $\mathcal{C}_1$ derived in the first step, where $v_{\max} = 1+\epsilon_1$ and $v_{\min} = \frac{\epsilon_1}{2\sqrt{d}}$.

\paragraph{Step 3.}
We finally tune $\epsilon_1$ to obtain the final bound. With probability at least $1-\delta$, we have that
\begin{align*}
    \|\wh{\theta}_t - \theta^\star\|_{\wb{V}_t} 
    &\stackrel{(a)}{=}(\wh{\theta}_t- \theta^\star)^T \wb{V}_t^{1/2}z_t \stackrel{(b)}{\leq} \max_{z \in \mathcal{Z}} (\wh{\theta}_t - \theta^\star)^T\wb{V}_t^{1/2}z\\ 
    &= \max_{z\in \mathcal{Z}} \min_{v\in\mathcal{C}_1} \left\{  (\wh{\theta}_t - \theta^\star)^T\wb{V}_t^{1/2}(z - v) + (\wh{\theta}_t - \theta^\star)^T\wb{V}_t^{1/2}v \right\}\\ 
    &\stackrel{(c)}{\leq} \max_{z \in \mathcal{Z}} \min_{v\in\mathcal{C}_1} \Big\{  \|\wh{\theta}_t - \theta^\star\|_{\wb{V}_t}\|z-v\|_2 + \sqrt{\kappa_{n,\delta}}\|v\|_{2} \Big\} \\
    &\stackrel{(d)}{\leq}\epsilon_1\|\wh{\theta}_t - \theta^\star\|_{\wb{V}_t} + (1+\epsilon_1)\sqrt{\kappa_{n,\delta}},
\end{align*}
where $(a)$ follows from Eq.~\ref{eq:norm_decomposition}, (b) from the fact that $z_t \in \mathcal{Z}$, $(c)$  holds with probability at least $1-\delta$ by Lem.~\ref{lemma:conf-cover} and $(d)$ by properties~\ref{item:c1.p1} and~\ref{item:c1.p3} of the cover $\mathcal{C}_1$. 
The statement of the theorem follows by setting $\epsilon_1 = \frac{1}{\log n}$ and rearranging.

\subsection{Proof of Lem.~\ref{lemma:conf-cover}}\label{app:conf.set.proof.supplemma}

The proof follows similar steps as in~\citep[][Thm.\ 8]{lattimore2017end}.

\begin{proof}
Take any $v\in\mathcal{C}_1$ and $t\in[n]$. Then,
\begin{align}
    \notag
    (\wh{\theta}_t - \theta^\star)^T\wb{V}_t^{1/2}v 
    &\stackrel{(a)}{=}  \left(\wb{V}_t^{-1} \sum_{s=1}^t \phi(X_s,A_s)Y_s - \theta^\star\right)^T\wb{V}_t^{1/2}v\\ 
    \notag
    &\stackrel{(b)}{=} \left(\wb{V}_t^{-1} \sum_{s=1}^t \phi(X_s,A_s)(\phi(X_s,A_s)^T\theta^\star + \xi_s) - \theta^\star\right)^T\wb{V}_t^{1/2}v\\ 
    \notag
    &\stackrel{(c)}{=} \left(\wb{V}_t^{-1} V_t\theta^\star + \wb{V}_t^{-1} \sum_{s=1}^t\phi(X_s,A_s)\xi_s - \theta^\star\right)^T\wb{V}_t^{1/2}v\\
    \label{eq:decomp}
    &\stackrel{(d)}{=} \underbrace{\left(\wb{V}_t^{-1} V_t\theta^\star - \theta^\star\right)^T\wb{V}_t^{1/2}v}_{(i)} + \underbrace{\sum_{s=1}^tv^T\wb{V}_t^{-1/2} \phi(X_s,A_s)\xi_s}_{(ii)},
\end{align}
where (a) is from the definition of $\wh{\theta}_t$, (b) since $Y_s = \phi(X_s,A_s)^T\theta^\star + \xi_s$ with $\xi_s \sim \mathcal{N}(0,\sigma^2)$, (c) from the definition of $V_t$, and (d) after rearranging. Let us bound (i). Since $\theta^\star = \wb{V}_t^{-1}\wb{V}_t\theta^\star$, we have
\begin{align*}
    (i) = v^T \wb{V}_t^{-1/2}(V_t - \wb{V}_t)\theta^\star = - \nu v^T \wb{V}_t^{-1/2}\theta^\star,
\end{align*}
where we used $\wb{V}_t = \nu I + V_t$. Therefore,
\begin{align*}
    |(i)| \leq \nu |v^T \wb{V}_t^{-1/2}\theta^\star| \leq \nu \|v \|_2\|\wb{V}_t^{-1/2}\theta^\star\|_2 = \nu \|v \|_{2}\|\theta^\star\|_{\wb{V}_t^{-1}},
\end{align*}
where the second inequality is by Cauchy-Schwartz inequality. Since $\wb{V}_t \succeq \nu I$, $\|\theta^\star\|_{\wb{V}_t^{-1}} \leq \frac{1}{\sqrt{\nu}}\|\theta^\star\|_2 \leq \frac{B}{\sqrt{\nu}}$. This yields
\begin{align*}
    |(i)| \leq B\sqrt{\nu}\|v \|_{2}.
\end{align*}
Let us consider the second term. Since $\wb{V}_t^{-1/2}$ is random, we proceed using the same covering argument as in the proof in \citep[][Thm. 8]{lattimore2017end}. Let $\epsilon_2 > 0$ (whose value will be specified later). Recall that our input is a finite set of $d$-dimensional vectors $\mathcal{C}_1$ such that $\|v\|_2 \leq v_{\max} < \infty$ and $|v_i| \geq v_{\min} > 0$ hold for all $v\in\mathcal{C}_1$ and $i\in[d]$. Note that the latter condition implies $\|v\|_2 \geq v_{\min}\sqrt{d}$. Our goal is to build an $\epsilon_2$-covering set of $\{\wb{V}_t^{-1/2}v : t\in[n], v\in\mathcal{C}_1\}$. Since this set is random, we build a deterministic one that contains the former almost surely and cover it instead. Note that, for any $t\in[n]$, $\wb{V}_t^{-1/2}$ is such that (1) $\wb{V}_t^{-1/2} \succ 0$, (2) $\|\wb{V}_t^{-1/2}\|_2 = \sigma_{\max}(\wb{V}_t^{-1/2}) \leq \frac{1}{\sqrt{\nu}}$, and (3) $\sigma_{\min}(\wb{V}_t^{-1/2}) \geq \frac{1}{\sqrt{\nu + L^2n}}$. Let $\mathcal{D}$ denote the set of $d\times d$ matrices with these properties, that is,
\begin{align*}
    \mathcal{D} := \left\{ D \in \mathbb{R}^{d\times d} : D \succ 0, \|D\|_2 \leq \frac{1}{\sqrt{\nu}},\sigma_{\min}(D) \geq \frac{1}{\sqrt{\nu + L^2n}} \right\}.
\end{align*}
Then, $\wb{V}_t^{-1/2} \in \mathcal{D}$ for all $t\in[n]$ and our initial set to be covered is almost surely contained into $\mathcal{B} := \{Dv : D\in\mathcal{D},v\in\mathcal{C}_1\}$. Furthermore, $v_{\min}\sqrt{\frac{d}{\nu + L^2n}}\leq \|b\|_2 \leq \frac{v_{\max}}{\sqrt{\nu}}$ for all $b\in\mathcal{B}$. We shall now cover $\mathcal{B}$.
Let $\{e_1,\dots,e_d\}$ be the canonical basis of $\mathbb{R}^d$ and, for all $v\in\mathcal{C}_1$ we introduce a cover with \textit{geometric scale} as
\begin{align*}
    \tilde{\mathcal{C}}_{2,v} := \left\{\sum_{i=1}^d a_i e_i \big|\ \forall i\in[d] : a_i \in \left\{ \pm  \frac{\epsilon_2\|v\|_2(1+\epsilon_2)^j}{\sqrt{\nu + L^2n}} : j=0,1,\dots,\wb{j}\right\}\right\},
\end{align*}
where $\wb{j} := \left\lceil \frac{\log\left(\frac{v_{\max}}{\epsilon_2v_{\min}}\sqrt{\frac{\nu + L^2n}{d\nu}}\right)}{\log(1+\epsilon_2)}\right\rceil$ is such that $\frac{\epsilon_2\|v\|_2(1+\epsilon_2)^{\wb{j}}}{\sqrt{\nu + L^2n}} \geq \frac{v_{\max}}{\sqrt{\nu}}$ (i.e., the maximum absolute value of each element in $\mathcal{B}$).
Then, our cover is $\tilde{\mathcal{C}}_2 = \bigcup_{v\in\mathcal{C}_1} \tilde{\mathcal{C}}_{2,v}$. Let us analyze some its properties. First its size is
\begin{align}\label{eq:c2-size}
    |\tilde{\mathcal{C}}_2| \leq |\mathcal{C}_1|\left(2 +\frac{\log\left(\frac{v_{\max}}{\epsilon_2v_{\min}}\sqrt{\frac{\nu + L^2n}{d\nu}}\right)}{\log(1+\epsilon_2)} \right)^d.
\end{align}
Then, we can show the following covering property in $l_\infty$-norm.
\begin{proposition}\label{prop:cover2}
For all $v\in\mathcal{C}_1$, $t\in[n]$, there exists $\wb{w}_{v,t} \in \tilde{\mathcal{C}}_2$ such that
\begin{align*}
    \forall i \in [d] : \left|\Big[\wb{V}_t^{-1/2}v - \wb{w}_{v,t} \Big]_i\right| \leq \epsilon_2\max\left\{\left|\Big[\wb{V}_t^{-1/2}v \Big]_i \right|, \frac{\|v\|_2}{\sqrt{\nu + L^2n}}\right\}.
\end{align*}
\begin{proof}
For simplicity, denote $b := \wb{V}_t^{-1/2}v$. By definition, we have $b \in \mathcal{B}$ (i.e., the deterministic set that we actually covered). We shall build a vector $w\in\mathcal{C}_2$ which has the desired property.
Take any component $b_i$, with $i \in [d]$, then\\
(1) If $|b_i| < \frac{\epsilon_2\|v\|_2}{\sqrt{\nu + L^2n}}$, then we can set $w_i = \frac{\epsilon_2\|v\|_2}{\sqrt{\nu + L^2n}}\mathrm{sign}(b_i)$ and we have
\begin{align*}
    |w_i - b_i| \leq |w_i| = \frac{\epsilon_2\|v\|_2}{\sqrt{\nu + L^2n}}.
\end{align*}
(2) If $|b_i| \geq \frac{\epsilon_2\|v\|_2}{\sqrt{\nu + L^2n}}$, by the geometrical cover, we can find a point $w_i$ such that $1 \leq \frac{|w_i|}{|b_i|} \leq 1 + \epsilon_2$. Too see this, suppose, without loss of generality, that $b_i$ is positive. Note that, since $b_i$ lies in the range $[\frac{\epsilon_2\|v\|_2}{\sqrt{\nu + L^2n}}, \frac{v_{\max}}{\sqrt{\nu}}]
$ which is covered geometrically, there exists a real value $0 \leq k \leq \bar{j}$ such that $b_i = \frac{\epsilon_2\|v\|_2}{\sqrt{\nu + L^2n}}(1+\epsilon_2)^k$. Then, if we set $w_i = \frac{\epsilon_2\|v\|_2}{\sqrt{\nu + L^2n}}(1+\epsilon_2)^{\lceil k\rceil}$, we can easily verify the desired property. This implies
\begin{align*}
    |w_i - b_i| \leq |w_i| - |b_i| \leq \epsilon_2|b_i|,
\end{align*}
where the left-hand side  is from the reverse triangle inequality.
The statement follows by combining the two cases.
\end{proof}
\end{proposition}
An immediate consequence of Proposition \ref{prop:cover2} is that, for all $v\in\mathcal{C}_1$, $t\in[n]$, there exists $\wb{w}_{v,t} \in \tilde{\mathcal{C}}_2$ which can be written as $\wb{w}_{v,t} = \wb{V}_t^{-1/2}v + \zeta$, where $\zeta\in\mathbb{R}^d$ is a vector of errors such that $|\zeta_i| \leq \epsilon_2\max\left\{\Big|\Big[\wb{V}_t^{-1/2}v\Big]_i\Big|, \frac{\|v\|_2}{\sqrt{\nu + L^2n}}\right\}$ for all $i\in[d]$.

Note that, by definition, $\tilde{\mathcal{C}}_2$ contains vectors with norm that scales in $\sqrt{d}$ (e.g., the vector with all components larger or equal to $v_{\max}/\sqrt{\nu}$, which has norm $v_{\max}\sqrt{d/\nu}$ belongs to $\tilde{\mathcal{C}}_2$). These vectors will create an undesirable dependency on $d$ later on, and so we need to perform some pruning before proceeding. Take any $b\in\mathcal{B}$ and suppose that $b=Dv$ for $v\in\mathcal{C}_1$ and $D\in\mathcal{D}$. Let $\mathcal{I} := \{i\in[d] : |b_i| < \frac{\epsilon_2\|v\|_2}{\sqrt{\nu + L^2n}}\}$ be the set of components $i$ such that $|b_i|$ is below the starting point of our geometrical grid $\tilde{C}_{2,v}$ and $\mathcal{I}^c = [d] \setminus \mathcal{I}$. From the proof of Proposition \ref{prop:cover2}, we know that the vector $w \in \tilde{\mathcal{C}}_2$ that is the closest to $b$ is such that $|w_i| \leq \frac{\epsilon_2\|v\|_2}{\sqrt{\nu + L^2n}}$ for $i\in\mathcal{I}$ and $|w_i|/|b_i| \leq 1+\epsilon_2$ for $i\in\mathcal{I}^c$. Therefore,
\begin{align*}
\|w\|^2_2 = \sum_{i\in\mathcal{I}} |w_i|^2 + \sum_{i\in\mathcal{I}^c} |w_i|^2 \leq  |\mathcal{I}|\frac{\epsilon_2^2\|v\|_2^2}{\nu + L^2n} +(1+\epsilon_2)^2 \sum_{i\in\mathcal{I}^c} |b_i|^2 \leq \frac{d\epsilon_2^2\|v\|_2^2}{\nu + L^2n} + (1+\epsilon_2)^2\|b\|^2_2.
\end{align*}
This implies that $\|w\|_2 \leq \frac{\sqrt{d}\epsilon_2\|v\|_2}{\sqrt{\nu + L^2n}} + (1+\epsilon_2)\|b\|_2$. Recall that $\|b\|_2 \leq \frac{v_{\max}}{\sqrt{\nu}}$ and $\|v\|_2 \leq v_{\max}$. Thus, $\|w\|_2 \leq \frac{\sqrt{d}\epsilon_2v_{\max}}{\sqrt{\nu + L^2n}} + (1+\epsilon_2)\frac{v_{\max}}{\sqrt{\nu}} \leq \frac{v_{\max}}{\sqrt{\nu}}\left(1+ \epsilon_2(1+\sqrt{d})\right)$.
This condition holds for all ``useful" vectors in our cover (i.e., those that are the closest to some of the vectors we need to cover). Therefore, we can safely set $\mathcal{C}_2 = \left\{w\in\tilde{\mathcal{C}}_2 : \|w\|_2 \leq \frac{v_{\max}}{\sqrt{\nu}}\left(1+ \epsilon_2(1+\sqrt{d})\right)\right\}$ as our final cover. Note that Proposition \ref{prop:cover2} still holds for $\mathcal{C}_2$ since we removed only vectors that cannot be the closest to any of the points to be covered. In the following, we set $w_{\max} := \frac{v_{\max}}{\sqrt{\nu}}\left(1+ \epsilon_2(1+\sqrt{d})\right)$ as the maximum norm of any vector in $\mathcal{C}_2$.

Let us now go back to bounding term (ii) in Eq.~\ref{eq:decomp}. Let $\wb{w}_{v,t} := \argmin_{w\in\mathcal{C}_2}\left\|\wb{V}_t^{-1/2}v - w\right\|_{1}$ be the vector in our cover $\mathcal{C}_2$ which is the closest to $\wb{V}_t^{-1/2}v$ uniformly over all components.
Then,
\begin{align*}
        (ii) &:= \sum_{s=1}^tv^T\wb{V}_t^{-1/2} \phi(X_s,A_s)\xi_s = \left(\wb{V}_t^{-1/2}v\right)^T\sum_{s=1}^t \phi(X_s,A_s)\xi_s\\ & = \left(\wb{V}_t^{-1/2}v - \wb{w}_{v,t}\right)^TW_t + \wb{w}_{v,t}^T W_t \leq \underbrace{\left\|\wb{V}_t^{-1/2}v - \wb{w}_{v,t}\right\|_{\wb{V}_t}}_{(a)}\underbrace{\|W_t\|_{\wb{V}_{t}^{-1}}}_{(b)} + \underbrace{\wb{w}_{v,t}^T W_t}_{(c)},
\end{align*}
where we defined $W_t := \sum_{s=1}^t \phi(X_s,A_s)\xi_s$. We start from (a). Using the error-decomposition property from Proposition \ref{prop:cover2}, we can write $\left\|\wb{V}_t^{-1/2}v - \wb{w}_{v,t}\right\|_{\wb{V}_t}= \left\|\zeta\right\|_{\wb{V}_t}$ for some vector $\zeta\in\mathbb{R}^d$ with $|\zeta_i| \leq \epsilon_2\max\left\{|[\wb{V}_t^{-1/2}v]_i|, \frac{\|v\|_2}{\sqrt{\nu + L^2n}}\right\}$ for all $i\in[d]$.
Since this implies $|\zeta_i| \leq \epsilon_2\left(|[\wb{V}_t^{-1/2}v]_i| +  \frac{\|v\|_2}{\sqrt{\nu + L^2n}}\right)$, we have
\begin{align*}
 \left\|\zeta\right\|_{\wb{V}_t} \stackrel{(d)}{\leq} \epsilon_2 \left\|\wb{V}_t^{-1/2}v\right\|_{\wb{V}_t} + \frac{\epsilon_2\|v\|_2}{\sqrt{\nu + L^2n}} \left\|\bm{1}_d\right\|_{\wb{V}_t} \stackrel{(e)}{\leq} \epsilon_2 \left\|v\right\|_{2} + \frac{\epsilon_2\|v\|_2\sqrt{d}}{\sqrt{\nu + L^2n}} \left\|\wb{V}_t^{1/2}\right\|_{2} \stackrel{(f)}{\leq} \epsilon_2 \left\|v\right\|_{2} + \epsilon_2\|v\|_2\sqrt{d},
\end{align*}
where in (d) we used the triangle inequality ($\bm{1}_d$ denotes the d-dimensional vector of ones), in (e) we used $\left\|\bm{1}_d\right\|_{\wb{V}_t} \leq \left\|\wb{V}_t^{1/2}\right\|_{2}\left\|\bm{1}_d\right\|_{2}$, and in (f) we upper bounded the maximum eigenvalue of $\left\|\wb{V}_t^{1/2}\right\|_{2}$ by $\sqrt{\nu + L^2n}$. Therefore, we conclude,
\begin{align*}
    (a) := \left\|\wb{V}_t^{-1/2}v - \wb{w}_{v,t}\right\|_{\wb{V}_t} \leq \epsilon_2(1+\sqrt{d})\|v\|_2.
\end{align*}
Term (b) can be bounded by Lemma \ref{lemma:abbasi-w}. For any $\delta' \in (0,1)$, with probability at least $1-\delta'$,
\begin{align*}
    (b) := \|W_t\|_{\wb{V}_{t}^{-1}} \leq \sqrt{2\sigma^2d \log\left(\frac{1+\frac{tL^2}{d\nu}}{\delta'}\right)}.
\end{align*}
Term (c) can be bounded by Lemma \ref{lemma:ls-cover} (whose bound holds uniformly over all elements in $\mathcal{C}_2$). Recall that $\|w\|_2 \leq w_{\max}$ for all $w\in\mathcal{C}_2$. For any $\chi > 0$ and $\delta' \in (0,1)$, with probability at least $1-\delta'$,
\begin{align*}
    (c) := \wb{w}_{v,t}^T W_t &\leq \sqrt{2\sigma^2\gamma_n \max\left\{\chi, \|\wb{w}_{v,t}\|_{\wb{V}_t}^2\right\}\log\left(\frac{\Gamma_{w_{\max}^2L^2, \chi} | \mathcal{C}_2|}{\delta'}\right)}.
\end{align*}
Note that, by definition of $\mathcal{C}_2$, $\|\wb{w}_{v,t}\|_{\wb{V}_t}^2 \geq \sigma_{\min}(\wb{V}_t)\|\wb{w}_{v,t}\|_2^2 \geq \frac{\nu d\epsilon_2^2\|v\|^2}{\nu + L^2n} \geq \frac{\nu d^2\epsilon_2^2 v_{\min}^2}{\nu + L^2n}$. Hence, setting $\chi \leftarrow \chi_n' := \frac{\nu d^{2}\epsilon_2^2 v_{\min}^2}{\nu + L^2n}$,
\begin{align*}
    \Gamma_{w_{\max}^2L^2, \chi_n'} = 1 + \frac{\log (w_{\max}^2L^2n/\chi_n')}{\log \gamma_n} \leq 1 + \log (w_{\max}^2L^2n/\chi_n') \log(n)
\end{align*}
where the last inequality is from $\log(1 + \frac{1}{\log n}) \geq \frac{1}{2\log n}$ for $n \geq 2$. This yields
\begin{align*}
    (c) \leq \sqrt{2\sigma^2\gamma_n  \|\wb{w}_{v,t}\|_{\wb{V}_t}^2\log\left(\frac{(1 + \log (w_{\max}^2L^2n/\chi_n') \log(n))|\mathcal{C}_2|}{\delta'}\right)}.
\end{align*}

Let us now bound $\left\| \wb{w}_{v,t}\right\|_{\wb{V}_t}^2$. We have
\begin{align*}
    \left\| \wb{w}_{v,t}\right\|_{\wb{V}_t}^2 &= \left\| \wb{w}_{v,t} \pm \wb{V}_t^{-1/2}v\right\|_{\wb{V}_t}^2 = \left\| \wb{w}_{v,t} - \wb{V}_t^{-1/2}v\right\|_{\wb{V}_t}^2 + \left\| \wb{V}_t^{-1/2}v\right\|_{\wb{V}_t}^2 + 2\left(\wb{w}_{v,t} - \wb{V}_t^{-1/2}v\right)^T\wb{V}_t\left(\wb{V}_t^{-1/2}v\right)\\ &\leq \left\| \wb{w}_{v,t} - \wb{V}_t^{-1/2}v\right\|_{\wb{V}_t}^2 + \left\| v\right\|_{v}^2 + 2\left\| \wb{w}_{v,t} - \wb{V}_t^{-1/2}v\right\|_{\wb{V}_t}\left\| v\right\|_{2} \\ &\leq (\epsilon_2)^2(1+\sqrt{d})^2\|v\|_2^2 + \|v\|_2^2 + 2\epsilon_2(1+\sqrt{d})\|v\|_2^2 = \left(1 + \epsilon_2(1+\sqrt{d})\right)^2\|v\|_2^2,
\end{align*}
where in the last inequality we used the previous bound on $(a) = \left\| \wb{w}_{v,t} - \wb{V}_t^{-1/2}v\right\|_{\wb{V}_t}$.

Putting (a), (b), and (c) together we obtain the following bound on (ii):
\begin{align*}
    (ii) &= v^T \wb{V}_t^{-1/2}W_t \leq \|v\|_{2} \epsilon_2(1+\sqrt{d}) \sqrt{2\sigma^2d \log\left(\frac{1+\frac{nL^2}{d\nu}}{\delta'}\right)} \\ &+ \|v\|_{2}\left(1+ \epsilon_2(1+\sqrt{d})\right) \sqrt{2\sigma^2\gamma_n\log\left(\frac{(1 + \log (w_{\max}^2L^2n/\chi_n') \log(n))|\mathcal{C}_2|}{\delta'}\right)}.
\end{align*}
If we now set $\epsilon_2 \leftarrow \frac{1}{2d\log n}$, we have $\chi_n' = \frac{\nu v_{\min}^2}{4(\nu + L^2n)(\log n)^2}$. Setting $\chi_n'' = \chi_n' / (w_{\max}^2L^2)$ and using $w_{\max} = \frac{v_{\max}}{\sqrt{\nu}}\left(1+ \epsilon_2(1+\sqrt{d})\right) \leq \frac{v_{\max}}{\sqrt{\nu}}\gamma_n$, $\chi_n'' \geq \frac{\nu^2 v_{\min}^2}{4L^2(\nu + L^2n)(\log n)^2v_{\max}^2\gamma_n^2} = \chi_n$. Thus,
\begin{align*}
    (ii) \leq \|v\|_{2} \left( \sqrt{\frac{2\sigma^2 \log\left(\frac{1+\frac{nL^2}{d\nu}}{\delta'}\right)}{(\log n)^2}} + \sqrt{2\sigma^2\gamma_n^3\log\left(\frac{(1 + \log (n/\chi_n) \log(n))}{\delta'}\right) + 2\gamma_n^3\log|\mathcal{C}_2|}\right).
\end{align*}
Furthermore, using \eqref{eq:c2-size}, the log-size of the cover $\mathcal{C}_2$ is
\begin{align*}
    \Upsilon_n = \log |\mathcal{C}_2| &\leq \log(|\mathcal{C}_1|) + d\log\left(2 +\frac{\log\left(2d\log n\frac{v_{\max}}{v_{\min}}\sqrt{\frac{\nu + L^2n}{d\nu}}\right)}{\log(1+\frac{1}{2d\log n})} \right) \\ &\leq \log(|\mathcal{C}_1|) + d\log\left(2 + 4d\log\left(2d\log n\frac{v_{\max}}{v_{\min}}\sqrt{\frac{\nu + L^2n}{d\nu}}\right)\log n \right)
\end{align*}
To conclude the proof, we notice that the derivation above holds uniformly for all $v\in\mathcal{C}_1$ and $t\in[n]$ with probability at least $1-2\delta'$ since we applied both Lemma \ref{lemma:abbasi-w} (for term (b) in (ii)) and Lemma \ref{lemma:ls-cover} (for term (c) in (ii)). Thus, the statement follows by setting $\delta = 2\delta'$.
\end{proof}

\subsection{Auxiliary Results}
\label{app:conf.set.proof.auxiliary}

\begin{lemma}\label{lemma:abbasi-w}[Lemma 9 of \cite{abbasi2011improved}]
Let $\tau$ be a stopping time with respect to filtration $\{\mathcal{F}_t\}_{t=1}^\infty$ and $W_t := \sum_{s=1}^t \phi(X_s,A_s)\xi_s$. Then, for any $\delta \in (0,1)$, with probability at least $1-\delta'$,
\begin{align*}
    \|W_\tau\|_{\wb{V}_{\tau}^{-1}} \leq \sqrt{2\sigma^2\log \left( \frac{\det(\wb{V}_{\tau})^{1/2}\nu^{-d/2}}{\delta}\right)} \leq \sqrt{2\sigma^2d \log\left(\frac{1+\frac{\tau L^2}{d\nu}}{\delta}\right)}.
\end{align*}
\end{lemma}

The following result is a specialization of Lemma 2.6 of \cite{fan2015exponential} or Lemma 4.2 of \cite{khan2009p}.
\begin{lemma}\label{lemma:khan}
Let $n\in\mathbb{N}$ and $\{Y_t\}_{t=1}^n$ be a sequence of sub-Gaussian random variables adapted to filtration $\mathcal{F}$ such that $\expec{Y_t | \mathcal{F}_{t-1}} = 0$ and 
\begin{align*}
    \forall \zeta \in \mathbb{R} : \expec{e^{\zeta Y_t} | \mathcal{F}_{t-1}} \leq e^{\frac{\zeta^2\sigma_t^2}{2}},
\end{align*}
where $\sigma_t^2 := \Var[Y_t | \mathcal{F}_{t-1}]$. Then, for all $\epsilon\geq 0, v > 0$,
\begin{align*}
    \prob{\exists t \leq n : \sum_{s=1}^t Y_s \geq \epsilon, \sum_{s=1}^t \sigma_s^2 \leq v } \leq e^{-\frac{\epsilon^2}{2v}}.
\end{align*}
\end{lemma}
\begin{proof}
The result follows straightforwardly from Lemma 2.6 of \cite{fan2015exponential} or Lemma 4.2 of \cite{khan2009p} after optimizing for $\zeta$.
\end{proof}

\begin{lemma}[Lemma 14 of \cite{lattimore2017end}]\label{lemma:ls-gauss}
Let $n\in\mathbb{N}$ and $\epsilon > 0$. Let $\{Y_t\}_{t=1}^n$ be a sequence of Gaussian random variables adapted to filtration $\mathcal{F}$ such that $\expec{Y_t | \mathcal{F}_{t-1}} = 0$ and $\Var[Y_t | \mathcal{F}_{t-1}] \leq b$ for some $b > 0$. Then
\begin{align*}
    \prob{\exists t \leq n : \sum_{s=1}^t Y_s \geq \sqrt{2\gamma_nP_t\log\frac{\Gamma_{b,\epsilon}}{\delta}}} \leq \delta,
\end{align*}
where $P_t = \max\{\epsilon, \sum_{s=1}^t \Var[Y_t | \mathcal{F}_{t-1}]\}$, $\gamma_n = 1 + \frac{1}{\log n}$, and $\Gamma_{b,\epsilon} = 1 + \frac{\log (nb / \epsilon)}{\log \gamma_n}$.
\end{lemma}
\begin{proof}
The proof uses the same peeling argument as in \cite{lattimore2017end} but follows different steps.

Let $\tau \leq n$ be a stopping time with respect to $\mathcal{F}$ whose value will be specified later. Define $\Upsilon_t := \sum_{s=1}^t \Var[Y_t | \mathcal{F}_{t-1}]$ as the sum of predictable variances and $f(v) := \sqrt{2\gamma_n\max\{v,\epsilon\}\log\frac{1}{\delta'}}$. Let us define a sequence of scalars $v_{-1}, v_0, \dots v_{k_n}$, which will be used to discretize the predictable variances, with $v_{-1}=0$, $v_0$ to be specified later, $v_j = \gamma_nv_{j-1}$ for $j \geq 1$, and $k_n$ such that $v_{k_n} \geq nb$ (which implies $v_{k_n} \geq \Upsilon_n$). Note that the theorem holds trivially when $\Upsilon_\tau=0$, so we consider the case where this variable is positive. We have
\begin{align*}
    \prob{\sum_{s=1}^\tau Y_s \geq f(\Upsilon_\tau)} &\stackrel{(a)}{\leq} \sum_{j=0}^{k_n} \prob{\sum_{s=1}^\tau Y_s \geq f(\Upsilon_\tau), \Upsilon_\tau \in (v_{j-1}, v_j]}\\ &\stackrel{(b)}{\leq} \sum_{j=0}^{k_n} \prob{\sum_{s=1}^\tau Y_s \geq f(v_{j-1}), \Upsilon_\tau \leq v_j} \stackrel{(c)}{\leq} \sum_{j=0}^{k_n} e^{-\frac{f(v_{j-1})^2}{2v_j}},
\end{align*}
where (a) uses a union bound, (b) holds since $f$ is non-decreasing, and (c) is from Lemma \ref{lemma:khan}. Using the definition of $\{v_j\}_{j\geq-1}$,
\begin{align*}
    \sum_{j=0}^{k_n} e^{-\frac{f(v_{j-1})^2}{2v_j}} = e^{-\frac{\gamma_n\epsilon\log\frac{1}{\delta'}}{v_0}} + \sum_{j=1}^{k_n} e^{-\frac{2\gamma_n\max\{v_j/\gamma_n,\epsilon\}\log\frac{1}{\delta'}}{2v_j}} \leq (\delta')^{\frac{\gamma_n\epsilon}{v_0}} + k_n\delta'.
\end{align*}
Since $v_{k_n} = \gamma_n^{k_n}v_0$, we have that $k_n = \left\lceil \frac{\log(nb/v_0)}{\log(\gamma_n)}\right\rceil$ suffices to have $v_{k_n} \geq nb$. Setting $v_0 \leftarrow \gamma_n\epsilon$, \begin{align*}
    \prob{\sum_{s=1}^\tau Y_s > f(\Upsilon_\tau)} \leq \delta'\left( 1 + \left\lceil \frac{\log(nb/\epsilon) - \log(\gamma_n)}{\log(\gamma_n)}\right\rceil\right) \leq \delta'\left( 1 + \frac{\log(nb/\epsilon)}{\log(\gamma_n)}\right) = \delta'\Gamma_{b,\epsilon}.
\end{align*}
The result follows by setting $\delta \leftarrow \delta'\Gamma_{b,\epsilon}$ and $\tau \leftarrow \min\left\{t\leq n : \sum_{s=1}^t Y_s > \sqrt{2\gamma_nP_t\log\frac{\Gamma_{b,\epsilon}}{\delta}}\right\}$.
\end{proof}

The following result can be derived using a similar argument as in the proof of Lemma 15 of \cite{lattimore2017end}.
\begin{lemma}\label{lemma:ls-cover}
Let $\mathcal{C} \subset \{w \in \mathbb{R}^d : \|w\| \leq b\}$ be a finite set of vectors in $\mathbb{R}^d$ with norm bounded by $b > 0$ and $W_t$ as defined in Lemma \ref{lemma:abbasi-w}. Then, for all $\epsilon > 0$ and $\delta \in (0,1)$,
\begin{align*}
    \prob{\exists t\leq n, w\in\mathcal{C} : w^T W_t \geq \sqrt{2\sigma^2\gamma_n \max\{\epsilon, \|w\|_{\wb{V}_t}^2\}\log\left(\frac{\Gamma_{b^2L^2,\epsilon} | \mathcal{C}|}{\delta}\right)}} \leq \delta,
\end{align*}
where $\gamma_n$ and $\Gamma_{b^2L^2,\epsilon}$ are those defined in Lemma \ref{lemma:ls-gauss}.
\end{lemma}
\begin{proof}
Fix $w\in\mathcal{C}$. Note that
\begin{align*}
    \frac{w^T W_t}{\sigma} = \sum_{s=1}^t \frac{w^T \phi(X_s,A_s) \xi_s}{\sigma}
\end{align*}
is a sum of Gaussian random variables adapted to $\mathcal{F}$ such that
\begin{align*}
    \Var\left[\frac{w^T \phi(X_s,A_s) \xi_s}{\sigma} | \mathcal{F}_{s-1}\right] = \frac{(w^T \phi(X_s,A_s))^2}{\sigma^2} \underbrace{\Var[\xi_s | \mathcal{F}_{s-1}]}_{= \sigma^2} \leq \|w\|^2\|\phi(Y_s,A_s)\|^2 \leq b^2L^2.
\end{align*}
Furthermore,
\begin{align*}
    \sum_{s=1}^t (w^T \phi(Y_s,A_s))^2 = \sum_{s=1}^t w^T \phi(Y_s,A_s)\phi(Y_s,A_s)^T w = \|w\|_{V_t}^2 \leq \|w\|_{\wb{V}_t}^2,
\end{align*}
where the last inequality is from $V_t \preceq \wb{V}_t$. Therefore, using Lemma \ref{lemma:ls-gauss}, with probability at least $1-\delta'$,
\begin{align*}
    w^TW_t \leq \sqrt{2\sigma^2\gamma_n \max\{\epsilon, \|w\|_{\wb{V}_t}^2\}\log\frac{\Gamma_{b^2L^2,\epsilon} }{\delta'}}.
\end{align*}
The result follows after taking a union bound over all elements in $\mathcal{C}$.
\end{proof}

\section{Additional Experiments}\label{app:experiments}

\subsection{Implementation Details}\label{app:implementation-details}

\red{In our implementation of \algo, we ignore the projection of the parameters computed by regularized least squares onto $\Theta$. Moreover, we remove the restriction that the alternative parameters should lie in $\Theta$. That is, we use
\begin{align}
\Theta_{\mathrm{alt}} := \{ \theta' \in \R^d\ |\ \exists x\in\X,\ a^\star_{\theta^\star}(x) \neq a^\star_{\theta'}(x) \},
\end{align}
and similarly for $\wb{\Theta}_{t}$. In this case, for linear bandits with Gaussian noise, the infimum over alternative models in the constraint of~\eqref{eq:optim-lb} can be computed in closed form as
\begin{align}\label{eq:pdlin.glr}
2\sigma^2\!\!\! \inf_{\theta'\in\Theta_{\mathrm{alt}}} \sum_{x,a}\eta(x,a) d_{x,a}(\theta^\star, \theta') = \inf_{\theta'\in\Theta_{\mathrm{alt}}} \| \theta^\star - \theta' \|_{V_\eta}^2 =\!\!\! \min_{\substack{x\in\X,\\a \neq a^\star_{\theta^\star}(x)}}
\frac{ \Delta_{\theta^\star}(x,a)^2}{\|\phi(x,a) - \phi_{\theta^\star}^\star(x))\|_{V_\eta^{-1}}^2},
\end{align}
where $V_\eta = \sum_{x,a} \eta(x,a) \phi(x,a) \phi(x,a)^\transp$ and $\phi_{\theta^\star}^\star(x) = \phi(x,a^\star_{\theta^\star}(x))$. The same closed-form can be used for the infimum in the constraint \eqref{eq:pdlin.gt}. Regarding the exploitation test, we restrict the set of alternative reward parameters to those with ``incompatible'' optimal arm in the last observed context. That is, we use the test
\begin{align}
\inf_{\theta'\in \wt{\Theta}_{t-1}} \| \wh{\theta}_{t-1} - \theta' \|_{\wb{V}_{t-1}}^2 > \beta_{t-1},
\end{align}
where $\wt{\Theta}_{t-1} = \{ \theta' \in \R^d\ |\ a^\star_{\wh\theta_{t-1}}(X_t) \neq a^\star_{\theta'}(X_t) \}$. Once again, the infimum can be computed in closed form as before (without the minimum over contexts).
}

\subsection{Experiment Configurations}

We provide the detailed configurations of the experiments reported in the main paper. We use the same confidence intervals in all experiments. For \algo, we set $\beta_t = \sigma^2 (\log(t) + d\log\log(n))$ and $\gamma_t = \sigma^2 (\log(S_t) + d\log\log(n))$ as prescribed by Thm.~\ref{th:conf-theta} (without numerical constants). For OAM, we use the same $\beta_t$ for the exploitation test. For LinUCB, we use the confidence set of \citep{abbasi2011improved} without numerical constants. Similarly, we implement LinTS as defined in~\cite{agrawal2013thompson} but without the extra-sampling factor $\sqrt{d}$ used to prove its frequentist regret. All plots are the results of $100$ runs with $95\%$ Student's t confidence intervals. 

In both experiments, for \algo we set $\alpha_\omega=1$, $\alpha_\lambda = 0.5$, and we normalize the gradients by context in $l_2$-norm. We do not reset the optimizer at the beginning of each phase. We use the theoretical exponential schedule for $z_k$ and $p_k$ as defined in Thm.~\ref{thm:regret.bound}. We set $z_0 = 1$, $\lambda_1 = 0$ for the first experiment and $z_0 = |\A|$, $\lambda_1 = 50$ for the second one. The reward noise is $\sigma = 0.5$ in the first experiment and $\sigma = 1$ in the second one.

\paragraph{Generation of Random Problems}

We adopt the following procedure in order to generate the random bandit models for the second experiment. We first randomly sample a sparse $|\X||\A|\times d$ feature matrix and a sparse vector $\theta^\star$ with entries uniformly distributed in $[0,1]$. We then compute the resulting optimal arms for each context and check whether they span $\R^d$. If they do, we discard the generated features/parameter and repeat the previous procedure. Otherwise we keep the bandit problem. Discarding problems where the features of the optimal arms span $\R^d$ is done in order to avoid easy bandit problems in which exploration is not necessary (see \citep{hao2019adaptive})\footnote{Problems that can be solved by a greedy strategy would not reveal any interesting empirical difference between \algo and the other baselines.}.

\subsection{Parameter Analysis}

We provide an empirical study of how different choices for the relevant parameters of SOLID affect the algorithm's performance in the toy problem of Sec. \ref{sec:experiments}. We note that the purpose of this section is to build some intuition on how SOLID behaves with different parameters rather than assessing which configurations are globally better.

We use the two-context toy problem of Sec. \ref{sec:experiments} with $\xi = 0.1$ and $\sigma^2 = 1$. We study the effect of the following parameters, with corresponding default values.
\begin{itemize}
\item $z_0$ (default $30$): the initial normalization factor;
\item $\lambda_1$ (default $0$): the initial multiplier;
\item $\alpha^\omega$ (default $0.1$): learning rate for $\omega$. We keep it fixed instead of decreasing with the phase length as suggested by the theory;
\item $\alpha^\lambda$ (default $0.5$): learning rate for $\lambda$. We keep it fixed as for $\alpha^\omega$;
\item $z_k, p_k$ (default $z_k = z_0e^k$, $p_k = z_k e^{2k}$): the schedule for the phase length. We use the one for which we derive regret guarantees by default but we also experiment with other schedules. By default we do not reset the optimizer at the beginning of each phase.
\end{itemize}
We vary each parameter in a suitable range while keeping all the others fixed to their default values. The results are described in the following paragraphs.

\begin{figure*}[t]
\centering
\includegraphics[height=5cm]{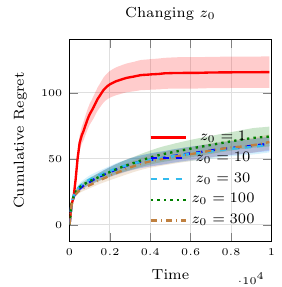}
\includegraphics[height=5cm]{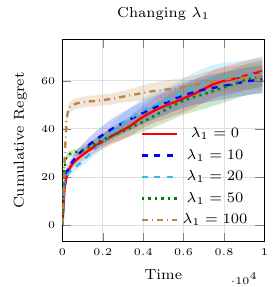}
\caption{The effect of changing $z_0$ (left) and $\lambda_1$ (right).}\label{fig:ablation-z-lambda}
\end{figure*}

\begin{figure*}[t]
\centering
\includegraphics[height=5cm]{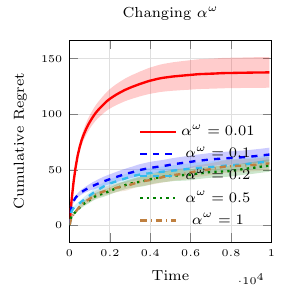}
\includegraphics[height=5cm]{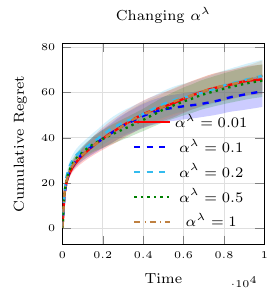}
\caption{The effect of changing $\alpha^\omega$ (left) and $\alpha^\lambda$ (right).}\label{fig:ablation-alpha}
\end{figure*}

\paragraph{Changing $z_0$} As mentioned in the main paper, the initial value of the parameter $z$ controls both the feasibility of the optimization problem and the trade-off between minimizing regret and gathering information about the optimal arms when $t$ is small. While a small value of $z_0$ might lead SOLID to collect a large amount of information, this might bring high finite regret as derived in the regret bound. Fig. \ref{fig:ablation-z-lambda}(left) confirms this claim, where the value $z_0=1$ suffers high initial regret but the resulting curve has a better slope.

\paragraph{Changing $\lambda_1$} Though the initial multiplier has no particular impact on the regret bound, in practice it induces a behavior similar to $z_0$, where larger values lead SOLID to collect more information about $\theta^\star$ in the very first learning steps (see Fig. \ref{fig:ablation-z-lambda}(right)).

\paragraph{Changing the step sizes} Fig. \ref{fig:ablation-alpha} shows the effect of varying $\alpha^\omega$ and $\alpha^\lambda$. In this particular case, $\alpha^\lambda$ seems to have no remarkable effect on SOLID's performance. On the other hand, the algorithm is quite sensible to the choice of $\alpha^\omega$, with very small values performing poorly since the policy is updated rarely and remains close to uniform for a long time. More aggressive step sizes seem to yield the best performance.

\begin{figure*}[t]
\centering
\includegraphics[height=4.5cm]{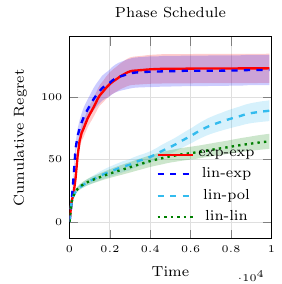}
\includegraphics[height=4.5cm]{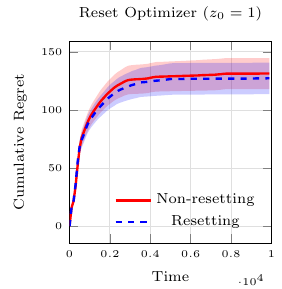}
\includegraphics[height=4.5cm]{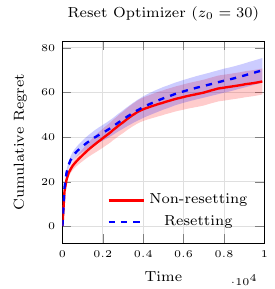}
\caption{Different phase schedules (left) and effect of resetting the optimizer (middle and right plots).}\label{fig:ablation-schedule}
\end{figure*}

\paragraph{Phase schedule} 

We test different schedules for $z_k$ and $p_k$ with respect to the one prescribed by the theory. We have $z_k = z_0 e^k, p_k=z_ke^{2k}$ (exp-exp), $z_k = z_0 (1 + k), p_k=z_ke^{k}$ (lin-exp), $z_k = z_0 (1+k), p_k=z_k(1+k)^2$ (lin-pol), and $z_k = z_0 (1+k), p_k=z_k(1+k)$ (lin-lin). Fig. \ref{fig:ablation-schedule}(left) shows the result (here we set $z_0=1$ to better highlight the contribution of the different schedules). The exponential schedules are as expected more conservative since the algorithm spends more time optimizing with small values of $z$ (i.e., seeks more information).  The linear and polynomial schedules behave, on the other hand, more greedily and suffer less regret, though the resulting curve has larger slope.

We also test the effect of resetting the optimizer (middle and right plots in Fig. \ref{fig:ablation-schedule}). We see that resetting the optimizer does not significantly affect the algorithm's performance both in case $z=1$ and $z=30$. This is likely due to the fact that phases are long (thanks to the exponential schedule) and that the algorithm spends many steps in the exploit phase, where no optimization is performed.

\begin{figure*}[t]
\centering
\includegraphics[height=5cm]{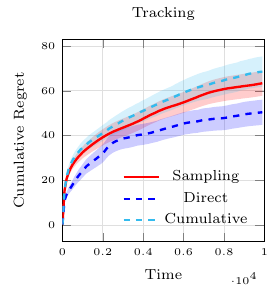}
\includegraphics[height=5cm]{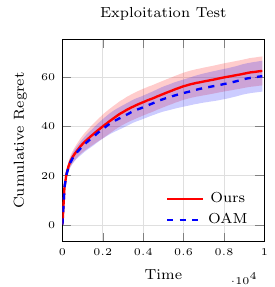}
\caption{Different tracking strategies (left) and comparison with the exploitation test used in OAM.}\label{fig:ablation-track-glrt}
\end{figure*}

\paragraph{Tracking} We compare the sampling strategy adopted by SOLID with the popular direct and cumulative tracking rules. Interestingly, Fig. \ref{fig:ablation-track-glrt}(left) shows that sampling from $\omega$ constitutes a nice trade-off between cumulative tracking and the more aggressive direct tracking. Note that, while our theoretical results can be easily derived for cumulative tracking, we do not know whether the same can be done for direct tracking.

\paragraph{Exploitation test}

We note that the test performed by \algo in order to decide whether to explore or exploit is slightly different from the one adopted in OAM. In fact, the closed-form of the infimum over the alternative set (Eq. \ref{eq:pdlin.glr}) leads to terms of the form $\Delta_{\wh{\theta}_t}(x,a)^2/\|\phi(x,a)-\phi(x)^\star\|_{\wb{V}_t^{-1}}^2$ while OAM uses $\Delta_{\wh{\theta}_t}(x,a)^2/\|\phi(x,a)\|_{\wb{V}_t^{-1}}^2$. We verify empirically (Fig. \ref{fig:ablation-track-glrt}(right)) that the two tests lead to very similar performance.

\subsection{Real Dataset}

We report additional results on real data. We use the Jester Dataset \citep{goldberg2001eigentaste} which consists of joke ratings in a continuous range from $-10$ to $10$ for a total of $100$ jokes and 73421 users. We select a subset of 40 jokes and 19181 users rating all these 40 jokes. 

We build a linear contextual problem as follows. We first extract separate $36$-dimensional user (context) and joke (arm) features via a low-rank matrix factorization. Then, we concatenate these user and joke features (thus obtaining vectors with $72$ entries) and fit a $64\times 64$ neural-network with ReLU non-linearities to predict the ratings of a random subset of $75\%$ of the users, using these feature vectors as inputs. We obtain $R^2 \simeq 0.95$ on the remaining $25\%$ users. Finally, we take the features extracted in the last layer of the network as the features for our bandit problem and the parameters of the same layer as $\theta^\star$. Rewards in our bandit problem are generated from this linear model by perturbing the prediction with $\mathcal{N}(0, 0.5^2)$ noise. We thus obtain a problem with $d=65$ (the $64$ hidden neurons plus the bias term), $40$ arms (the jokes), and a total of $19181$ users.

We run the algorithms for $2 \cdot 10^6$ steps, with each run randomizing a subset of $1\%$ of the total users (hence $|\X|$ = 191) and using all $40$ arms. For \algo, we use the same parameters as in the experiment with random models. Due to the computational bottleneck demonstrated in the previous experiments, we could not run OAM on this problem. The results are shown in Figure \ref{fig:jester} and confirm that \algo achieves superior performance than the other baselines.

\begin{figure*}[t]
\centering
\includegraphics[height=5cm]{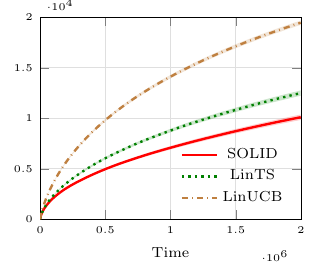}
\caption{Experiment on a real dataset (Jester).}\label{fig:jester}
\end{figure*}

\end{document}